\documentclass{article}

\PassOptionsToPackage{numbers, compress}{natbib}

\usepackage[final]{neurips_2023}

\usepackage[utf8]{inputenc} 
\usepackage[T1]{fontenc}    
\usepackage{hyperref}       
\usepackage{url}            
\usepackage{booktabs}       
\usepackage{amsfonts}       
\usepackage{amsmath}
\usepackage{amssymb}
\usepackage{mathtools}
\usepackage{mathabx}
\usepackage{amsthm}
\usepackage{bbm}
\usepackage{nicefrac}       
\usepackage{microtype}      
\usepackage{xcolor}         
\numberwithin{equation}{section}
\usepackage[linesnumbered,ruled,longend]{algorithm2e}

\SetCommentSty{mycommfont}
\usepackage{tikz}
\newcommand*\circled[1]{\tikz[baseline=(char.base)]{
            \node[shape=circle,draw,inner sep=0.5pt] (char) {#1};}}

\usepackage{caption}
\usepackage{subcaption}
\usepackage{graphicx}
\usepackage{pifont}
\usepackage{makecell}
\usepackage{MnSymbol}
\usepackage{cancel}
\usepackage{wrapfig}
\usepackage{verbatim}

\usepackage{soul}

\newtheorem{theorem}{Theorem}

\newtheorem{lemma}{Lemma}
\newtheorem{assumption}{Assumption}
\newtheorem{definition}{Definition}
\newtheorem{proposition}{Proposition}
\newtheorem{remark}{Remark}

\newcommand{\add}[1]{\textcolor{black}{#1}}
\newcommand{\javad}[1]{\textcolor{black}{#1}}
\newcommand{\jj}[1]{\textcolor{black}{#1}}
\newcommand{\hyunin}[1]{\textcolor{black}{#1}}
\newcommand{\hh}[1]{\textcolor{black}{#1}}

\newcommand{\floor}[1]{\left\lfloor #1 \right\rfloor}

\DeclareMathOperator*{\argmin}{arg\,min}

\usepackage{natbib}
\renewcommand{\footnotemark}{\fnsymbol{footnote}}

\title{Tempo Adaptation in Non-stationary Reinforcement Learning}

\author{%
Hyunin Lee$^{1,*}$\thanks{\scriptsize{* Corresponding authors. \hyunin{This work was supported by grants from ARO, ONR, AFOSR, NSF, and the UC Noyce Initiative.}}} \quad Yuhao Ding$^{1}$ \quad Jongmin Lee$^1$ \quad \textbf{Ming Jin}$^{2,*}$ \\ \textbf{Javad Lavaei}$^1$ \quad \textbf{Somayeh Sojoudi}$^1$\\
$^1$UC Berkeley, Berkeley, CA 94709 \\
$^2$Virginia Tech, Blacksburg, VA 24061 \\
\texttt{\{hyunin,yuhao$\_$ding,jongmin.lee,lavaei,sojoudi\}@berkeley.edu}\\
\texttt{jinming@vt.edu}
}

\begin{document}

\maketitle

\begin{abstract}

    We first raise and tackle a ``time synchronization'' issue between the agent and the environment in non-stationary reinforcement learning (RL), a crucial factor hindering its real-world applications. In reality, environmental changes occur over wall-clock time ($\mathfrak{t}$) rather than episode progress ($k$), where wall-clock time signifies the actual elapsed time within the fixed duration $\mathfrak{t} \in [0, T]$. In existing works, at episode $k$, the agent rolls a trajectory and trains a policy before transitioning to episode $k+1$. In the context of the time-desynchronized environment, however, the agent at time $\mathfrak{t}_k$ allocates $\Delta \mathfrak{t}$ for trajectory generation and training, subsequently moves to the next episode at $\mathfrak{t}_{k+1}=\mathfrak{t}_{k}+\Delta \mathfrak{t}$. Despite a fixed total number of episodes ($K$), the agent accumulates different trajectories influenced by the choice of \textit{interaction times} ($\mathfrak{t}_1,\mathfrak{t}_2,...,\mathfrak{t}_K$), significantly impacting the suboptimality gap of the policy. We propose a Proactively Synchronizing Tempo (\texttt{ProST}) framework that computes a suboptimal sequence $\{ \mathfrak{t}_1,\mathfrak{t}_2,...,\mathfrak{t}_K \} (= \{ \mathfrak{t} \}_{1:K})$ by minimizing an upper bound on its performance measure, i.e., the dynamic regret. Our main contribution is that we show that a suboptimal $\{ \mathfrak{t} \}_{1:K}$ trades-off between the policy training time (agent tempo) and how fast the environment changes (environment tempo). Theoretically, this work develops a suboptimal $\{ \mathfrak{t} \}_{1:K}$ as a function of the degree of the environment's non-stationarity while also achieving a sublinear dynamic regret. Our experimental evaluation on various high-dimensional non-stationary environments shows that the \texttt{ProST} framework achieves a higher online return at suboptimal $\{ \mathfrak{t} \}_{1:K}$ than the existing methods.

\end{abstract}

\section{Introduction}

The prevailing reinforcement learning \javad{(RL)} paradigm gathers past data, trains models in the present, and deploys them in the \emph{future}. This approach has proven successful for \textit{stationary} Markov decision processes (MDPs), where the reward and transition functions remain constant \cite{mnih2013playing,Silver2016MasteringTG,dulac2019challenges}. However, challenges arise when the environments undergo significant changes, particularly when the reward and transition functions are dependent on time or latent factors \cite{zintgraf2019varibad,zhu2020transfer,padakandla2021survey}, in \textit{non-stationary} MDPs. Managing non-stationarity in environments is crucial for real-world \javad{RL} applications. Thus, adapting to changing environments is pivotal in non-stationary RL.

This paper addresses a practical concern that has inadvertently been overlooked within traditional non-stationary RL environments, namely, the time synchronization between the agent and the environment. We raise the impracticality of utilizing \emph{episode-varying} environments in existing non-stationary RL research, as such environments do not align with the real-world scenario where changes occur regardless of the agent's behavior. In an episode-varying environment, the agent has complete control over determining the time to execute the episode $k$, the duration of policy training between the \javad{episodes $k$ and $k+1$}, and the transition time to the episode $k+1$. The issue stems from the premise that the environment undergoes dynamic changes throughout the course of each episode, with the rate of non-stationarity contingent upon the behavior exhibited by the agent. However, an independent \textit{wall-clock time} ($\mathfrak{t}$) exists in a real\javad{-}world environment, thereby \javad{the} above three events are now recognized as wall-clock time $\mathfrak{t}_k$, training time $\Delta \mathfrak{t}$, and $\mathfrak{t}_{k+1}$. The selection of interaction times ($\mathfrak{t}_{k},\mathfrak{t}_{k+1}$) has a notable impact on the collected trajectories, while the interval $\mathfrak{t}_{k+1}-\mathfrak{t}_{k}$ establishes an upper limit on the duration of training ($\Delta \mathfrak{t}$). This interval profoundly influences the suboptimality gap of the policy. In the context of \javad{a} time-desynchronized environment, achieving an optimal policy requires addressing a previously unexplored question: the determination of the \textit{optimal time sequence} $\{ \mathfrak{t}_1,\mathfrak{t}_2,..,.\mathfrak{t}_K\} (= \{ \mathfrak{t} \}_{1:K} )$  \javad{at which} the agent \javad{should interact} with the environment.

\begin{figure}[t]
    \centering
    \includegraphics[width=1.0\textwidth]{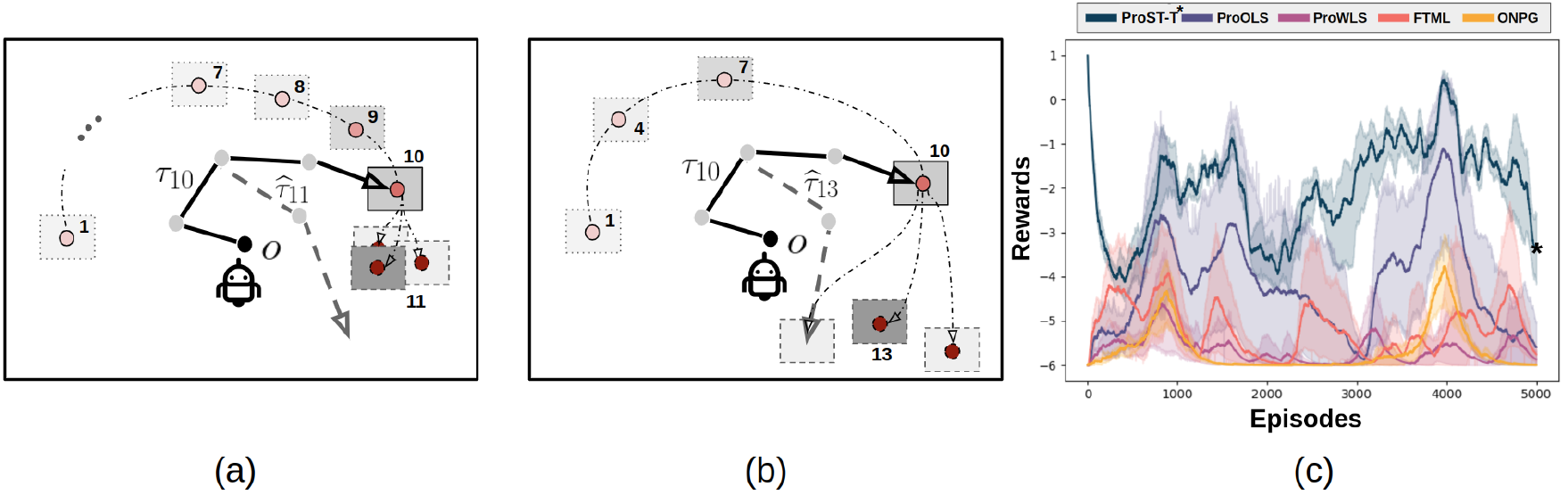}
    \caption{(a) 2D goal reacher in a time-desynchronized environment for one policy update\javad{, where the} agent learns an inaccurate policy on an accurate model\javad{;} (b) For three policy updates, the agent learns a near-optimal policy on an inaccurate model\javad{;} (c) Rewards per episode in 2D goal reacher with four model-free baselines, where \texttt{ProST-T}$^{*}$ is one of our proposed methods.}
    \label{fig:fig1}
\end{figure}
 
We elucidate the significance of the aforementioned research question through \javad{an} example. Consider a robot with the goal of reaching inside a gray-shaded non-fixed target box, known as the goal reacher (Appendix \ref{appendix: 2D reacher environment setting}). Note that the reward changes as the position of the box changes over time (Figure \ref{fig:fig1}-(a)). We begin by considering a scenario in which the wall-clock time and episode are synchronized, wherein the environment evolves alongside the episode. During each episode $k$, the agent rollouts a trajectory and iteratively updates the policy $N$ times, with the assumption that one policy update requires one second, \javad{and} then the agent transitions to the subsequent episode $k+1$. In conventional non-stationary RL environments, it is evident that a larger value of $N$ provides an advantage in terms of \javad{a} faster adaptation to achieve a near-optimal policy. However, regardless of the chosen value of $N$, the agent will consistently encounter the same environment in the subsequent episode. Now\javad{,} consider a scenario in which \javad{the} wall-clock time and episode are desynchronized. In this context, given a fixed wall-clock time duration $\mathfrak{t} \in [0,10]$, the agent is faced with \javad{the} additional task of determining both the total number of interactions (denoted as the total episode $K$) and the specific time \javad{instances} for these interactions $\{ \mathfrak{t} \}_{1:K}$\javad{,} where $\mathfrak{t}_{k} \in [0,10],\mathfrak{t}_{k-1} < \mathfrak{t}_{k}$ for $\forall k \in [K]$. Figure \ref{fig:fig1}\javad{(a)} shows an agent that interacts with the environment ten times, that is, \hyunin{$\{\mathfrak{t}\}_{1:K}=\{1,2,...,10\}$}, and spends \javad{the time} interval $(\mathfrak{t}_{k},\mathfrak{t}_{k+1})$ to train the policy, which consumes \javad{one} second ($K=10, N=1)$. The high frequency of interaction ($K=10$) provides adequate data for precise future box position learning (\hyunin{$\mathfrak{t}=11$}), yet a single policy update $(N=1)$ may not approximate the optimal policy. Now, if the agent interacts with the environment four times, i.e. \hyunin{$\{\mathfrak{t}\}_{1:K}=\{1,4,7,10\}$} (see Figure \ref{fig:fig1}\javad{(b)}), it becomes feasible to train the policy over a duration of three seconds ($K=4, N=3$). A longer period of policy training $(N=3)$ helps the agent in obtaining a near-optimal policy. However, limited observation data $(K=4)$ and \hyunin{large} time intervals (\hyunin{$\mathfrak{t} \in\{11,12,13\}$}) \javad{may} lead to less accurate box predictions. This example underscores the practical importance of aligning \javad{the interaction time of the agent} with the environment in non-stationary RL. Determining the optimal \javad{sequence $\{ \mathfrak{t} \}_{1:K}$} involves a trade-off between achieving an optimal model and an optimal policy.

Based on the previous example, our key insight is that, in non-stationary environments, \javad{the} new factor \textbf{tempo} emerges. Informally, tempo refers to the pace of processes occurring in a non-stationary environment. We define \textbf{environment tempo} as how fast the environment changes and \textbf{agent tempo} as how frequently it updates the policy. Despite the importance of considering the tempo to find the optimal $\{ \mathfrak{t} \}_{1:K}$, the existing formulations and methods for non-stationarity RL are insufficient. None of the existing works has adequately addressed this crucial aspect.

Our framework, \texttt{ProST}, provides a solution to \javad{finding} the optimal $\{ \mathfrak{t} \}_{1:K}$ \add{by computing a minimum solution \javad{to an} upper bound \javad{on} its performance measure}. It proactively optimizes the time sequence by leveraging the agent tempo and the environment tempo. \javad{The} \texttt{ProST} framework is divided into two components: future policy optimizer ($\texttt{OPT}_{\pi}$) and time optimizer ($\texttt{OPT}_{\mathfrak{t}}$), and is characterized by three key features: 1) it is \emph{proactive} in nature as it forecasts the future MDP model; 2) it is \emph{model-based} as it optimizes the policy in the created MDP; and 3) it \javad{is a} \emph{synchronizing tempo} framework as it finds \javad{a} suboptimal training time by adjusting how many times the agent needs to update the policy (agent tempo) \javad{relative} to how fast the environment changes (environment tempo).
Our framework is general in the sense that it \javad{can} be equipped with \javad{any common algorithm for policy update}. Compared to the existing works \cite{chandak2020optimizing,al2017continuous,finn2019online}, our approach achieves higher rewards and \javad{a} more stable performance over time (\javad{see} Figure \ref{fig:fig1}\javad{(c)} and Section \ref{experiments}). 

We analyze the statistical and computational properties of \texttt{ProST} in \javad{a} tabular MDP, which is named \texttt{ProST-T}. Our framework learns in a novel MDP, \javad{namely} elapsed time-varying MDP, and quantifies its non-stationarity with a novel metric, \javad{namely} time-elapsing variation budget, where both consider \hyunin{wall-clock time taken}. We \hyunin{analyze} the dynamic regret of \texttt{ProST-T} (Theorem \ref{theorem1}) into two components: dynamic regret of $\texttt{OPT}_{\pi}$ that learns a future MDP model (Proposition \ref{proposition1}) and \javad{dynamic regret} of $\texttt{OPT}_{\mathfrak{t}}$ that computes a near-optimal policy in that model (Theorem \ref{theorem2}, Proposition \ref{proposition2}). We show that both regrets \javad{satisfy} a \javad{sublinear rate with respect} to the total \javad{number of episodes} regardless of \javad{the} agent tempo. \javad{More} importantly, we \javad{obtain} suboptimal training time \javad{by minimizing an objective \jj{that strikes} a balance between} the upper \javad{bounds} of those two dynamic regrets, \javad{which reflect} the tempo\javad{s} of the agent and the environment (Theorem \ref{theorem3}). We find an interesting property that the future MDP model error of $\texttt{OPT}_{\pi}$ serves as a common factor on both regrets and \javad{show that the upper bound on the} dynamic regret of \texttt{ProST-T} can be \javad{improved by} a joint optimization problem of learning both different weights on observed data and a model (Theorem \ref{thm:3}, Remark \ref{rmk:W-LSE_modelerror}).

Finally, we introduce \texttt{ProST-G}, \javad{which is} an adaptable learning algorithm for high-dimensional tasks achieved through a practical approximation of \texttt{ProST}. Empirically, \texttt{ProST-G} provides solid evidence on different reward returns depending on policy training time and the significance of learning the future MDP model. \texttt{ProST-G} also consistently finds a near-optimal policy, outperforming four popular RL baselines that are used in non-stationary environments on three different Mujoco tasks.

\textbf{Notation}

The sets of natural, real, and non-negative real numbers are denoted by $\mathbb{N}, \mathbb{R}$\javad{, and} $\mathbb{R}_{+}$, respectively. For a finite set $Z$, the notation $|Z|$ denotes its cardinality and the notation $\Delta (Z)$ denotes the probability simplex over $Z$. For $X \in \mathbb{N}$, we define $[X] \javad{:=} \{1,2,..,X\}$. For \javad{a} variable $X$, we denote $\widehat{X}$ as a \textit{forecasted} (or \textit{predicted}) variable at the current time, \javad{and} $\widetilde{X}$ as the observed value in the past. Also, for any time variable $\mathfrak{t}>0$ and $k\in\mathbb{N}$, we denote \javad{the} time sequence $\{ \mathfrak{t}_1, \mathfrak{t}_2,..,\mathfrak{t}_k\}$ as $\{\mathfrak{t}\}_{1:k}$ , and variable $X$ at time $\mathfrak{t}_k$ as $X_{\mathfrak{t}_k}$. We use the \javad{shorthand} notation $X_{(k)}$(or $X^{(k)}$) for $X_{\mathfrak{t}_k}$(or $X^{\mathfrak{t}_k}$). We \javad{use the notation} $\{x\}_{a:b}$ \javad{to denote} a sequence of variables $\{x_{a},x_{a+1},...,x_{b}\}$, and $\{x\}_{(a:b)}$ \javad{to represent} a sequence of variables $\{x_{\mathfrak{t}_{a}},x_{\mathfrak{t}_{a+1}},...,x_{\mathfrak{t}_{b}}\}$. \javad{Given two} variables $x$ \javad{and} $y$, \javad{let} $x \vee y$ \javad{denote} $\max (x,y)$, and $x \wedge y$ \javad{denote} $\min (x,y)$. \javad{Given two} complex numbers $z_1$ and $z_2$, we \javad{write} $z_2=W(z_1)$ \javad{if} ${z_2}e^{z_2} =z_1$\javad{, where $W$ is the Lambert function}. \hyunin{Given a variable $x$, the notation $a = \mathcal{O}(b(x))$ means that $a \leq C \cdot b(x)$ for some constant $C > 0$ that is independent of $x$, and the notation $a = \Omega(b(x))$ means that $a \geq C \cdot b(x)$ for some constant $C > 0$ that is independent of $x$.} We have described the specific details in Appendix \ref{appendix:notations}.
\section{Problem statement: \javad{D}esynchronizing timelines}
\label{problem_statement_and_background}
\subsection{Time\javad{-}elapsing Markov Decision Process}
In this paper, we study a non-stationary Markov Decision Process (MDP) \javad{for which} the transition probability and the reward change \javad{over time}. We begin by clarifying that the term \emph{episode} is agent-centric, not environment-centric. Prior solutions for episode-varying (or step-varying) MDPs operate under the assumption that the timing of MDP changes aligns with the agent commencing a new episode (or step). \javad{We} introduce \javad{the new concept of} \textbf{time\javad{-}elapsing MDP}. It starts from \javad{the} wall-clock time $\mathfrak{t}=0$ to $\mathfrak{t}=T$\javad{, where} $T$ is fixed. The time\javad{-}elapsing MDP at time $\mathfrak{t} \in [0,T]$ is defined as $\mathcal{M}_\mathfrak{t} \coloneqq \langle \mathcal{S},\mathcal{A},H,P_\mathfrak{t},R_\mathfrak{t},\gamma \rangle $\javad{, where} $\mathcal{S}$ is the state space, $\mathcal{A}$ is the action space, $H$ is the number of steps, $P_\mathfrak{t} : \mathcal{S} \times \mathcal{A} \times \mathcal{S} \rightarrow \Delta(\mathcal{S})$ is the transition probability , $R_\mathfrak{t} : \mathcal{S}\times \mathcal{A} \rightarrow \mathbb{R}$ is the reward function, and $\gamma$ is a discounting factor. Prior to executing the first episode, the agent determines the total number of interactions with the environment (denoted as the \hyunin{number of} total episode $K$) and subsequently computes the sequence of interaction times $ \{ \mathfrak{t} \}_{1:K}$ through an optimization problem. We denote $\mathfrak{t}_k$ as the wall-clock time of the environment when the agent starts \javad{the} episode $k$. \javad{Similar to} the existing non-stationary RL \javad{framework}, the agent's objective is \javad{to learn} a policy $\pi^{\mathfrak{t}_k}: \mathcal{S} \rightarrow \Delta(\mathcal{A})$ for all $k$. This is achieved through engaging in a total of $K$ \javad{episode} interactions\javad{, namely} $\{ \mathcal{M}_{\mathfrak{t}_1},\mathcal{M}_{\mathfrak{t}_1},...,\mathcal{M}_{\mathfrak{t}_K} \}$\javad{, where} the agent dedicates \javad{the} time interval $(\mathfrak{t}_k,\mathfrak{t}_{k+1})$ for policy \javad{training and then} obtains a sequence of \add{sub}optimal policies $\{\pi^{\mathfrak{t}_1},\pi^{\mathfrak{t}_2},...,\pi^{\mathfrak{t}_K}\}$ to maximize \javad{a} non-stationary policy evaluation metric, \textit{dynamic regret}.

\javad{Dealing with time-elapsing MDP instead of conventional MDP} raises \javad{an important} question that should be addressed: within \javad{the} time duration $[0,T]$, \javad{what} time sequence $\{\mathfrak{t}\}_{1:K}$ yields favorable trajectory samples to obtain \javad{an} optimal policy\javad{? This question is also related to the following:} what is optimal \javad{value of $K$, i.e. the total number of episode} that encompasses \javad{a satisfactory} balance between \javad{the} amount of observed trajectories and \javad{the} accuracy of policy training\javad{? These interwined questions are concerned with an important aspect of RL, which is} the computation of the optimal policy for a given $\mathfrak{t}_k$. In \javad{Section \ref{From_FT-MBPO_to_F-MBPO}}, we propose \javad{the} \texttt{ProST} framework that computes \javad{a} suboptimal $K^*$ and \javad{its} corresponding \add{sub}optimal time sequence $\{\mathfrak{t}^*\}_{1:K^*}$ based on the information of the environment's rate of change. \javad{In Section \ref{method}, we compute a} near-optimal policy for $\{\mathfrak{t}^*\}_{1:K^*}$. \javad{Before proceeding with the above results,} we introduce a new metric quantifying the environment's pace of change, \javad{referred to as} time\javad{-}elapsing variation budget.

\subsection{Time\javad{-}elapsing variation budget}

\textit{Variation budget}  \cite{ding2022non,bradtke1996linear,Gur2014StochasticMP} is a metric to quantify the speed with which the environment changes. Driven by our motivations, we introduce a new metric imbued with real-time considerations, \javad{named} \textit{time\javad{-}elapsing variation budget} $B(\Delta \mathfrak{t})$. Unlike the existing variation budget, which quantifies the environment's \javad{non-stationarity} across \javad{episodes} $\{1,2,..,K\}$, our definition accesses it across $\{\mathfrak{t}_1,\mathfrak{t}_2,...,\mathfrak{t}_K\}$, where \javad{the} interval $\Delta \mathfrak{t} =\mathfrak{t}_{k+1}-\mathfrak{t}_k$ remains constant regardless of $ k\in[K-1]$. For further analysis, we define two time\javad{-}elapsing variation budgets, one for transition probability and \javad{another} for reward function.
\begin{definition}[Time\javad{-}elapsing variation budget\javad{s}] For \javad{a given sequence} $\{ \mathfrak{t}_1,\mathfrak{t}_2,..,\mathfrak{t}_K\}$, assume that \javad{the} interval $\Delta \mathfrak{t}$ is equal to the policy training time $\Delta_\pi$. We define two time\javad{-}elapsing variation budgets $B_{p}(\Delta_\pi)$ and $B_{r}(\Delta_\pi)$ as
    \begin{equation*}
            B_{p}(\Delta_\pi) \coloneqq \sum_{k=1}^{K-1} \sup_{s,a}|| P_{\mathfrak{t}_{k+1}}(\cdot~|s,a) - P_{\mathfrak{t}_k}(\cdot~|s,a) ||_{1}, ~B_{r}(\Delta_\pi)  \coloneqq \sum_{k=1}^{K-1}\sup_{s,a} |  R_{\mathfrak{t}_{k+1}}(s,a) - R_{\mathfrak{t}_{k}}(s,a)  |.
    \end{equation*}
\end{definition}
To enhance the representation of a real-world system \javad{using the} time\javad{-}elapsing variation budget\javad{s}, we \javad{make} the following assumption. 
\begin{assumption}[Drifting constants]
    There \javad{exist} constants $c>1$ and $\alpha_r,\alpha_p \geq 0$ \javad{such that} $B_p(c\Delta_\pi) \add{\leq} c^{\alpha_{p}} B_p(\Delta_\pi)$ and $B_r(c\Delta_\pi) \add{\leq} c^{\alpha_{r}} B_r(\Delta_\pi)$. We call $\alpha_p$ and $\alpha_r$ the drifting constants. 
    \label{assum:tevb}
\end{assumption}

\subsection{\add{Sub}optimal training time}

Aside from the formal MDP framework, the agent can be informed of varying time\javad{-}elapsing variation budget\javad{s} based on \javad{the} training time $\Delta_\pi \in (0,T)$ even within the same time-elapsing MDP. Intuitively, \javad{a short time} $\Delta_\pi$ is inadequate to obtain a near-optimal policy, yet it facilitates frequent interactions with the environment, leading to a reduction in empirical model error due to \javad{a} larger volume of data. On the contrary, \javad{a long time} $\Delta_\pi$ may ensure obtaining a near-optimal policy but also \javad{introduces} greater uncertainty in \javad{learning} the environment. \javad{This motivates us to find a} \textbf{suboptimal training time} $\Delta^*_\pi \in (0,T)$ that strikes a balance between the sub-optimal gap of the policy and the empirical model error. If it exists, then $\Delta^*_{\pi}$ provides \javad{a} suboptimal $K^*= \lfloor T / \Delta^*_{\pi} \rfloor$, and \javad{a} suboptimal time sequence where $\mathfrak{t}^*_k = \mathfrak{t}_1 + \Delta^*_{\pi}\cdot(k-1)$ for \javad{all} $k \in [K^*]$. Our \texttt{ProST} framework computes \javad{the parameter} $\Delta^*_{\pi}$\javad{, then} sets $\{ \mathfrak{t}^* \}_{1:K^*}$, \javad{and finally} finds a \emph{future} near-optimal policy \javad{for} time $\mathfrak{t}^*_{k+1}$ at time $\mathfrak{t}^*_k$. \javad{In the next section, we first study} how to approximate \javad{the one-episode-ahead} suboptimal policy $\pi^{*,\mathfrak{t}_{k+1}}$ at time $\mathfrak{t}_k$ when $\{ \mathfrak{t} \}_{1:K}$ is given.

\section{Future policy optimizer}
\label{method}
\begin{figure}[ht]
    \centering
    \includegraphics[width=\textwidth]{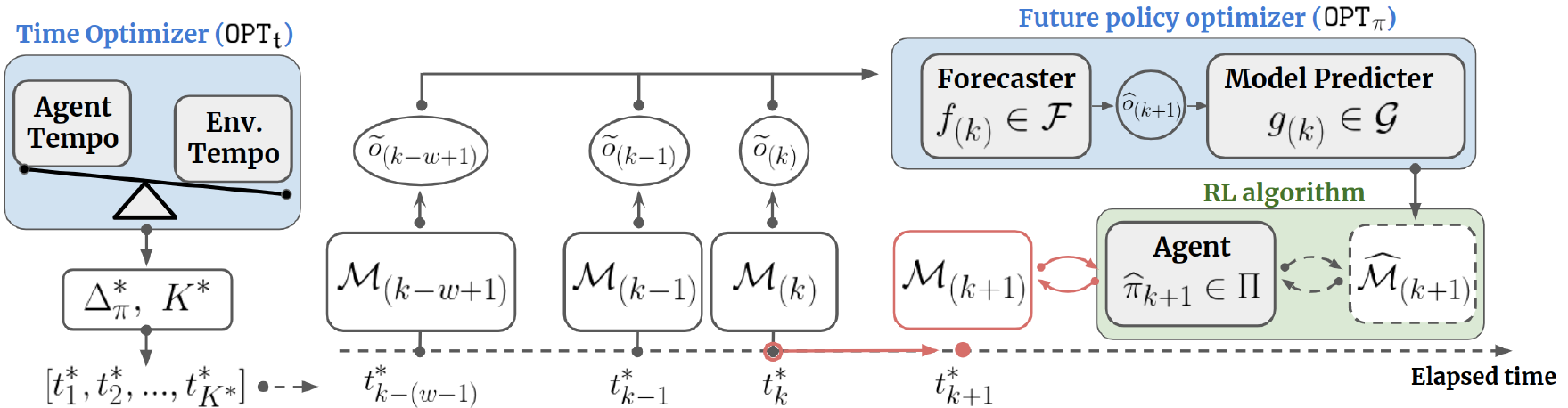}
    \caption{\texttt{ProST} framework}
    \label{fig:meta_alg}
\end{figure}

For given $\mathfrak{t}_{k}$ \javad{and} $\mathfrak{t}_{k+1}$, \javad{the} future policy optimizer ($\texttt{OPT}_{\pi}$), \javad{as a module of the} \texttt{ProST} framework \hyunin{(Figure \ref{fig:meta_alg})}, computes \javad{a near-}optimal policy for the future time $\mathfrak{t}_{k+1}$ at \javad{time} $\mathfrak{t}_k$ \javad{via} two procedures: \javad{(i) it first} forecasts the future MDP model of time $\mathfrak{t}_{k+1}$ at time $\mathfrak{t}_k$ utilizing the MDP forecaster function\javad{, (ii) it then} utilizes \javad{an \jj{arbitrary}} policy optimization algorithm within the forecasted MDP model $\texttt{OPT}_{\pi}$ to obtain \javad{a} future near-optimal policy $\pi^{*.\mathfrak{t}_{k+1}}$.

\subsection{MDP forecaster}
Our \texttt{ProST} framework is applicable in an environment that meets the following assumption.
\begin{assumption}[Observable non-stationary set $\mathcal{O}$]
    \javad{Assume that} \jj{the} non-stationarity of $\mathcal{M}_{\mathfrak{t}_k}$ \javad{is} fully characterized by a non-stationary \javad{parameter} $o_{\mathfrak{t}_k} \in \mathcal{O}$. \javad{Assume also that} the agent observes a noisy non-stationary \javad{parameter} $\tilde{o}_{\mathfrak{t}_k}$ \javad{at the end of episode $k \in [K]$ (at time $\mathfrak{t}_k$)}.
    \label{assume:observable_nonstatinoaryvar}
\end{assumption}
It is worth noting that  Assumption $\ref{assume:observable_nonstatinoaryvar}$ is mild, as prior research in non-stationary RL has proposed techniques to estimate $o_{(k)}$ through latent factor identification methods \cite{zintgraf2019varibad,chen2022adaptive,huang2021adarl,feng2022factored,kwon2021rl}, and our framework accommodates the incorporation of \javad{those works for the estimation of} $o_{(k)}$. Based on Assumption \ref{assume:observable_nonstatinoaryvar}, we define the MDP forecaster function $g \circ f$ \javad{below}.

\begin{definition}[MDP forecaster $g \circ f$]
    \javad{Consider two function} classes $\mathcal{F}$ and $\mathcal{G}$ \javad{such that} $\mathcal{F}: \mathcal{O}^w \rightarrow \mathcal{O} $ \javad{and} $\mathcal{G} : \mathcal{S} \times \mathcal{A} \times \mathcal{O} \rightarrow \mathbb{R} \times
 \Delta(\mathcal{S})$\hyunin{, where $w \in  \mathbb{N}$. Then, for} $f_{(k)} \in \mathcal{F}$ and $g_{(k)}\in {G}$, we define MDP forecaster at time $t_k$ as $(g \circ f)_{(k)} : \mathcal{O}^{w} \times \mathcal{S} \times \mathcal{A} \rightarrow \mathbb{R} \times \Delta(\mathcal{S})$.
 \label{def:MDPforecaster}
\end{definition}
The function $f_{(k)}$, \javad{acting as} a non-stationarity forecaster, predicts a non-stationary \javad{parameter} $\hat{o}_{(k+1)}$ at time $t_{k+1}$ based on \javad{the last $w$ observations given by the set} $\{\tilde{o}\}_{(k-w+1:k)}$, i.e.\javad{,} $\hat{o}_{(k+1)} = f(\{\tilde{o}\}_{(k-w+1,k)} )$. The agent can determine \jj{the} \javad{number of used historical observations}, denoted as $w$, by leveraging information from the environment (Section \ref{From_FT-MBPO_to_F-MBPO}). Then, the function $g_{(k)}$, \javad{acting as a} model predictor, predicts a reward $\widehat{R}_{(k+1)}(s,a)$ and a transition probability $\widehat{P}_{(k+1)}(\cdot|s,a)$ \javad{for} time $t_{k+1}$, i.e.\javad{,} $(\widehat{R}_{(k+1)},\widehat{P}_{(k+1)}) = g_{(k)}(s,a,\hat{o}_{k+1})$. Finally, the $\texttt{OPT}_\pi$ generates the estimated future MDP $\widehat{\mathcal{M}}_{(k+1)} = \langle \mathcal{S}, \mathcal{A}, H, \widehat{P}_{(k+1)} , \widehat{R}_{(k+1)}, \gamma \rangle $ \javad{associated with} time $t_{k+1}$.

\subsection{\javad{Finding} future optimal policy}

Now, consider an arbitrary RL algorithm \javad{provided by the user} to obtain \javad{an} optimal policy from the model $\widehat{\mathcal{M}}_{(k+1)}$. For a given time sequence $\{\mathfrak{t}\}_{1:K}$, the $\texttt{OPT}_\pi$ finds \jj{a} \javad{near-}optimal future policy as follows: (1) observe and forecast, (2) optimize \javad{using the} future MDP model. 

\textbf{(1) Observe and forecast. } At time $\mathfrak{t}_k$, the agent executes an episode $k$ in the environment $\mathcal{M}_{(k)}$, completes its trajectory $\tau_{(k)}$, and observes the noisy non-stationary \javad{parameter} $\hat{o}_{(k)}$ (Assumption \ref{assume:observable_nonstatinoaryvar}). The algorithm then updates the function $f_{(k)}$ based on \javad{the last $w$ observed parameters}, and the function $g_{(k)}$ with input from all previous trajectories. Following these updates, the MDP forecaster at time $t_k$ predicts $\widehat{P}_{(k+1)}$ and $\widehat{R}_{(k+1)}$, thus creating the MDP model $\widehat{\mathcal{M}}_{(k+1)}$ for time $\mathfrak{t}_{k+1}$.

\textbf{(2) Optimize \javad{using the} future MDP model. } Up until time $\mathfrak{t}_{k+1}$, the agent continually updates the policy within the estimated future MDP $\widehat{\mathcal{M}}_{(k+1)}$ for \javad{a} given duration $\Delta_\pi$. Specifically, the agent rollouts synthetic trajectories $\hat{\tau}_{(k+1)}$ in $\widehat{\mathcal{M}}_{(k+1)}$, \javad{and} utilizes any policy update algorithm to obtain a policy $\widehat{\pi}_{(k+1)}$. Following the duration $\Delta_\pi$, the agent stops \javad{the training by the time $\mathfrak{t}_{k+1}$} and moves to the next episode $\mathcal{M}_{(k+1)}$ with policy $\widehat{\pi}_{(k+1)}$. 

We elaborate on the \javad{above} procedure in Algorithm \ref{alg:meta_algorithm} \javad{given} in Appendix \ref{section:meta_algorithm}. 

\section{Time optimizer}
\label{From_FT-MBPO_to_F-MBPO}

\subsection{Theoretical analysis}

We now present our main theoretical contribution, \javad{which is regarding} the time optimizer ($\texttt{OPT}_{\mathfrak{t}}$): computing \javad{a} suboptimal policy training time $\Delta^*_{\pi}$ (the agent tempo). Our theoretical analysis starts \javad{with} specifying \javad{the} \hyunin{components} of the $\texttt{OPT}_{\mathfrak{\pi}}$ optimizer, which we refer to as \texttt{ProST-T} (\javad{note that} \texttt{-T} stands for an instance in \javad{the} tabular setting). We employ the Natural Policy Gradient (NPG) with entropy regularization \cite{kakade2001natural} as a policy update algorithm \javad{in} $\text{OPT}_{\pi}$. We denote the entropy regularization coefficient as $\tau$, the learning rate as $\eta$, the policy evaluation approximation gap arising due to finite samples \javad{as $\delta$}, and the past reference length for forecaster $f$ \javad{as $w$}. Without loss of generality, we assume that \javad{each} policy iteration takes \javad{one} second. The theoretical analysis is conducted within a tabular environment, allowing us to relax Assumption \ref{assume:observable_nonstatinoaryvar}, which means \javad{that one can} estimate non-stationary \javad{parameters} by counting visitation of state and action pairs at time $\mathfrak{t}_k$, denoted as $n_{(k)}(s,a)$, rather than \javad{observing} them. Additionally, we incorporate the exploration bonus term at time $\mathfrak{t}_k$ \javad{into $\widehat{R}_{(k+1)}$}, denoted as $\Gamma^{(k)}_w(s,a)$, which is proportional to $\sum_{\tau=k-w+1}^{k}(n_{(\tau)}(s,a))^{-1/2}$ \javad{and aims} to promote the exploration of states and actions that are \javad{visited infrequently}.

We compute $\Delta_\pi^*$ by minimizing \javad{an} upper bound \javad{on} the \texttt{ProST-T}'s dynamic regret. The dynamic regret of $\texttt{ProST-T}$ is characterized by \javad{the} \textit{model prediction error} \javad{that} measures the MDP forecaster's error by defining the difference between $\widehat{\mathcal{M}}_{(k+1)}$ and $\mathcal{M}_{(k+1)}$ through a Bellman equation.

\begin{definition}[Model prediction error]
    At time $\mathfrak{t}_k$, \javad{the} MDP forecaster predicts \javad{a} model $\widehat{\mathcal{M}}_{(k+1)}$ \jj{and} \javad{then we obtain a near-optimal policy $\widehat{\pi}^{(k+1)}$ based on $\widehat{\mathcal{M}}_{(k+1)}$}. For \javad{each pair} $(s,a)$, \jj{we} denote the state value function and the state action value function of $\widehat{\pi}^{(k+1)}$ in $\widehat{\mathcal{M}}_{(k+1)}$ at step $h\in [H]$ as $\widehat{V}^{(k+1)}_{h}(s)$ and $\widehat{Q}^{(k+1)}_{h}(s,a)$\javad{, respectively. We also} denote \javad{the} model prediction error \javad{associated with} time $\mathfrak{t}_{k+1}$ \javad{calculated} at time $\mathfrak{t}_k$ as $\iota^{(k+1)}_h(s,a)$\javad{, which is defined as}
    \begin{equation*}
        \iota^{(k+1)}_h(s,a) \javad{:}= \left( R_{(k+1)} + \gamma P_{(k+1)} \widehat{V}^{(k+1)}_{h+1} - \widehat{Q}^{(k+1)}_{h} \right) (s,a).
    \end{equation*}
\end{definition}

We now \javad{derive an} upper bound \javad{on the} \texttt{ProST-T} dynamic regret. We expect the upper bound to be likely controlled by two factors: the error of the MDP forecaster's \javad{prediction of} the future MDP model and the error of the NPG algorithm \javad{due to approximating} the optimal policy within an estimated future MDP model. This insight is clearly articulated \javad{in the next theorem}. 

\begin{theorem}[\texttt{ProST-T} dynamic regret $\mathfrak{R}$]
    Let $\iota^{K}_{H} = \sum_{k=1}^{K-1} \sum_{h=0}^{H-1} \iota^{(k+1)}_{h}(s^{(k+1)}_h,a^{(k+1)}_h) $ and $\bar{\iota}_{\infty}^{K} \coloneqq \sum_{k=1}^{K-1} ||\bar{\iota}_{\infty}^{k+1}||_\infty$, where $\iota^{K}_{H}$ is a data-dependent error. For \javad{a} given $p \in (0,1)$, the dynamic regret of the forecasted policies $\{ \widehat{\pi}^{(k+1)} \}_{1:K-1}$ of \texttt{ProST-T} is upper bounded with probability \javad{at least} $1-p/2$ \javad{as follows:}
     \begin{equation*}
         \mathfrak{R} \left( \{ \widehat{\pi}^{(k+1)}\}_{1:K-1},K) \right) \leq \mathfrak{R}_{I} + \mathfrak{R}_{II}
     \end{equation*}
     where $\mathfrak{R}_{I} = \bar{\iota}_{\infty}^{K}/ (1-\gamma) - \iota^{K}_H \add{+ \hyunin{C_{p}} \cdot \sqrt{K-1}},~~\mathfrak{R}_ {II} = C_{II}[\Delta_\pi] \cdot (K-1)$, and $\hyunin{C_p},C_{II}[\Delta_\pi]$ are \javad{some} functions of $p$, $\Delta_\pi$, respectively.
    \label{theorem1}
\end{theorem}

Specifically, the upper bound is composed of two terms: $\mathfrak{R}_{I}$ \javad{that} originates from \javad{the} MDP forecaster error between $\mathcal{M}_{(k+1)}$ and $\widehat{\mathcal{M}}_{(k+1)}$, and $\mathfrak{R}_{II}$ \javad{that} arises due to the suboptimality gap between $\pi^{*,(k+1)}$ and $\widehat{\pi}^{(k+1)}$. Theorem \ref{theorem1} clearly demonstrates that \javad{a} prudent construction of the MDP forecaster that controls the model prediction errors and the selection of the agent tempo $\Delta_\pi$ is significant in guaranteeing \javad{sublinear rates for} $\mathfrak{R}_{I}$ \javad{and} $\mathfrak{R}_{II}$. To \javad{understand the role of} the environment tempo \javad{in} $\mathfrak{R}_{I}$, we observe that the MDP forecaster utilizes $w$ previous observations\javad{,} which inherently encapsulates the environment tempo. \javad{We} expect the model prediction errors, at least in part, to be controlled by the environment tempo $B(\Delta_\pi)$, so \javad{that a} trade-off between two tempos can be framed as the trade-off between $\mathfrak{R}_{I}$ and $\mathfrak{R}_{II}$. \javad{Hence, it is desirable to somehow minimize} the upper bound \javad{with respect to $\Delta_{\pi}$ to obtain a solution,} denoted as  $\Delta^*_{\pi}$ \jj{, which} strikes a balance between $\mathfrak{R}_{I}$ and $\mathfrak{R}_{II}$. 

\subsubsection[RII analysis]{Analysis of $\mathfrak{R}_{II}$} 

\javad{A direct analysis of the upper bound $\mathfrak{R}_I +\mathfrak{R}_{II}$ is difficult since its dependence on $K$ is not explicit. To address this issue, we recall that an optimal $\Delta^*_\pi$ should be a natural number that guarantees the sublinearity of both $\mathfrak{R}_{I}$ and $\mathfrak{R}_{II}$ with respect to the total number of episodes $K$.} We first compute \javad{a} set $\mathbb{N}_{II} \subset \mathbb{N}$ \javad{that includes those values of} $\Delta_\pi$ \javad{that} guarantee \javad{the sublinearity of} $\mathfrak{R}_{II}$\javad{, and then compute a} set $\mathbb{N}_{I} \subset \mathbb{N}$ that guarantees \javad{the sublinearity of} $\mathfrak{R}_{I}$\javad{. Finally, we solve for} $\Delta^*_\pi$ in \javad{the} common set $\mathbb{N}_{I} \cap \mathbb{N}_{II}$.

\begin{proposition}[$\Delta_\pi$ bounds for sublinear $\mathfrak{R}_{II}$]
    A \hyunin{total step} $H$ is given by MDP. For \javad{a number} $\epsilon > 0$ \javad{such that} $H=\hyunin{\Omega} \left( \log \left( ( \hyunin{\widehat{r}_{\text{max}}} \vee \hyunin{r_{\text{max}}} ) / \epsilon \right) \right)$, we choose $\delta,\tau,\eta$ to satisfy $\delta =\mathcal{O} \left( \epsilon \right),~ \tau =\Omega \left( \epsilon / \log |\mathcal{A}| \right)$ \javad{and} $\eta \leq \left(1-\gamma\right) / \tau$\hyunin{, where $\widehat{r}_{\text{max}}$ and $r_{\text{max}}$ are \jj{the} maximum reward of the forecasted model and \jj{the} maximum reward of the environment\jj{, respectively}}. \javad{Define} $\mathbb{N}_{II} \javad{:=} \{n~|~ n > \frac{1}{\eta \tau} \log \left( \frac{\hyunin{C_1}(\gamma+2)}{\epsilon} \right), n \in \mathbb{N}\}$\hyunin{, where $C_1$ is a constant}. Then $ \mathfrak{R}_{II} \leq 4 \epsilon (K-1)$ \javad{for all $\Delta_{\pi} \in \mathbb{N}_{II}$}.
    \label{proposition1}
\end{proposition}

\javad{As a by-product of Proposition \ref{proposition1}, the sublinearity of} $\mathfrak{R}_{II}$ can be realized \javad{by choosing} $\epsilon = \mathcal{O}((K-1)^{\alpha-1})$ for any $\alpha \in [0,1)$, which suggests that a tighter upper bound \javad{on} $\mathfrak{R}_{II}$ requires a smaller $\epsilon$ \javad{and subsequently} a larger $\Delta_\pi \in \mathbb{N}_{II}$. The hyperparameter conditions in Proposition \ref{proposition1} can be found in Lemma \ref{lemma:optHyp_without_entropy} and \ref{lemma:optHyp_with_entropy} in Appendix \ref{appendix:proof}. 

\subsubsection[RI analysis]{Analysis of $\mathfrak{R}_{I}$}

We now \javad{relate $\mathfrak{R}_{I}$ to} the environment tempo $B(\Delta_\pi)$ using the well-established non-stationary adaptation technique \javad{of} Sliding Window regularized Least\javad{-}Squares Estimator (\texttt{SW-LSE}) as \javad{the} MDP forecaster \cite{cheung2019learning,cheung2022hedging,cheung2020reinforcement}. The tractability of \javad{the} \texttt{SW-LSE} algorithm allows to upper\javad{-}bound the model predictions errors $\iota^K_H$ \javad{and} $\bar{\iota}^K_\infty$ by \javad{the} environment tempo extracted from the past $w$ observed trajectories, leading to \javad{a} sublinear $\mathfrak{R}_{I}$ as demonstrated in the following theorem.

\begin{theorem}[Dynamic regret $\mathfrak{R}_{I}$ \javad{with} $f=$ \texttt{SW-LSE}]
\label{theorem2}
For \javad{a} given $p\in (0,1)$, if the exploration bonus constant $\beta$ and regularization parameter $\lambda$ satisfy $
        \beta = \Omega(|\mathcal{S}| H \sqrt{\log \left(  H / p \right) } )$ and $\lambda \geq 1$, then $\mathfrak{R}_I$ is bounded with probability \javad{at least} $1-p$ \javad{as follows:}
\begin{align*}
    \mathfrak{R}_{I} \leq  C_{I}[B(\Delta_\pi)]\cdot w + C_k \cdot \sqrt{\frac{1}{w}\log{\left( 1+ \frac{H}{\lambda}w \right)}} + \hyunin{C_p} \cdot \sqrt{K-1}\label{eqn:R_I_upper}
\end{align*}
where $C_{I}[B(\Delta_\pi)] = \left(1/(1-\gamma) + H \right) \cdot B_r(\Delta_\pi) + (1+H\hat{r}_{\text{max}})\gamma / (1-\gamma) \cdot B_p(\Delta_\pi)$, and $C_k$ is a constant on the order of $\mathcal{O}(K)$.
\end{theorem}

For a brief sketch of how \texttt{SW-LSE} \javad{makes the} environment tempo \javad{appear in the} upper bound, we outline that \javad{the} model prediction errors are \javad{upper-bounded} by two forecaster errors\javad{, namely} $P_{(k+1)}$ - $\widehat{P}_{(k+1)}$ and $R_{(k+1)} - \widehat{R}_{(k+1)}$, along with the visitation count $n_{(k)}(s,a)$. Then\javad{, the} \texttt{SW-LSE} algorithm provides a solution  $\left(\widehat{P}_{(k+1)},\widehat{R}_{(k+1)} \right)$ as a closed form of linear combinations of \javad{past $w$} estimated values $\{\widetilde{P},\widetilde{R}\}_{(k-w+1:w)}$. Finally, employing the Cauchy inequality and triangle inequality, we derive two forecasting errors \javad{that are upper-bounded} by the environment tempo. For the final step before obtaining \javad{a sub}optimal $\Delta^*_\pi$, we compute $\mathbb{N}_{I}$ that guarantees the \javad{sublinearity of}  $\mathfrak{R}_I$.

\begin{proposition}[$\Delta_\pi$ bounds for sublinear $\mathfrak{R}_{I}$] Denote $B(1)$ as \javad{the} environment tempo when $\Delta_\pi=1$, which is a summation over \javad{all} time steps. \javad{Assume that the} environment satisfies $B_r(1) +  B_p(1) \hat{r}_{\text{max}}/(1-\gamma) = o(K)$ and we choose  $w = \mathcal{O}((K-1)^{2/3}/( C_{I}[B(\Delta_\pi)])^{2/3})$. \javad{Define the set} $\mathbb{N}_{I}$ to be $\{ n~|~n < K ,~n \in \mathbb{N} \}$. \javad{Then} $\mathfrak{R}_{I}$ is \javad{upper-bounded} as  $\mathfrak{R}_I = \mathcal{O} \left(  C_{I}[B(\Delta_\pi)]^{1/3} \left(K-1\right)^{2/3} \sqrt{\log{\left((K-1)/ C_{I}[B(\Delta_\pi)] \right)}}\right) \nonumber$ and also satisfies \javad{a} sublinear upper bound\javad{, provided that $\Delta_\pi \in \mathbb{N}_I$}\jj{.}
\label{proposition2}
\end{proposition}

The upper bound on \javad{the} environment tempo $B(1)$ in proposition \ref{proposition2} \javad{is aligned} with our expectation that dedicating an excessively long time to a single iteration \javad{may} not allow for \javad{an} effective policy approximation, thereby hindering the achievement of a sublinear dynamic regret. Furthermore, our insight that a larger environment tempo prompts the MDP forecaster to consider a shorter past reference length, aiming to mitigate forecasting uncertainty, is consistent with the condition involving $w$ stated in Proposition \ref{proposition2}.

\subsubsection[Optimal tempo Delta*pi]{\add{Sub}optimal tempo $\Delta^*_\pi$}

So far, we \javad{have shown} that \javad{an} upper bound \javad{on the} \texttt{ProST} dynamic regret is composed of two \javad{terms} $\mathfrak{R}_{I}$ and $\mathfrak{R}_{II}$\javad{, which} are characterized by \javad{the} environment tempo and the agent tempo, respectively. Now, we claim that \javad{a} suboptimal tempo that minimizes \texttt{ProST}'s dynamic regret could be obtained by the optimal solution $\Delta^*_\pi=\argmin_{\Delta_\pi \in \mathbb{N}_{I} \cap \mathbb{N}_{II}} \left(\mathfrak{R}^{\text{max}}_I+\mathfrak{R}^{\text{max}}_{II}\right)$\javad{,} where $\mathfrak{R}^{\text{max}}_I$ \javad{and} $\mathfrak{R}^{\text{max}}_{II}$ \jj{denote} the upper bound\jj{s on} $\mathfrak{R}_I$  \jj{and} $\mathfrak{R}_{II}$. We show \javad{that} $\Delta^*_\pi$ strikes a balance between the environment tempo and the agent tempo since $\mathfrak{R}^{\text{max}}_{I}$ is \javad{a} non-decreasing function of $\Delta_\pi$ and $\mathfrak{R}^{\text{max}}_{II}$ is \javad{a} non-increasing function of $\Delta_\pi$. Theorem \ref{theorem3} shows that the optimal tempo $\Delta^*_\pi$ depends on the environment's drifting constants \javad{introduced in} Assumption \ref{assum:tevb}. 

\begin{theorem}[Suboptimal tempo $\Delta^*_\pi$]
Let $k_{\texttt{Env}} =  \left( \alpha_r \vee \alpha_p \right)^2 C_I[B(1)]$, $k_{\texttt{Agent}} = \log{\left(1/(1-\eta \tau)\right)} C_1 (K-1)(\gamma+2)$. \jj{Consider three cases:} \textbf{case1}: $\alpha_r \vee \alpha_p=0$, \textbf{case2}: $\alpha_r \vee \alpha_p=1$, \textbf{case3}: $0<\alpha_r \vee \alpha_p<1$ \hyunin{or $\alpha_r \vee \alpha_p>1$}. Then $\Delta_\pi^*$ depends on the environment's drifting constants \jj{as follows:}
\begin{itemize}
    \item Case1: $\Delta^{*}_{\pi} =\add{T}$.
    \item Case2: $\Delta^{*}_{\pi} = \log_{1-\eta\gamma}{({k_{\texttt{Env}}/k_{\texttt{Agent}})}}+1$.
    \item Case3: $\Delta^{*}_{\pi}  = \exp{\left(-W \left[-  \frac{ \log{(1-\eta\tau)}}{\max{(\alpha_r,\alpha_p)}-1} \right]\right)}$, provided that the parameters are chosen so that $k_{\texttt{Agent}} = (1-\eta\tau)k_{\texttt{Env}}$.
\end{itemize}
    \label{theorem3}
\end{theorem}



\subsection{\javad{Improving} MDP forecaster}
\label{imporve_mdp_forecaster}
Determining \javad{a} suboptimal tempo by \javad{minimizing an} upper bound \jj{on} $\mathfrak{R}_{I}+\mathfrak{R}_{II}$ \javad{can be improved by using a} tighter upper bound. \javad{In} Proposition \ref{proposition1}, we \javad{focused} on the $Q$ approximation gap $\delta$ to provide \javad{a justifiable upper bound on} $\mathfrak{R}_{I} + \mathfrak{R}_{II}$. It is important to note that the factor $\delta$ arises not only from the finite sample trajectories as discussed in \cite{cen2022fast}, but also \javad{from the} forecasting error between $\mathcal{M}_{(k+1)}$ and $\widehat{\mathcal{M}}_{(k+1)}$. It is clear that the MDP forecaster establishes a lower bound \javad{on} $\delta$ denoted as $\delta_{\text{min}}$, which in turn sets a lower bound on $\epsilon$ and consequently on $\mathfrak{R}_I$. This inspection highlights that the MDP forecaster serves as a common factor that controls both $\mathfrak{R}_I$ and $\mathfrak{R}_{II}$, and \javad{a} further investigation to improve the accuracy of the forecaster is necessary for \javad{a better bounding \jj{on}} $\mathfrak{R}_{I}+\mathfrak{R}_{II}$.

Our approach to devising a precise MDP forecaster is \javad{that}, instead of \textit{selecting} the past reference length $w$ as indicated in Proposition \ref{proposition2}, \javad{we} set $w=k$, implying the utilization of all past observations. However, we address this by solving an additional optimization problem, resulting in a tighter \javad{bound on} $\mathfrak{R}_{I}$. We propose a method that adaptively assigns different weights $q \in \mathbb{R}_+^{k}$ to \javad{the} previously observed non-stationary \javad{parameters} up to time $t_k$, which reduces the burden of choosing $w$. Hence, we \javad{further analyze} $\mathfrak{R}_I$ through the utilization of the Weighted regularized Least-Squares Estimator (\texttt{W-LSE}) \cite{kuznetsov2018theory}. Unlike \texttt{SW-LSE}, \texttt{W-LSE} does not necessitate a predefined selection of $w$, \javad{but it instead} engages in a joint optimization procedure involving \javad{the} data weights $q$ and \javad{the} future model $\left(\widehat{P}_{(k+1)},\widehat{R}_{(k+1)}\right)$. \javad{To this end}, we define the forecasting reward model error as $\Delta^{r}_{k}(s,a)=\left\vert \left(R_{(k+1)} - \widehat{R}_{(k+1)} \right) (s,a) \right\vert $ \javad{and the} forecasting transition probability model error as $\Delta^{p}_{k}(s,a)=\left\vert \left\vert \left( P_{(k+1)} - \widehat{P}_{(k+1)} \right)(\cdot~|~s,a) \right\vert \right\vert_1$.
\begin{theorem}[$\mathfrak{R}_{I}$ upper bound \javad{with} $f$=\texttt{W-LSE}]
    \javad{By setting the} exploration bonus $\Gamma_{(k)}(s,a) = \frac{1}{2} \Delta^{r}_{k}(s,a) + \frac{\gamma \tilde{r}_{\text{max}}}{2(1-\gamma)}  \Delta^{p}_{k}(s,a)$, \javad{it holds that}
    \begin{align*}
         \mathfrak{R}_{I} \leq \left( 4H + \frac{2\gamma \left| \mathcal{S} \right| }{1-\gamma}  \left( \frac{1}{1-\gamma} + H\right)  \right) \left(  \frac{1}{2} \sum_{k=1}^{K-1} \Delta^{r}_{k}(s,a)  + \frac{\gamma \tilde{r}_{\text{max}}}{2(1-\gamma)}  \sum_{k=1}^{K-1} \Delta^{p}_{k}(s,a) \right).
    \end{align*}
    \label{thm:3}
\end{theorem}

\begin{remark}[Tighter $\mathfrak{R}_{\mathcal{I}}$ upper bound \javad{with} $f=$ \texttt{W-LSE}]
\label{rmk:W-LSE_modelerror}
If the \hyunin{optimization problem} of \texttt{W-LSE} is feasible, then the optimal data weight $q^{*}$ provides \javad{tighter bounds} for $\Delta^{r}_{k}$ \javad{and} $\Delta^{p}_{k}$ in comparison to \texttt{SW-LSE}, consequently leading to a tighter upper bound \javad{for  $\mathfrak{R}_{\mathcal{I}}$}. We \javad{prove in Lemmas \ref{corollary1} and \ref{corollary3} in Appendix \ref{appendix:proof} that $\bar{\iota}^K_{\infty}$ and $-\iota^K_{H}$ are upper-bounded in terms of $\Delta^{r}_{k}$ and $\Delta^{p}_{k}$.}
\end{remark}

\subsection{ProST-G}

The theoretical analysis outlined above serves as \javad{a} motivation to empirically investigate two key points: firstly, the existence of an optimal training time; secondly, the role of the MDP forecaster's contribution to the \texttt{ProST} framework's overall performance. To address these questions, we propose a practical instance, \javad{named} \texttt{ProST-G}, which particularly extends the investigation in Section \ref{imporve_mdp_forecaster}. \texttt{ProST-G} optimizes a policy with \javad{the} soft actor-critic (SAC) algorithm \cite{haarnoja2018soft}, utilizes \javad{the} integrated autoregressive integrated moving average (ARIMA) model for the proactive forecaster $f$, and uses a bootstrap ensemble of dynamic models where each model is a probabilistic neural network for the model predictor $g$. We further discuss specific details of \texttt{ProST-G} in Appendix \ref{section:F-MBPO} and in Algorithm \ref{alg:F-MBPO}.
\section{Experiments}
\label{experiments}
We evaluate \texttt{ProST-G} with four baselines in three Mujoco environments each with five different non-stationary speeds and two non-stationary datasets.

\textbf{(1) Environments: Non-stationary desired posture. }
We \javad{make} the rewards in the three environments non-stationary by altering the agent's desired directions. \javad{The} forward reward $R^{f}_t$ changes as $R^{f}_{t} = o_{t} \cdot \widebar{R}^{f}_{t}$, where $\widebar{R}_f$ is the original reward from the Mujoco environment. \javad{The} non-stationary \javad{parameter} $o_k$ is generated from \javad{the sine} function with five different speeds and from the real data $A$ and $B$. We then measure the time-elapsing variation budget by $\sum_{k=1}^{K-1}|o_{k+1}-o_k|$. Further details of the environment settings can be found in Appendix \ref{appendix:Environment_setting}. 

\textbf{(2) Benchmark methods. }
Four baselines are chosen to empirically support our second question: the significance of the forecaster. \textbf{MBPO} is the state-of-the-art model-based policy optimization \cite{janner2019trust}. \textbf{Pro-OLS} is a policy gradient algorithm that predicts the future performance and optimizes the predicted performance of the future episode \cite{chandak2020optimizing}. \textbf{ONPG} is an adaptive algorithm that performs a purely online optimization by fine-tuning the existing policy using only the trajectory observed online \cite{al2017continuous}. \textbf{FTRL} is an adaptive algorithm that performs follow-the-regularized-leader optimization by maximizing the performance on all previous trajectories \cite{finn2019online}.


\section{Discussions}

\subsection{Performance compare}

The \javad{outcomes} of the experimental results \javad{are} presented in Table \ref{table1:rewards}. The table summarizes the average return over the last 10 episodes during the training procedure. We have illustrated the complete training results in Appendix \ref{subsection:Full results}. In most cases, \texttt{ProST-G} outperforms MBPO in terms of rewards, highlighting the adaptability of \javad{the} \texttt{ProST} framework to dynamic environments. Furthermore, except for data $A$ and $B$, \texttt{ProST-G} consistently outperforms the other three baselines. This supports our motivation of using the proactive model-based method for \javad{a} higher adaptability in non-stationary environments compared to state-of-the-art model-free algorithms (Pro-OLS, ONPG, FTRL). We elaborate on the training details in Appendix \ref{appendix:hyperparameters and implementation details}.

\begin{table}[ht]
  \centering
  \caption{Average reward returns}
  \resizebox{\columnwidth}{!}{%
  \begin{tabular}{ccccccccccccccccc}
    \toprule
    \textbf{Speed} & $B(G)$ & \multicolumn{5}{c}{\textbf{Swimmer-v2}} & \multicolumn{5}{c}{\textbf{Halfcheetah-v2}} & \multicolumn{5}{c}{\textbf{Hopper-v2}} \\
    \cmidrule(rl){3-7} \cmidrule(rl){8-12} \cmidrule(rl){13-17}
    & & \textbf{Pro-OLS} & \textbf{ONPG} & \textbf{FTML} & \textbf{MBPO} & \textbf{\texttt{ProST-G}} & \textbf{Pro-OLS} & \textbf{ONPG} & \textbf{FTML} & \textbf{MBPO} & \textbf{\texttt{ProST-G}} & \textbf{Pro-OLS} & \textbf{ONPG} & \textbf{FTML} & \textbf{MBPO} & \textbf{\texttt{ProST-G}} \\
    \midrule
    1 & 16.14 & -0.40 & -0.26 & -0.08 & -0.08 & \textbf{0.57} & -83.79 & -85.33 & -85.17 & -24.89 & \textbf{-19.69} & \textbf{98.38} & 95.39 & 97.18 & 92.88 & 92.77 \\
    2 & 32.15 & 0.20 & -0.12 & 0.14 & -0.01 & \textbf{1.04} & -83.79 & -85.63 & -86.46 & -22.19 & \textbf{-20.21} & 98.78 & 97.34 & \textbf{99.02} & 96.55 & 98.13 \\
    3 & 47.86 & -0.13 & 0.05 & -0.15 & -0.64 & \textbf{1.52} & -83.27 & -85.97 & -86.26 & -21.65 & \textbf{-21.04} & 97.70 & 98.18 & 98.60 & 95.08 & \textbf{100.42} \\
    4 & 63.14 & -0.22 & -0.09 & -0.11 & -0.04 & \textbf{2.01} & -82.92 & -84.37 & -85.11 & -21.40 & \textbf{-19.55} & 98.89 & 97.43 & 97.94 & 97.86 & \textbf{100.68} \\
    5 & 77.88 & -0.23 & -0.42 & -0.27 & 0.10 & \textbf{2.81} & -84.73 & -85.42 & -87.02 & \textbf{-20.50} & -20.52 & 97.63 & 99.64 & 99.40 & 96.86 & \textbf{102.48} \\
    A & 8.34 & 1.46 & 2.10 & \textbf{2.37} & -0.08 & 0.57 & -76.67 & -85.38 & -83.83 & -40.67 & \textbf{83.74} & 104.72 & \textbf{118.97} & 115.21 & 100.29 & 111.36 \\
    B & 4.68 & \textbf{1.79} & -0.72 & -1.20 & 0.19 & 0.20 & -80.46 & -86.96 & -85.59 & -29.28 & \textbf{76.56} & 80.83 & \textbf{131.23} & 110.09 & 100.29 & 127.74 \\
    \bottomrule
  \end{tabular}%
  }
  \label{table1:rewards}
\end{table}

\subsection{Ablation study}
An ablation study was conducted on the two aforementioned questions. The following results support our inspection of \javad{Section} \ref{imporve_mdp_forecaster} and provide strong grounds for Theorem \ref{theorem3}.

\begin{figure}[h]
    \centering
    \includegraphics[width=\textwidth]{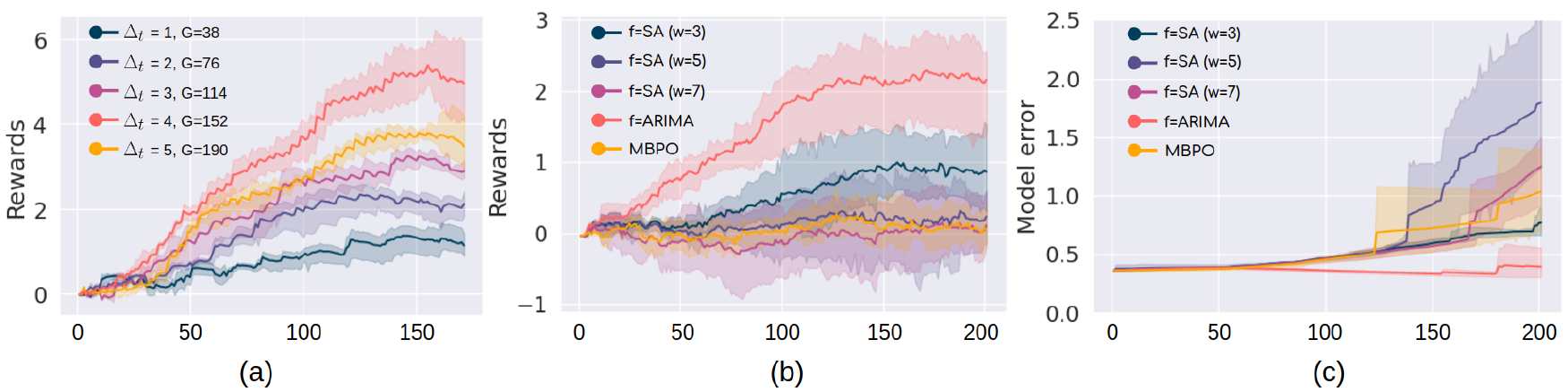}
    \caption{(a) Optimal $\Delta_\pi^*$\javad{;} (b) Different forecaster $f$ (ARIMA, SA)\javad{;} (c) The Mean squared Error (MSE) model loss of \texttt{ProST-G} with \javad{four} different forecasters (ARIMA and three SA) and the MBPO. \javad{The $x$-axis in each figure shows the episodes.}}
    \label{fig:ablation}
\end{figure}

\textbf{\add{Sub}optimal $\Delta^*_\pi$. }  The experiments are performed over five different policy training times \javad{$\Delta_\pi \in \{1,2,3,4,5\}$}, aligned with SAC's \hyunin{number of gradient steps $ G \in \{38, 76, 114, 152, 190\}$}, under a fixed environment speed. Different from \javad{our} theoretical analysis, we \javad{set} $\Delta_t=1$ with \hyunin{$G=38$}. We generate $o_k = sin(2 \pi \Delta_\pi k/ 37 )$\javad{,} which satisfies Assumption \ref{assum:tevb} (see Appendix \ref{appendix:environment setting details}). The shaded area\javad{s} of Figure\javad{s} \ref{fig:ablation} \javad{(a), (b) and (c)} are 95 \% confidence area among three different noise bounds \javad{of 0.01,0.02 and 0.03} in $o_k$. Figure \ref{fig:ablation}\javad{(a)} shows $\Delta_t=4$ is close to the optimal $G^*$ among five different choices.

\textbf{Function\javad{s} $f,g$. } We investigate the effect of \javad{the} forecaster $f$'s accuracy on the framework using two distinct functions: ARIMA and a simple average (SA) model, each tested with three different \javad{the values of $w$}. Figure \ref{fig:ablation}\javad{(b)} shows the average rewards of the SA model with \javad{$w \in \{3,5,7\}$} and ARIMA model (four solid lines). The shaded area is 95 \% the confidence area among 4 different speeds \javad{$\{1,2,3,4\}$}. Figure \ref{fig:ablation}\javad{(c)} shows the corresponding model error. Also, we investigate the effect of \javad{the} different model \javad{predictor} $g$ by comparing MBPO (reactive-model) and \texttt{ProST-G} with $f=$ARIMA (proactive-model) in Figure \ref{fig:ablation}\javad{(c)}. The high returns from \texttt{ProST-G} with $f=$ ARIMA, compared to those from MBPO\javad{, empirically} support that the forecasting component of \javad{the} $\textbf{ProST}$ framework can provide \javad{a satisfactory} adaptability to the baseline algorithm that is equipped with. Also, Figure\javad{s} \ref{fig:ablation}\javad{(b)} and \ref{fig:ablation}\javad{(c)} provide empirical evidence that \javad{the accuracy of $f$} is contingent on the sliding window size, thereby impacting \javad{the} model accuracy and subsequently influencing the agent's performance.

\section{Conclusion}

\javad{This work offers the first study on the important issue of} time synchronization \javad{for} non-stationary RL\javad{. To this end, we introduce the concept of} the tempo of adaptation in a non-stationary RL, \javad{and obtain a } suboptimal training time. \javad{We} propose a Proactively Synchronizing Tempo (ProST) framework, \javad{together with} two specific instances  \texttt{ProST-T}  and \texttt{ProST-G}. \javad{The} proposed method adjusts an agent's tempo to match the tempo of the environment to \javad{handle} non-stationarity through \javad{both} theoretical analysis and empirical evidence. \javad{The} ProST framework provides a new avenue to implement reinforcement learning in the real world by incorporating the concept of adaptation tempo\javad{.}

\javad{As a future work, it is important to generalize the proposed framework} to learn a safe guarantee policy in a non-stationary RL by considering the adaptation tempo of constraint violations \cite{jin2020stability,pfrommer2022safe}. Another \javad{generalization} is to introduce \javad{an alternative} dynamic regret metric, \hyunin{enabling a fair performance comparison among agents, even when they have varying numbers of total episodes.} Another \javad{future work} is to find an optimal tempo of the distribution correction in offline non-stationary RL, specifically how to adjust the relabeling function to offline data in a time-varying environment that is dependent on the tempo of the environment \cite{lee2022coptidice,lee2021optidice}. 

\bibliography{nips} 

\appendix

\newpage

\section{Details on Introduction}
\subsection{Experimental motivation}
\label{appendix: 2D reacher environment setting}
\textbf{1. Environment details of 2D goal reacher}
\begin{itemize}
    \item State space: $\mathcal{S} = \mathbb{R}^2$. For $(x,y) \in \mathcal{S}, |x|\leq 1, |y| \leq 1.$
    \item Action space: $\mathcal{A} = \{(\cos{(\pi/4 \jj{\times} k)},\sin{(\pi/4 \jj{\times} k)}) ~|~k = 0,1\jj{,...,}7 \}$ ($|\mathcal{A}|=8$)
    \item Reward function: if \jj{the} agent's state is in the Goal box, then \jj{it} receives +6. Otherwise, \jj{it} receives -0.5 rewards for every step.
    \item Transition probability: $s_{h+1} = s_h + a_h \cdot \epsilon$\jj{, where} $s_{h+1}$ is \jj{the} next state, $s_h$ is \jj{the} current state, $a_h$ is the current action, \jj{and} $\epsilon \in \mathbb{R}^2$ \jj{with} $ ||\epsilon||_2=1$ provides a stocasticity to the environment. 
    \item Horizon length: $H=13$
    \item Discounting factor: $\gamma=0.99$
    \item Grid size: 10
    \item Goal \jj{box:} The \jj{coordinates of the center} of the time-varying goal box \jj{are} $(x_g,y_g) $$=  (0.9\cos{(2\pi \jj{\times} k/2500)}, 0.9\sin{(2\pi \jj{\times} k/2500)})$\jj{, which} changes for episode $k \in [5000]$. The width and height \jj{of the box are equal to} 0.05\jj{.}
\end{itemize}

\textbf{2. Experiment details}

To motivate our proposed meta-framework \jj{via} a simple experiment, we used Q-learning as a component $A$ of our meta-algorithm to update the policy. The three baselines (ProOLS, ONPG, FTML) of Figure \ref{fig:fig1}\jj{(c)} were trained with four learning rates $\eta \in \{0.001, 0.003, 0.005, 0.007\}$ \jj{and the} entropy regularized parameter $ \tau = 0.1$, \jj{where} the shaded area of the three baselines is 95 \% confidence area among 4 different learning rates.  The \texttt{PTM-T} was trained with \jj{the} model rollout length $\widehat{H} \in \{ 50,60\}$, policy update iteration number $G \in \{10,50\}$, entropy regularized parameter $ \tau = 0.1$, Q\jj{-}learning update parameter $\alpha_Q \in \{0.7,0.9,0.99\}$\jj{, and} the learning rate $\eta=0.001$. The shaded area of \texttt{PTM-T} is 95 \% confidence area among the 12 different cases above. All four algorithms \hh{share} the same agent's policy network \hh{structure}. 

\section{Related \jj{W}orks}
\label{related_works}

Existing methods for non-stationary environments can be grouped into three \jj{categories}: 1) shoehorning: directly using established frameworks for stationary MDPs \jj{by} assuming no extra mechanisms are needed since non-stationarity already exists in standard RL due to policy updates; 2) model-based policy updates: updating models with new data, using short rollouts to prevent model exploitation \cite{janner2019trust,hafner2023mastering}, online model updates, \jj{or} latent factor identification \cite{zintgraf2019varibad,chen2022adaptive,huang2021adarl,feng2022factored,kwon2021rl}; and 3) anticipating future changes by forecasting policy gradients or value functions \cite{chandak2020optimizing, mao2020model, cheung2020reinforcement, ding2022non, ding2022provably}.

The advantage of the model-free method is its computational efficiency, allowing for direct learning of complex policies from raw data \cite{Mnih2015HumanlevelCT,lillicrap2015continuous}, while the advantage of the model-based method is its data efficiency, allowing one to learn fast by learning how the environment works \cite{sun2018dual,luo2018algorithmic}. However, both advantages are weakened in non-stationary environments since the optimizing non-stationary loss function induced by time-varying data distribution makes the model-free method challenging to adaptively obtain the optimal policy \cite{fei2020dynamic,chandak2020towards} and the model-based method challenging to estimate accurate non-stationary models \cite{cheung2020reinforcement,ding2022non}.

\textbf{Model-free method in non-stationary RL. }
\cite{al2017continuous} uses meta-learning among the training tasks to find initial hyperparameters of the policy networks that can be quickly fine-tuned when facing testing tasks that have not been encountered before. However, access to a prior distribution of training tasks is not available in real-world problems. To mitigate this issue, \cite{finn2019online} proposed the Follow-The-Meta-Leader (FTML) algorithm that continuously improves an initialization of parameters for non-stationary input data. However, it internally entails a lag when tracking optimal policy as it maximizes the current performance over all the past samples uniformly. To alleviate the lag problem, \cite{chandak2020optimizing,chandak2020towards} \jj{focused} on directly forecasting the non-stationary performance gradient to adapt the time-varying optimal policies. However, it still has problems of showing empirical analysis on bandit settings or a low-dimensional environment and lack of theoretical analysis which provides a bound on the adapted policy's performance. \cite{mao2020model} proposed adaptive Q-learning with a restart strategy and established its near-optimal dynamic regret bound. In addition, \cite{fei2020dynamic} proposed two model-free policy optimization algorithms based on the restart strategy and showed that dynamic regret satisfies polynomial space and time complexities. However, \jj{the} provable model-free methods \jj{in} \cite{mao2020model,fei2020dynamic} still lack empirical evidence and adaptability in complex environments. Furthermore, since the agent \jj{can} execute a policy in a fixed environment \jj{only} once due to the non-\jj{stationarity} of the environment, most existing model-free methods only update the policy once for each environment, which prevents the tracking of the time-varying optimal policies.

\textbf{Model-based method in non-stationary RL. }
\jj{The work} \cite{huang2021adarl} \jj{learned} the model change factors and their representation in heterogeneous domains with varying reward functions and dynamics. However, it \jj{has restrictions for use} in non-stationary environments, \jj{meaning that it is} applicable only for constant change factors or the domain adaptation setting. \cite{zintgraf2019varibad} proposed a Bayesian optimal learning policy algorithm by conditioning the action on both states and latent vectors that capture the agent's uncertainty in the environment. Also, \cite{feng2022factored} \jj{brought} insights from recent causality research to model non-stationarity as latent change factors across different environments, and learn policy conditioning on latent factors of the causal graphs. However, learning an optimal policy conditioning on the latent states \cite{zintgraf2019varibad,chen2022adaptive,huang2021adarl,feng2022factored,kwon2021rl} makes the theoretical analysis \hh{intractable.} \jj{The} recent works \cite{cheung2020reinforcement,ding2022non,ding2022provably} proposed model-based algorithms with a provable guarantee, but their algorithms are not scalable for complex environments and lack empirical evaluation for complex environments.

\section{Details on Problem \jj{S}tatement and Notations}
\subsection{Details on Notations}
\label{appendix:notations}

\textbf{Environment Interaction. }
 First, we denote \jj{the} state and action at \jj{the} wall-clock time $\mathfrak{t}_{k}$ of step $h$ as $s^{\mathfrak{t}_k}_h$ \jj{and} $a^{\mathfrak{t}_k}_h$, respectively. As mentioned in the main paper, we \jj{interchangeably use} the \jj{symbols} $s^{(k)}_h$ \jj{and} $a^{(k)}_h$ for $s^{\mathfrak{t}_k}_h$ \jj{and} $a^{\mathfrak{t}_k}_h$. At \jj{the} wall-clock time $\mathfrak{t}_k$, the agent starts from an initial state $s_0^{\mathfrak{t}_k} \sim \rho$. At step $h \in [H]$ of the episode $k$, the agent takes \jj{the} action $a_h^{\mathfrak{t}_k}=\pi^{\mathfrak{t}_k}(\cdot | s_h^{\mathfrak{t}_k})$ from the current state $s_h^{\mathfrak{t}_k}$. The agent then receives the reward $r_{h}^{\mathfrak{t}_k} \sim R_{\mathfrak{t}_k} (s_{h}^{\mathfrak{t}_k},a_{h}^{\mathfrak{t}_k})$ and moves to the next state $s^{\mathfrak{t}_k}_{h+1} \sim P_{\mathfrak{t}_k}( s_{h+1}^{\mathfrak{t}_k}| s_{h}^{\mathfrak{t}_k},a_{h}^{\mathfrak{t}_k})$. The trajectory ends when the agent reaches $s^{\mathfrak{t}_k}_{H}$.

\textbf{Future MDP $\widehat{\mathcal{M}}_{t_{k+1}}$. }
Our work \jj{creates a} one-episode\jj{-}ahead MDP $\widehat{\mathcal{M}}_{t_{k+1}}$ based on the observed data from \jj{the} $p$ lastest MDPs $\{ \mathcal{M}_{t_{k-p+1}},...,\mathcal{M}_{t_k} \}$ when the agent is stated in episode $k$. We define $\widehat{\mathcal{M}}_{t_{k+1}} \coloneqq \langle \mathcal{S}, \mathcal{A}, H, \widehat{P}_{t_{k+1}},\widehat{R}_{t_{k+1}}, \gamma \rangle $\jj{, where} $\widehat{P}_{\mathfrak{t}_{k+1}}$ \jj{and} $\widehat{R}_{t_{k+1}}$ are \jj{the} \textit{forecasted} future transition probability and reward function, respectively. As mentioned in the main paper, \hh{the agent also} interacts with the created future MDP $\widehat{\mathcal{M}}_{t_{k+1}}$ in the same way as it did with the original MDP $\mathcal{M}_{t_k}$. We denote \jj{the} state, action, and policy in $\widehat{\mathcal{M}}_{t_{k+1}}$ as $\widehat{s}^{t_{k+1}}_h,\widehat{a}^{t_{k+1}}_h,\widehat{\pi}^{t_{k+1}}$, \jj{or equivalently} $\widehat{s}^{(k+1)}_h,\widehat{a}^{(k+1)}_h,\widehat{\pi}^{(k+1)}$, respectively. We elaborate our main methodology in Section \ref{method}\jj{.}

\textbf{State value \jj{and} state-action value function\jj{s}. } For any given policy $\pi$ and the MDP $\mathcal{M}_{\mathfrak{t}_k}$, \jj{we} denote the state value function at \jj{the} \add{wall-clock time $\mathfrak{t}_k$}(episode $k$) as $V^{\pi,\mathfrak{t}_k} :\mathcal{S} \rightarrow \mathbb{R}$ and the state-action value function $k$ as $Q^{\pi,\mathfrak{t}_k} :  \mathcal{S} \times \mathcal{A} \rightarrow \mathbb{R}$. We define 
\begin{align*}
    V^{\pi,\mathfrak{t}_k}(s)  &\coloneqq \mathbb{E}_{\mathcal{M}_{\mathfrak{t}_k},\pi}\left[ \sum_{h=0}^{H-1} \gamma^{h} r_{h}^{\mathfrak{t}_k}~\Big\vert~s^{\mathfrak{t}_k}_0=s \right], \\ 
    Q^{\pi,\mathfrak{t}_k}(s,a) &\coloneqq \mathbb{E}_{\mathcal{M}_{\mathfrak{t}_k},\pi}\left[ \sum_{h=0}^{H-1} \gamma^h r_{h}^{\mathfrak{t}_k}~\Big\vert~s^{\mathfrak{t}_k}_0=s,~a^{\mathfrak{t}_k}_0=a \right].
\end{align*}
Also, given the future MDP $\widehat{\mathcal{M}}_{\mathfrak{t}_{k+1}}$, we denote the \textit{forecasted} state value as $\widehat{V}^{\pi,\mathfrak{t}_{k+1}}(s) : \mathcal{S} \rightarrow \mathbb{R}$ and \textit{forecasted} state-action value as $\widehat{Q}^{\pi,\mathfrak{t}_{k+1}} : \mathcal{S} \times \mathcal{A} \rightarrow \mathbb{R}$. We define 
\begin{align*}
    \widehat{V}^{\pi,\mathfrak{t}_{k+1}}(s)  &\coloneqq \mathbb{E}_{\widehat{\mathcal{M}}_{\mathfrak{t}_{k+1}},\pi}\left[ \sum_{h=0}^{H-1} \gamma^{h} \widehat{r}_{h}^{\mathfrak{t}_{k+1}}~\Big\vert~ \widehat{s}^{\mathfrak{t}_{k+1}}_{0} = s \right], \\
    \widehat{Q}^{\pi,\mathfrak{t}_{k+1}}(s,a) &\coloneqq \mathbb{E}_{\widehat{\mathcal{M}}_{\mathfrak{t}_{k+1}},\pi}\left[ \sum_{h=0}^{H-1} \gamma^{h} \widehat{r}_{h}^{\mathfrak{t}_{k+1}}~\Big\vert~ \widehat{s}^{\mathfrak{t}_{k+1}}_{0}=s,~\widehat{a}^{\mathfrak{t}_{k+1}}_{0}=a \right].
\end{align*}

As mentioned in the main paper, we \jj{simplify the symbols} $ V^{\pi,\mathfrak{t}_k},Q^{\pi,\mathfrak{t}_k},\widehat{V}^{\pi,\mathfrak{t}_{k+1}},\widehat{Q}^{\pi,\mathfrak{t}_{k+1}}$ \jj{as} $ V^{\pi,(k)},Q^{\pi,(k)},\widehat{V}^{\pi,(k+1)},\widehat{Q}^{\pi,(k+1)}$.

\textbf{Dynamic regret. } Aside from stationary MDPs, the agent aims to maximize the cumulative expected reward throughout the $K$ episodes by adopting a sequence of policies $\{ \pi^{\mathfrak{t}_k} \}_{1:K}$.  In non-stationary MDPs, the optimality of the \jj{policies} is evaluated in terms of \jj{the} dynamic regret $\mathfrak{R}\left( \{ \pi^{\mathfrak{t}_k}\}_{1:K},K \right)$ \jj{defined as}
\begin{equation}
    \mathfrak{R} \left( \{ \pi^{\mathfrak{t}_k} \}_{1:K},K \right) \coloneqq \sum_{k=1}^{K}\left( V^{\bigast,\mathfrak{t}_k}(\rho)-V^{\pi^{\mathfrak{t}_k},\mathfrak{t}_k}(\rho) \right)
    \label{def:dynamicregret}
\end{equation}

where $V^{\bigast,\mathfrak{t}_k} (=V^{\pi^{\bigast,\mathfrak{t}_k},\mathfrak{t}_k})$ \jj{denotes} the optimal state value function \jj{under} the optimal policy $\pi^{\bigast,\mathfrak{t}_k}$ at the wall-clock time $\mathfrak{t}_k$ (episode $k$) and $V^{\pi^{\mathfrak{t}_k},\mathfrak{t}_k}$ \jj{denotes} the state value with agent's $k^{th}$ episode's policy $\pi^k$. Dynamic regret is a stronger evaluation than the standard static regret \jj{that} considers the optimality of \jj{a} single policy over \jj{all} episodes.

\textbf{State value \jj{and} state-action value \jj{functions at} step $h$. }
We denote the state value function \jj{and the} state-action value function for any policy $\pi$ at \textit{step $h$} of the wall-clock time $\mathfrak{t}_k$ as $V^{\pi,\mathfrak{t}_k}_h$ \jj{and} $Q^{\pi,\mathfrak{t}_k}_h$\jj{,} respectively. We define
\begin{align*}
    V^{\pi,\mathfrak{t}_k}_h(s)  &\coloneqq \mathbb{E}_{\mathcal{M}_{\mathfrak{t}_k},\pi}\Big[ \sum_{i=h}^{H-1} \gamma^{i-h} r_i^{\mathfrak{t}_k}~|~s^{\mathfrak{t}_k}_h=s \Big], \\
    Q^{\pi,\mathfrak{t}_k}_h(s,a)  &\coloneqq \mathbb{E}_{\mathcal{M}_{\mathfrak{t}_k},\pi}\Big[ \sum_{i=h}^{H-1} \gamma^{i-h} r_i^{\mathfrak{t}_k}~|~s^{\mathfrak{t}_k}_h=s, a^{\mathfrak{t}_k}_h = a \Big].
\end{align*}
Then, the corresponding Bellman equation is  
\begin{equation}
    \label{eqn:bellman_step}
    Q^{\pi,\mathfrak{t}_k}_h(s,a) = \big( R_{\mathfrak{t}_k} + \gamma P_{\mathfrak{t}_k} V_{h+1}^{\pi,\mathfrak{t}_k}\big) (s,a),~~
    V^{\pi,\mathfrak{t}_k}_h(s) = \big< Q^{\pi,\mathfrak{t}_k}_h (s,\cdot),\pi(\cdot|s) \big>_\mathcal{A}, ~~ V_{H}^{\mathfrak{t}_k,\pi}(s) = 0 ~\forall s \in \mathcal{S}
\end{equation}
where $\left(P_{\mathfrak{t}_k} f \right)(s,a) := \mathbb{E}_{s^\prime \sim P^{\mathfrak{t}_k}(\cdot | s,a))}[f(s^\prime)]$ for \jj{every} function $f: \mathcal{S} \rightarrow \mathbb{R}$.

We denote $V^{\bigast,\mathfrak{t}_k}_h(s) =V^{\pi^{\bigast,\mathfrak{t}_k},\mathfrak{t}_k}_h(s)$ as the optimal state value function \jj{at} step $h$ of episode $k$. We omit the \jj{subscript} $h$ when $h=0$, that is, $V^{\pi,k}=V^{\pi,k}_0,~Q^{\pi,k}= Q^{\pi,k}_0$. Then, the corresponding \jj{Bellman equation} is 
\begin{align}
    \label{eqn:bellmanOptimal_step}
    Q^{\bigast,\mathfrak{t}_k}_h(s,a) = \big(R_{\mathfrak{t}_k} + \gamma &P_{\mathfrak{t}_k} V_{h+1}^{\bigast,\mathfrak{t}_k}\big) (s,a),~~
    V^{\bigast,\mathfrak{t}_k}_h(s) = \big< Q^{\bigast,\mathfrak{t}_k}_h(s,\cdot),\pi^{\bigast,\mathfrak{t}_k}(\cdot | s) \big>_{\mathcal{A}}, \\ \nonumber 
    &\pi^{\bigast,\mathfrak{t}_k}(s) = \max_{a} Q_h^{\bigast,\mathfrak{t}_k}(s,a).
\end{align}
We also denote \jj{the} \textit{forecasted} state value at \jj{the} wall-clock time $\mathfrak{t}_{k+1}$ of step $h$ when the agent \hh{is stated} at time $\mathfrak{t}_k$ as $\widehat{V}^{\pi,\mathfrak{t}_{k+1}}_h$ and \jj{the} $\textit{forecasted}$ state-action value as $\widehat{Q}^{\pi,\mathfrak{t}_{k+1}}_h$  in a forecasted MDP $\widehat{\mathcal{M}}_{\mathfrak{t}_{k+1}}$. We define 
\begin{align}
    \widehat{V}^{\pi,\mathfrak{t}_{k+1}}_h(s) &\coloneqq \mathbb{E}_{\widehat{\mathcal{M}}_{\mathfrak{t}_{k+1}},\pi}\Big[ \sum_{i=h}^{H-1} \gamma^{i-h} \widehat{r}_{i}^{\mathfrak{t}_{k+1}}~|~\widehat{s}^{\mathfrak{t}_{k+1}}_h=s \Big] ,\\
    \widehat{Q}^{\pi,\mathfrak{t}_{k+1}}_h(s,a) &\coloneqq \mathbb{E}_{\widehat{\mathcal{M}}_{\mathfrak{t}_{k+1}},\pi}\Big[ \sum_{i=h}^{H-1} \gamma^{i-h} \widehat{r}_{i}^{\mathfrak{t}_{k+1}}~|~ \widehat{s}^{\mathfrak{t}_{k+1}}_h=s, \widehat{a}^{\mathfrak{t}_{k+1}}_h = a \Big].
\end{align}
Then, the \jj{B}ellman equation is given by  
\begin{align}
    \label{eqn:futurebellman_step}
    \widehat{Q}^{\pi,\mathfrak{t}_{k+1}}_h(s,a) = \big( \widehat{R}_{\mathfrak{t}_{k+1}}+ &\gamma \widehat{P}_{\mathfrak{t}_{k+1}} \widehat{V}_{h+1}^{\pi,\mathfrak{t}_{k+1}}\big) (s,a),~~
    \widehat{V}^{\pi,\mathfrak{t}_{k+1}}_h(s) = \big< \widehat{Q}^{\pi,\mathfrak{t}_{k+1}}_h(s,\cdot),\pi(\cdot | s) \big>_\mathcal{A} \nonumber, \\ &\widehat{V}_{H}^{\pi, \mathfrak{t}_{k+1}}(s) = 0 ~\forall s \in \mathcal{S}. 
\end{align}
\jj{We} denote the \textit{future} optimal policy of the \textit{future} value function $\widehat{V}^{\pi,\mathfrak{t}_{k+1}}$ as $\widehat{\pi}^{\bigast,\mathfrak{t}_{k+1}}$. Then the \jj{B}ellman equation also holds for $\widehat{Q}^{\pi,\mathfrak{t}_{k+1}}_h(s)$ \jj{and} $\widehat{V}^{\pi,\mathfrak{t}_{k+1}}_h(s)$ as follows\jj{:}
\begin{align}
    \widehat{Q}^{\bigast,\mathfrak{t}_{k+1}}_h(s,a) =  \big(\widehat{R}_{\mathfrak{t}_{k+1}} +&\gamma\widehat{P}_{\mathfrak{t}_{k+1}} \widehat{V}_{h+1}^{\bigast,\mathfrak{t}_{k+1}}\big) (s,a),~~
    \widehat{V}^{\bigast,\mathfrak{t}_{k+1}}_h(s) = \big< \widehat{Q}^{\bigast,\mathfrak{t}_{k+1}}_h(s,\cdot),\widehat{\pi}^{\bigast,\mathfrak{t}_{k+1}}(\cdot | s) \big>_{\mathcal{A}}, \nonumber \\
    &\widehat{\pi}^{\bigast,\mathfrak{t}_{k+1}}(s) = \max_{a} \widehat{Q}_h^{\bigast,\mathfrak{t}_{k+1}}(s,a).
    \label{eqn:futurebellmanOptimal_step}
\end{align}

As mentioned in the main paper, we \jj{simplify the notations} $V^{\pi,\mathfrak{t}_k}_h,Q^{\pi,\mathfrak{t}_k}_h,\widehat{V}^{\pi,\mathfrak{t}_{k+1}}_h,\widehat{Q}^{\pi,\mathfrak{t}_{k+1}}_h$ \jj{as} $V^{\pi,(k)}_h,Q^{\pi,(k)}_h,\widehat{V}^{\pi,(k+1)}_h,\widehat{Q}^{\pi,(k+1)}_h$.

\textbf{Unnormalized (discounted) occupancy measure. }
We define the unnormalized (discounted) occupancy measure $\nu^{\pi,\mathfrak{t}_k}_{s_0,a_0} \in \Delta_{1/(1-\gamma)}(\mathcal{S} \times \mathcal{A})$ at wall-clock time $\mathfrak{t}_k$ (episode $k$) for \jj{a} given policy $\pi$ \jj{together with} an initial state \jj{$s_0$} and \jj{the} action \jj{$a_0$} as  
\begin{align}
    \nu^{\pi,\mathfrak{t}_k}_{s_0,a_0}(s,a) &\coloneqq \sum_{h=0}^{\infty}\gamma^h \mathbb{P} (s_h=s, a_h=a~|~s_0,a_0~;~\pi,P_{\mathfrak{t}_k})~~ ,~~\forall (s,a) \in \mathcal{S} \times \mathcal{A} \label{def:Nu}
\end{align}
where $\mathbb{P} (s_h=s, a_h=a~|~s_0,a_0~;~\pi,P^{\mathfrak{t}_k})$ is the probability of visiting $(s,a)$ at step $h$ when following policy $\pi$ from $(s_0,a_0)$ with \jj{the} transition probability $P_{\mathfrak{t}_{k+1}}$. 

We also define the unnormalized non-stationary (discounted) \textit{forecasted} occupancy measure $\widehat{\nu}^{\pi,\mathfrak{t}_{k+1}}_{s_0} \in \Delta_{1/(1-\gamma)}(\mathcal{S} \times \mathcal{A})$  for \jj{a} given policy $\pi$, an initial state \jj{$s_0$}\jj{, an} action $a_0$, and \jj{a} forecasted future transition probability $\widehat{P}_{\mathfrak{t}_{k+1}}$\jj{:}
\begin{align}
    \widehat{\nu}^{\pi,\mathfrak{t}_{k+1}}_{s_0,a_0}(s,a) &\coloneqq \sum_{h=0}^{\infty}\gamma^h \mathbb{P} (s_h=s, a_h=a~|~s_0,a_0,\pi, \widehat{P}_{\mathfrak{t}_{k+1}})~~ , \forall (s,a) \in \mathcal{S} \times \mathcal{A} \label{def:Nu_future}
\end{align}
where \jj{the} probability is defined in a forecasted environment with $\widehat{P}_{\mathfrak{t}_{k+1}}$.

\textbf{Model prediction error. } 
To measure how well our meta-function predicts the future environment, we define two different \textit{model prediction error}s $\iota^{\mathfrak{t}_{k+1}}_{\infty},\iota^{\mathfrak{t}_{k+1}}_{h} : \mathcal{S} \times \mathcal{A} \rightarrow \mathbb{R}$, which denote the \jj{Bellman} equation error when using $\widehat{V}$ \jj{and} $\widehat{Q}$ estimated in the future MDP instead of the true $V$ \jj{and} $Q$ function\jj{s}:
\begin{align}
    \bar{\iota}^{\mathfrak{t}_{k+1}}_\infty(s,a) &\coloneqq \left( R_{\mathfrak{t}_{k+1}} + \gamma P_{\mathfrak{t}_{k+1}} \widehat{V}^{\bigast,\mathfrak{t}_{k+1}}_{\infty} - \widehat{Q}^{\bigast,\mathfrak{t}_{k+1}}_{\infty} \right) (s,a),
    \label{def:optimalModelPredictionError} \\
    \iota^{\mathfrak{t}_{k+1}}_h(s,a) &\coloneqq \left( R_{\mathfrak{t}_{k+1}} + \gamma P_{\mathfrak{t}_{k+1}}\widehat{V}^{\widehat{\pi}^{\mathfrak{t}_{k+1}},\mathfrak{t}_{k+1}}_{h+1} - \widehat{Q}^{\widehat{\pi}^{\mathfrak{t}_{k+1}},\mathfrak{t}_{k+1}}_{h} \right) (s,a).
    \label{def:estimatedModelPredictionError}
\end{align}

\jj{As mentioned} in the main paper, we allow $\bar{\iota}^{\mathfrak{t}_{k+1}}_\infty(s,a)$ \jj{and} $\iota^{\mathfrak{t}_{k+1}}_h(s,a)$ \jj{to be interchangeably expressed by the symbols} $\bar{\iota}^{(k+1)}_\infty(s,a)$ \jj{and} $\iota^{(k+1)}_h(s,a)$.

\textbf{Local time-elapsing variation budget. }
\jj{Aside from the} time\jj{-}elapsing  variation budget, we define \jj{the} \emph{local} time\jj{-}elapsing variation budget\jj{s} $B^{(k-w:k)}_p$ \jj{and} $B^{(k-w:k)}_r$ that quantifie how fast the environment changes \jj{over} wall-clock time\jj{s} $\{ \mathfrak{t}_{k-w+1}, \mathfrak{t}_{k+1},...,\mathfrak{t}_{k}\}$ where $k-w,k \in [K]$\jj{:}

\begin{align*}
    B^{(k-w+1:k)}_{p}(\Delta_\pi) &\coloneqq \sum_{\tau=k-w+1}^{k} \sup_{s,a}|| P_{\mathfrak{t}_{\tau+1}}(\cdot~|s,a) - P_{\mathfrak{t}_\tau}(\cdot~|s,a) ||_{1}, \\
    B^{(k-w+1:k)}_{r}(\Delta_\pi) &\coloneqq \sum_{\tau=k-w+1}^{k} \sup_{s,a} |  R_{\mathfrak{t}_{k+1}}(s,a) - R_{\mathfrak{t}_{k}}(s,a)  |.
\end{align*}
\section{Proof of \jj{T}heoretical \jj{A}nalysis}


\subsection{Preliminary for ProST-T and theoretical analysis}
In this subsection, we elaborate \jj{on} the \texttt{ProST-T}'s environment setting and \hh{its components} $f,g$.
\subsubsection{Environment setting}
\label{appendix:Environment_setting}

We consider the tabular environment \jj{have the following properties:} 
\begin{enumerate}
    \item First,  $P_{(k)}$ \jj{and} $R_{(k)}$ are represented by the inner products of the feature functions $\psi : \mathcal{S}\times\mathcal{S}\times\mathcal{A}\rightarrow \mathbb{R}^{|\mathcal{S}|^2|\mathcal{A}|},\varphi : \mathcal{S}\times\mathcal{A}\rightarrow \mathbb{R}^{|\mathcal{S}||\mathcal{A}|}$ and the non-stationary variables $o^p_{(k)},o^r_{(k)} \in \mathcal{O}$, respectively, where $o^p_{(k)} : \mathcal{S} \times \mathcal{S} \times \mathcal{A} \rightarrow \mathbb{R}^{|\mathcal{S}|^2|\mathcal{A}|}$ \jj{and} $o^r_{(k)} : \mathcal{S} \times \mathcal{A} \rightarrow \mathbb{R}^{|\mathcal{S}||\mathcal{A}|}$. That is, $P_{(k)} = <\psi, o^p_{(k)}>$ and $R_{(k)} = <\varphi, o^r_{(k)}>$.
    \item Second, the agent estimates $o^p_{(k)}$ \jj{and} $o^r_{(k)}$ rather than \jj{observing them}. More specifically, we consider the non-stationary variable set $\mathcal{O}$ to be the set $\{P_{(k)}\}_{1:K},\{R_{(k)}\}_{1:K}$. The agent then \jj{attempts} to \textit{estimate} $o_k$ (denote $P_{(k)}$ as $o^p_{(k)}$ and $R_{(k)}$ as $o^r_{(k)}$) through its $w$ lastest trajectories, where Assumption \ref{assume:observable_nonstatinoaryvar} does not \jj{need} to be satisfied in this setting. That is, the agent estimates $P_{(k)}$ by $\hat{o}^p_k$ and $R_{(k)}$ by $\hat{o}^r_k$ from observations of last $w$ trajectories, \jj{i.e.,} $\tau_{k-(w-1):k}$.
\end{enumerate}
We elaborate on the above two \jj{settings below:} 

\textbf{1. $P_{(k)},R_{(k)}$ are inner products of $\psi,\varphi$ and $o^p_{(k)},o^r_{(k)}$\jj{.}}
    
    Let us define a set of \hh{one-}hot reward vectors over all states and the action space\jj{, namely} $\mathbbm{1} _r \coloneqq \{ \varphi^{y} \in \{0,1\}^{|\mathcal{S}||\mathcal{A}|}~|~ \sum_{i=1}^{|\mathcal{S}| |\mathcal{A}|}\varphi^{y}_{i}=1 \}$\jj{, and} similarly define a set of \hh{one-}hot transition probability vectors\jj{, namely} $\mathbbm{1}_p \coloneqq \{ \psi^{y} \in \{0,1\}^{|\mathcal{S}|^2|\mathcal{A}|}~|~ \sum_{i=1}^{|\mathcal{S}|^2|\mathcal{A}|}  \psi^{y}_i=1 \}$. \jj{We then define} one-to-one function\jj{s} $\varphi$ \jj{and} $\psi$ \jj{such that} $\varphi :\mathcal{S} \times \mathcal{A} \rightarrow \mathbbm{1}_r$ \jj{and} $\psi :\mathcal{S} \times \mathcal{S} \times \mathcal{A} \rightarrow \mathbbm{1}_p$. Namely, $\varphi(s,a) (\psi(s^\prime,s,a))$  is a \hyunin{one-hot} vector \jj{such} that the $(i)^{th}$ entry equals 1. We \jj{use} the notation $\varphi^k_h = \varphi(s^{(k)}_h,a^{(k)}_h)$ for the observed $(s^{(k)}_h,a^{(k)}_h)$ on the trajectory $\tau_k$, and \jj{similarly} $\psi^k_h = \psi(s^{(k)}_{h+1},s^{(k)}_h,a^{(k)}_h)$.

    Then, we set $\mathcal{O} = \{ P_{(k)}, R_{(k)} \}_{k=1}^{\infty}$ in \texttt{ProST-T}. Also, we set $o_k$ to consist of two \jj{parameters as} $o_k = (o^p_{(k)},o^r_{(k)}) $. We define a \hh{function} $o^p_{(k)} \coloneqq \{o : \mathcal{S} \times  \mathcal{S} \times \mathcal{A} \rightarrow \mathbb{R}^{|\mathcal{S}|^2|\mathcal{A}|}~|~o(s^\prime,s,a)=P_{(k)}(s^\prime|s,a),~\forall(s^\prime,s,a) \}$ \jj{and} a \hh{function} $o^r_{(k)} \coloneqq \{o : \mathcal{S}\times\mathcal{A}\rightarrow \mathbb{R}^{|\mathcal{S}||\mathcal{A}|} ~|~ o(s,a) = R_{(k)}(s,a),~\forall(s,a) \}$.
    Then, the transition probability and reward value $P_{(k)}$ \jj{and} $R_{(k)}$ can be constructed by the inner product\jj{s} of the stationary function\jj{s $\varphi$ and $\psi$} and the unknown non-stationary \jj{parameters} $o^p_{(k)}$ and $o^r_{(k)}$ as follows, 
    \begin{align}
        P_{(k)}(s^\prime~|~s,a) &~ \coloneqq~ < \psi(s^\prime,s,a), o^p_{(k)}(s^\prime,s,a)> \text{ for } \forall (s^\prime,s,a),\label{eqn:lemmaeqn15}\\ 
        R_{(k)}(s,a) & ~\coloneqq ~ < \varphi(s,a), o^r_{(k)}(s,a)> \text{ for } \forall (s,a). \label{eqn:lemmaeqn16}
    \end{align}
    For \jj{notational} simplicity, we \jj{use} $<\psi,o^p_{(k)}>$ \jj{and} $<\varphi,o^r_{(k)}>$ \jj{to show the inner products of the} functions $\psi,o^p_{(k)}$ and $\varphi,o^r_{(k)}$\jj{,} respectively. \jj{Therefore,} $P_{(k)}= <\psi,o^p_{(k)}>$ \jj{and} $R_{(k)}=<\varphi,o^r_{(k)}>$.
    
    \jj{To give an} intuitive explanation, \jj{note that} $o^p_{(k)}$ contains all transition probabilities for all $(s^\prime,s,a)$ \jj{in} a vector form with size $\mathbb{R}^{|\mathcal{S}|^2|\mathcal{A}|}$ and $o^r_{(k)}$ contains all rewards for all $(s,a)$ \jj{in} a vector form with size $\mathbb{R}^{|\mathcal{S}||\mathcal{A}|}$.

\textbf{2. The agent estimates $o^r_{(k)}$ \jj{and} $o^p_{(k)}$ rather than \jj{observing} them}

    We have defined the \jj{functions} $o^p_{(k)}$ \jj{and} $o^r_{(k)}$ as the transition probability and reward \jj{functions} at episode $k$\jj{,} respectively. Now\jj{,} the agent \jj{strives} to estimate $o^p_{(k)}$ \jj{and} $o^r_{(k)}$\jj{, denoted as} $\widehat{o}^p_{(k)}$ \jj{and} $\widehat{o}^r_{(k)}$\jj{, from} the current trajectory $\tau_k$\jj{:}
    \begin{align*}
        \widehat{o}^p_{(k)}(s^\prime,s,a) &= \frac{n_{(k)}(s^\prime,s,a)}{\lambda + n_{(k)}(s,a)},~~\forall (s^\prime,s,a) \in \mathcal{S} \times \mathcal{S} \times \mathcal{A}, \\ 
        \widehat{o}^r_{(k)}(s,a) &= \frac{ \sum_{h=0}^{H-1} \mathbbm{1} \left[ (s,a)=(s^{(k)}_h,a^{(k)}_h) \right] \cdot r^{(k)}_h}{n_{k}(s,a)},~~\forall (s,a) \in  \mathcal{S} \times \mathcal{A}
    \end{align*}
    where $n_{(k)}(s,a)$ denotes visitation count of state $s$ \jj{under} action $a$ through trajectory $\tau_{(k)}$ \hh{and $n_{(k)}(s,a,s^\prime)$ denotes visitation count of state $s$ under action $a$ and subsequent next state $s^\prime$ through trajectory}. We denote $\widehat{o}_{k,h}^p =  \widehat{o}^p_{(k)} (s^{(k)}_{h+1},s^{(k)}_{h},a^{(k)}_{h})$ and $\widehat{o}_{k,h}^r = \widehat{r}^k_h (s^{(k)}_{h},a^{(k)}_{h})$\jj{.} 

\jj{It can be verified that the following relations hold at episode $k$} for \jj{the state and action} pairs from the $k^{\text{th}}$ trajectory $\{ s^{(k)}_0,a^{(k)}_0,s^{(k)}_1,a^{(k)}_1,..,s^{(k)}_{H-1},a^{(k)}_{H-1},s^{(k)}_{H}\}$\jj{:}
    \begin{align}
        P_{(k)}(s^{(k)}_{h+1}~|~s^{(k)}_{h},a^{(k)}_{h}) &~ =~ < \psi(s^{(k)}_{h+1},s^{(k)}_{h},a^{(k)}_{h}), o^p_{(k)}(s^{(k)}_{h+1},s^{(k)}_{h},a^{(k)}_{h})> ~~~~ ,\forall h\in[H], \\ 
        R_{(k)}(s^{(k)}_{h},a^{(k)}_{h}) & ~= ~ < \varphi(s^{(k)}_{h},a^{(k)}_{h}), o^r_{(k)}(s^{(k)}_{h},a^{(k)}_{h})> ~~~~ ,\forall h\in[H]. \
    \end{align}

    \jj{Note that} the observed non-stationary \jj{parameters} $\widehat{o}^p_{(k)}$ \jj{and} $\widehat{o}^r_{(k)}$ can be interpreted  partially observed \jj{vectors}.

\subsubsection[Functions f,g]{Functions $f$ and $g$}
\label{appendix:Meta_function_fg}
The function $f$ estimates and \jj{the function} $g$ predicts as follows\jj{
:}
\begin{enumerate}
    \item \textbf{Function $f$}: $f$ forecasts one-episode\jj{-}ahead non-stationary \jj{parameters} $\hat{o}^p_{(k+1)}$ \jj{and} $\hat{o}^r_{(k+1)}$ by minimizing the following loss function $\mathcal{L}_{f^\diamond}$ with \jj{the} regularization parameter $\lambda \in \mathbb{R}_{+}$\jj{:}
    \begin{align*}
        \mathcal{L}_{f^\diamond}(\phi~;~\widehat{o}^\diamond_{(k-w+1:k)}) =  \lambda ||\phi||^2 + \sum_{s=k-w+1}^{k} \sum_{h=0}^{H-1} \left( (\square^s_h)^{\top}\phi- \widehat{o}^{\diamond}_{s,h}\right)
    \end{align*}
    where $\diamond = r,p$ and $\square = \varphi$ if $\diamond = r$. We set $\square = \psi$ if $\diamond = p$. We let $\phi^k_{f^\diamond} = \text{argmin}_{\phi} \mathcal{L}_{f^\diamond}(\widehat{o}^\diamond_{k-(w-1):k})$. We use $\phi^k_{f^\diamond}$ as $\widehat{o}^\diamond_{k+1}$ .
    \item  \textbf{Function $g$}: Then $g$ predicts the function\jj{s} $\widehat{P}_{(k+1)}$ \jj{and} $\widehat{R}_{(k+1)}$\jj{, denoted as} $\widehat{g}^P_{(k+1)}$ \jj{and} $\widehat{g}^R_{(k+1)}$\jj{,} as $\widehat{P}_{(k+1)} = \widehat{g}^P_{(k+1)} \coloneqq < \varphi, \widehat{o}^p_{(k+1)}>$ and $\widehat{R}_{(k+1)} = \widehat{g}^R_{(k+1)} \coloneqq < \varphi, \widehat{o}^r_{k+1}>  +   2\Gamma^{(k)}_w$\jj{, where} $\Gamma^{(k)}_w(s,a) : \mathcal{S} \times \mathcal{A} \rightarrow \mathbb{R}$ is the exploration bonus term that adapts the counter-based bonus terms in the literature.
\end{enumerate}

    We elaborate on above two procedures \jj{below:}

\textbf{1. The function $f$ solves \jj{an} optimization problem to \jj{obtain} the future $\widehat{o}_{(k+1)}$.}
    
    The function $g \circ f$ forecasts the $k+1^{th}$ episode's non-stationary \jj{parameters} as $(\widehat{o}^p_{(k+1)},\widehat{o}^r_{(k+1)}) $ from $\widehat{o}_{(k-w+1:k)}$\jj{, where} $w$ is the sliding window length (past reference length). The function $f$ forecasts $o^p_{(k+1)}$ \jj{and} $o^r_{(k+1)}$ by minimizing the following two regularized least\jj{-}squares optimization problems \cite{cheung2019learning}.
    \begin{align}
        \widehat{o}^{p}_{(k+1)} &= \argmin_{o \in \mathbb{R}^{|\mathcal{S}|^2|\mathcal{A}|}} \left(\lambda ||o||^2 + \sum_{s=k-w+1,h=0}^{k,H} \left( (\psi^s_h)^{\top}o- \widehat{o}^{p}_{s,h}\right) \right)
        \label{eqn:SWRLS_predictP} \\
        \widetilde{o}^{r}_{(k+1)} &= \argmin_{o \in \mathbb{R}^{|\mathcal{S}||\mathcal{A}|}} \left(\lambda ||o||^2 + \sum_{s=k-w+1,h=0}^{k,H-1} \left( (\varphi^s_h)^{\top}o- \widehat{o}^r_{s,h}\right) \right)
        \label{eqn:SWRLS_predictR}
    \end{align}

\textbf{2. \jj{The} function $g$ predicts $\widehat{P}_{(k+1)}$ \jj{and} $\widehat{R}_{(k+1)}$ from $\widehat{o}_{k+1}$.}
    
    From the equations (17a) and (17b) of the paper \cite{ding2022provably}, the explicit solutions of \eqref{eqn:SWRLS_predictP} and \eqref{eqn:SWRLS_predictR} \jj{are given as}
    \begin{equation}
        \widehat{o}^p_{(k+1)}(s^\prime,s,a) = \frac{\sum_{t=k-w+1}^{k} n_t(s^\prime,s,a)}{\lambda + \sum_{t=k-w+1}^{k} n_t(s,a)},~\widetilde{o}^r_{(k+1)}(s,a) = \frac{\sum_{t=k-w+1}^{k} \sum_{h=0}^{H-1} \mathbbm{1} \left[ (s,a)=(s^t_h,a^t_h) \right] \cdot r^t_h}{\lambda + \sum_{t=k-w+1}^{k} n_{t}(s,a)}. \label{eqn:r_p_estimation}
    \end{equation}
    
    Then, the \texttt{ProST-T} predicts the future model using the function\jj{s} $\widehat{g}^P_{k+1}$ \jj{and} $\widehat{g}^R_{k+1}$ as follows\jj{:}
    \begin{align*}
        \widehat{g}^P_{k+1}(s^\prime,s,a) &\coloneqq < \varphi(s^\prime,s,a), \widehat{o}^p_{(k+1)}(s^\prime,s,a)>, \\
        \widetilde{g}^R_{k+1}(s,a) &\coloneqq < \varphi(s,a), \widetilde{o}^r_{(k+1)}(s,a)>, \\ 
        \widehat{g}^R_{k+1}(s,a) &\coloneqq \widetilde{g}^R_{k+1}(s,a) + 2\Gamma^{(k)}_w(s,a).
    \end{align*}
    We utilize the exploration bonus $\Gamma^{(k)}_w(s,a) : \mathcal{S} \times \mathcal{A} \rightarrow \mathbb{R}$ \jj{to} explore \jj{those state and action pairs} that \jj{are} less visited. We define it as $\Gamma^{(k)}_w(s,a) = \beta \left( \sum_{t=k-w+1}^{k}n_t(s,a)  + \lambda \right)^{-1/2}$ \jj{with} $\beta >0$.
    Then, we use $\widehat{g}^P_{k+1}$ \jj{and} $\widehat{g}^R_{k+1}$ \jj{to denote} the future MDP's $\widehat{P}_{(k+1)}$ \jj{and} $\widehat{R}_{(k+1)}$,   respectively. 
    From the following analysis, we \jj{write} $\widehat{P}_{(k+1)} = \widehat{g}^P_{(k+1)}$, $\widetilde{R}_{(k+1)} = \widetilde{g}^R_{(k+1)}$, \jj{and} $\widehat{R}_{(k+1)} = \widehat{g}^R_{(k+1)}$.

\subsubsection[Baseline algorithm Alg]{Baseline algorithm\jj{s} $\texttt{Alg}$ \jj{and} $\texttt{Alg}_\tau$}
\label{appendix:baseline_algorithm}

The \texttt{ProST-T} utilizes softmax parameterization that naturally ensures that the policy lies in the probability simplex. For any function that satisfies $\theta : \mathcal{S} \times \mathcal{A} \rightarrow \mathbb{R}$, the policy $\pi^{(k)}$ is generated by the softmax transformation of $\theta^{(k)}$ at \jj{the} wall\jj{-}clock time $\mathfrak{t}_k$. Furthermore, to promote exploration and discourage premature convergence to suboptimal policies in a non-stationary environment, we \jj{implement} a widely used strategy known as entropy regularization. We augment the future state value function with an additional $\pi^{(k)}(s)$ entropy term, denoted by $\tau \mathcal{H}(s,\pi^{(k)})$, where $\tau > 0$. We \jj{perform} a theoretical analysis with two baseline algorithms : Natural Policy Gradient (NPG) \texttt{Alg} and Natural Policy Gradient (NPG) with entropy regularization $\texttt{Alg}_{\tau}$

\textbf{Softmax parameterization. }
For any function that satisfies $\theta : \mathcal{S} \times \mathcal{A} \rightarrow \mathbb{R}$, the policy $\pi^{(k)}$ is generated by the softmax transformation of $\theta^{(k)}$ at the wall\jj{-}clock time $\mathfrak{t}_k$. \jj{Using} the notation $\pi^{(k)} = \pi_{\theta^{(k)}}$\jj{, the} soft parameterization is defined as
\begin{equation*}
    \pi_{\theta^{(k)}}(a|s) \coloneqq \frac{\exp{(\theta^{(k)}(s,a))}}{\sum_{a^\prime \in \mathcal{A}}{\exp{(\theta^{(k)} (s,a^\prime))}}}~~,\forall (s,a) \in \mathcal{S} \times \mathcal{A}.
\end{equation*}
Under \jj{the} softmax parameterization, the NPG update rule admits a simple form of update rule \jj{given in} line 17 of Algorithm \ref{alg:FT-MBPO} in Appendix \ref{section:meta_algorithm}. This is elaborated in \cite{cen2022fast}.

\textbf{Entropy regularized value maximization. }
 For any policy $\pi$, we define \jj{the} \textit{forecasted} entropy-regularized state value function $\widehat{V}^{\pi,\mathfrak{t}_{\mathfrak{t}_k+1}}_{\tau}(s)$ as
\begin{equation*}
     \widehat{V}^{\pi,\mathfrak{t}_{k+1}}_{\tau}(s) \coloneqq  \widehat{V}^{\pi,\mathfrak{t}_k+1}(s)+\tau \mathcal{H}(s,\pi)
\end{equation*}
where $\tau \geq 0$ is a regularization parameter and $\mathcal{H}(s,\pi)$ is a discounted entropy defined as
\begin{equation*}
    \mathcal{H}(s,\pi) \coloneqq \mathbb{E}_{\widehat{\mathcal{M}}_{(k+1)}} \left[ \sum^{H-1}_{h=0} -\gamma^h \log\pi(\hat{a}^{(k+1)}_h|\hat{s}^{(k+1)}_h) \vert \hat{s}^{(k+1)}_{0} =s \right].
\end{equation*}
Also, we define the \textit{forecasted} regularized Q-function $\widehat{Q}^{\pi,(k+1)}_{\tau}$ as
\begin{align*}
    \begin{split}
            &\widehat{Q}^{\pi,\mathfrak{t}_{k+1}}_{\tau}(s,a) = \hat{r}^{\mathfrak{t}_{k+1}}_h + \gamma \mathbb{E}_{s^\prime \sim \widehat{P}_{\mathfrak{t}_{k+1}}( \cdot| s,a)}\left[  \widehat{V}^{\pi,\mathfrak{t}_{k+1}}_{\tau}(s^\prime) \right] \\
        &\text{where   } (s^\prime,s,a) = (\hat{s}^{(k+1)}_{h+1}, \hat{s}^{(k+1)}_h,\hat{a}^{(k+1)}_h).
    \end{split}
\end{align*}

\subsection{Notation for theoretical analysis}
\label{appendix:Notation for theoretical analysis}
This subsection introduces some notations that we \jj{will} use in the proofs.

    At \jj{the} wall-clock time $\mathfrak{t}_{k}$, \jj{we} define the \textit{forecasting model error} $ \Delta^{r}_{\mathfrak{t}_{k}}(s,a)$ and \textit{forecasting transition probability model error} $\Delta^{p}_{\mathfrak{t}_{k}}(s,a)$ below: 
    \begin{align}
        \Delta^{r}_{\mathfrak{t}_{k}}(s,a) &\coloneqq  \left\vert \left(R_{(k+1)} - \widetilde{R}_{(k+1)} \right) (s,a) \right\vert, \label{eqn:forecasting reward model error}\\ 
        \Delta^{p}_{\mathfrak{t}_{k}}(s,a) & \coloneqq  \left\vert \left\vert \left( P_{(k+1)} - \widehat{P}_{(k+1)} \right)(\cdot~|~s,a) \right\vert \right\vert_1. \label{eqn:forecasting transition probability model error}
    \end{align}
    Recall that $\widetilde{R}_{(k+1)}$ \jj{and} $\widehat{P}_{(k+1)}$ \jj{estimate the} future reward and transition probability by solving the optimization \jj{problems} \eqref{eqn:SWRLS_predictP} \jj{and} \eqref{eqn:SWRLS_predictR}\jj{.}

    \jj{We define a} model error that considers the bonus term \jj{as}
    \begin{align*}
        \Delta^{Bonus,r}_{\mathfrak{t}_{k}}(s,a) \coloneqq  \left\vert \left(R_{(k+1)} - \widehat{R}_{(k+1)} \right) (s,a) \right\vert
    \end{align*}
    where $\widehat{R}_{(k+1)}(s,a) = \widetilde{R}_{(k+1)}(s,a) + 2 \Gamma^{(k)}_w(s,a)$.
    
     We also define \jj{the} \textit{empirical} forecasting reward model error  $\bar{\Delta}^{r}_{\mathfrak{t}_{k},h}$ \jj{and the} \textit{empirical} forecasting transition probability model error $\bar{\Delta}^{p}_{\mathfrak{t}_{k},h}$\jj{:}
    \begin{align*}
        \bar{\Delta}^{r}_{\mathfrak{t}_{k},h} &\coloneqq  \left\vert \left(R_{(k+1)} - \widetilde{R}_{(k+1)} \right) (s^{(k+1)}_h,a^{(k+1)}_h) \right\vert, \\ 
        \bar{\Delta}^{p}_{\mathfrak{t}_{k},h} & \coloneqq  \left\vert \left\vert \left( P_{(k+1)} - \widehat{P}_{(k+1)} \right)(\cdot~|~s^{(k+1)}_h,a^{(k+1)}_h) \right\vert \right\vert_1
    \end{align*}
    \jj{as well as} the \emph{empirical} bonus \jj{based on the} reward model error\jj{:} 
    \begin{equation*}
        \bar{\Delta}^{Bonus,r}_{\mathfrak{t}_{k},h} \coloneqq  \left\vert \left(R_{(k+1)} - \widehat{R}_{(k+1)} \right) (s^{(k+1)}_{h},a^{(k+1)}_{h}) \right\vert.
    \end{equation*}
    \jj{Likewise, we define} \textit{total empirical} forecasting reward model error $\bar{\Delta}^{r}_{K}$ \jj{and the} \textit{total empirical} forecasting transition probability model error  $\bar{\Delta}^{p}_{K}$\jj{:}
    \begin{align}
        \bar{\Delta}^{r}_{K} \coloneqq \sum_{k=1}^{K-1} \sum_{h=0}^{H-1}  \bar{\Delta}^{r}_{\mathfrak{t}_{k},h}, \\
        \bar{\Delta}^{p}_{K} \coloneqq \sum_{k=1}^{K-1} \sum_{h=0}^{H-1}  \bar{\Delta}^{p}_{\mathfrak{t}_{k},h}.
    \end{align}
    \jj{We simplify the symbols} $\Delta^{r}_{\mathfrak{t}_{k}}(s,a),\Delta^{p}_{\mathfrak{t}_{k}}(s,a),\Delta^{Bonus,r}_{\mathfrak{t}_{k}}(s,a),\bar{\Delta}^{r}_{\mathfrak{t}_{k},h},\bar{\Delta}^{p}_{\mathfrak{t}_{k},h},\bar{\Delta}^{Bonus,r}_{\mathfrak{t}_{k},h}$ \jj{as} $\Delta^{r}_{(k)}(s,a),\Delta^{p}_{(k)}(s,a),\Delta^{Bonus,r}_{(k)}(s,a),\bar{\Delta}^{r}_{(k),h},\bar{\Delta}^{p}_{(k),h},\bar{\Delta}^{Bonus,r}_{(k),h}$, respectively. 
    
    We also define \jj{a} variable $\Lambda^{\mathfrak{t}_{k}}_{w}(s,a)$ that quantifies the visitation\jj{:}
    \begin{align}
        \Lambda^{\mathfrak{t}_{k}}_{w}(s,a) = \left[ \lambda + \sum_{t=(1 \wedge k-w+1)}^{k} n_{t}(s,a) \right]^{-1}. \label{def:Biggamma}
    \end{align}
    \jj{It can be verified that}
    \begin{equation}
        \Gamma^{\mathfrak{t}_{k}}_w(s,a) = \beta \sqrt{\Lambda^{\mathfrak{t}_{k}}_{w}(s,a)}. \label{eqn:gamma_Lambda_relationship}
    \end{equation}
    \jj{As before, we simplify the notations} $\Lambda^{\mathfrak{t}_{k}}_{w}(s,a)$ \jj{and} $\Gamma^{\mathfrak{t}_{k}}_w(s,a)$ \jj{as} $\Lambda^{(k)}_{w}(s,a)$ \jj{and} $\Gamma^{(k)}_w(s,a)$. We define $r_{\text{max}},\widetilde{r}_{\text{max}},R_{(k+1)}^{\text{max}}$\jj{, and }$\widetilde{R}_{(k+1)}^{\text{max}}$ as follows\jj{:}
    \begin{align*}
        R_{(k+1)}^{\text{max}} &\coloneqq \max_{(s,a)}|R_{(k+1)}(s,a)|, \\
        r_{\text{max}} &\coloneqq \max_{1\leq k \leq K-1} R_{(k+1)}^{\text{max}} ,\\
        \widetilde{R}_{(k+1)}^{\text{max}} &\coloneqq \max_{(s,a)}|\widetilde{R}_{(k+1)}(s,a)|,\\
        \widetilde{r}_{\text{max}} &\coloneqq \max_{1\leq k \leq K-1} \widetilde{R}_{(k+1)}^{\text{max}}
    \end{align*}\
    and since $||\widehat{R}_{(k+1)}(s,a)||_{\infty} \leq ||\widetilde{R}_{(k+1)}(s,a)||_{\infty} + ||2\Gamma^{(k)}_w(s,a)||_{\infty} = \widetilde{R}_{(k+1)}^{\text{max}} + \frac{2\beta}{\sqrt{\lambda}}$, we define $\hat{r}^{k+1}_{\text{max}}$ as
    \begin{align*}
        \hat{r}_{(k+1)}^{\text{max}} &\coloneqq \widetilde{R}_{(k+1)}^{\text{max}} + \frac{2\beta}{\sqrt{\lambda}}.
    \end{align*}
    Also, since $\beta$ \jj{and} $\lambda$ are hyperparameters independent of $k$, \jj{we have that} 
        \begin{align}
            \hat{r}_{\text{max}} &= \widetilde{r}_{\text{max}} + \frac{2\beta}{\sqrt{\lambda}}.
        \label{eqn:r_max_and_r_tilde}
    \end{align}

\subsection{Proofs}
\label{appendix:proof}

\begin{proof}[\textbf{Proof of Theorem \ref{theorem1}}]
Following the definition of the dynamic regret (Definition \ref{def:dynamicregret})\jj{, it} can be separated into three terms\jj{:}
\begin{flalign*}
    &\mathfrak{R} \left( \{ \widehat{\pi}^{(k+1)}\}_{1:K-1},K) \right) \\  
    &= \sum_{k=1}^{K-1}\left( V^{\bigast,(k+1)}(s_0)- V^{\widehat{\pi}^{(k+1)},(k+1)}(s_0) \right) \\
    &= \underbrace{ \sum_{k=1}^{K-1} \left(V^{\bigast,(k+1)}(s_0)-\widehat{V}^{\bigast,(k+1)}(s_0) \right)}_{ \circled{1}} 
    + \underbrace{ \sum_{k=1}^{K-1} \left(\widehat{V}^{\bigast,(k+1)}(s_0)-\widehat{V}^{\widehat{\pi}^{(k+1)},(k+1)}(s_0) \right)}_{ \circled{2}} \\
    &+ \underbrace{ \sum_{k=1}^{K-1}\left( \widehat{V}^{\widehat{\pi}^{(k+1)},(k+1)}(s_0)-V^{\widehat{\pi}^{(k+1)},(k+1)}(s_0) \right)}_{ \circled{3}}
\end{flalign*}

\textbf{1. Upper bound on \circled{1}. }
The gap between $V^{\pi^{\bigast,(k+1)},(k+1)}(s_0)$ and $\widehat{V}^{\widehat{\pi}^{\bigast,(k+1)},(k+1)}(s_0)$ comes from the gap \jj{between} two optimal value functions evaluated \jj{for} two different MDPs: $\mathcal{M}_{(k+1)}$ \jj{and} $\widehat{\mathcal{M}}_{(k+1)}$. 

We \jj{will} first come up with \jj{an} upper bound \jj{on} the difference between $ Q^{\bigast,(k+1)}_{h}(s,a)$ and $\widehat{Q}^{\bigast,(k+1)}_{h}(s,a)$ for any $(s,a) \in \mathcal{S} \times \mathcal{A}$. The difference can be separated into three terms as follows:

\begin{align*}
    Q^{\bigast,(k+1)}_{h}(s,a) - \widehat{Q}^{\bigast,(k+1)}_{h}(s,a) \leq& \underbrace{\vert\vert Q^{\bigast,(k+1)}_{h}(s,a) - Q^{\bigast,(k+1)}_{\infty}(s,a) \vert\vert_{\infty}}_{\circled{1.1}} \\
    &+ \underbrace{\left( Q^{\bigast,(k+1)}_{\infty}(s,a) - \widehat{Q}^{\bigast,(k+1)}_{\infty}(s,a) \right)}_{\circled{1.2}}\\
    & + \underbrace{\vert\vert \widehat{Q}^{\bigast,(k+1)}_{h}(s,a) - \widehat{Q}^{\bigast,(k+1)}_{\infty}(s,a) \vert\vert_{\infty}}_{\circled{1.3}} 
\end{align*}

\textbf{1.1. Term\jj{s} $\circled{1.1}$ and $\circled{1.3}$. }

First, the term $\circled{1.1}$ can be bounded as follows\jj{:}
\begin{align*}
    \circled{1.1} &= \left|\left|\mathbb{E}_{\mathcal{M}_{(k+1)},\pi^\bigast}\Big[ \sum_{i=0}^{H-h-1} \gamma^{i} r_{i+h}^{(k+1)} - \sum_{i=0}^{\infty}\gamma^i r_i^{(k+1)}~|~s^{(k+1)}_h=s, a^{(k+1)}_h = a \Big]  \right|\right|_{\infty} \\ 
    & \leq \left|\sum_{i=H-h}^{\infty} \gamma^i r_{\text{max}} \right| \\
    & = \frac{\gamma^{H-h}}{1-\gamma} r_{\text{max}}
\end{align*}
Through a similar process, we can also obtain the upper bound\jj{:} $\circled{1.3} \leq \gamma^{H-h}/(1-\gamma) \hat{r}_{\text{max}}$. 

\textbf{1.2. Term $\circled{1.2}$. }

\jj{An} upper bound \jj{on} the term $\circled{1.2}$ can be \jj{obtained by} utilizing $\bar{\iota}^{(k+1)}_\infty(s,a)$ (Def \eqref{def:optimalModelPredictionError}). Then, the Q-function gap between $Q^{\bigast,(k+1)}_{\infty}$ \jj{and} $\widehat{Q}^{\bigast,(k+1)}_{\infty}$ can be represented using the \jj{Bellman} equation as follows\jj{:}
\begin{align}
    \circled{1.2} &= \left(Q^{\bigast,(k+1)}_{\infty} - \widehat{Q}^{\bigast,(k+1)}_{\infty}\right) (s,a) \label{eqn:up1_0}\\
    &=  \big(R_{(k+1)} + \gamma P_{(k+1)} V_{\infty}^{\bigast,(k+1)}\big) (s,a) - \widehat{Q}^{\bigast,(k+1)}_{\infty}(s,a) \label{eqn:up1_1} \\
    &=  \big( R_{(k+1)} + \gamma P_{(k+1)} \widehat{V}_{\infty}^{\bigast,(k+1)} - \widehat{Q}^{\bigast,(k+1)}_{\infty}\big)(s,a) + \gamma P_{(k+1)} \left(  V_{\infty}^{\bigast,(k+1)} - \widehat{ V}_{\infty}^{\bigast,(k+1)}\right)(s,a) \nonumber \\ 
    & \leq \bar{\iota}^{k+1}_\infty (s,a)+ \gamma  P_{(k+1)} \left(  V_{\infty}^{\bigast,(k+1)} - \widehat{ V}_{\infty}^{\bigast,(k+1)}\right)(s,a) \nonumber \\
    &=  \bar{\iota}^{k+1}_h(s,a)  + \gamma  P_{(k+1)} \left( \langle Q_{\infty}^{\bigast,(k+1)} , \pi^{\bigast,(k+1)} \rangle_{\mathcal{A}} -  \langle \widehat{Q}_{\infty}^{\bigast,(k+1)} , \widehat{\pi}^{\bigast,(k+1)} \rangle_{\mathcal{A}} \right)(s,a)\label{eqn:up1_2}\\ 
    &=  \bar{\iota}^{k+1}_\infty (s,a) + \gamma P_{(k+1)} \big( \langle Q_{\infty}^{\bigast,(k+1)} - \widehat{Q}_{\infty}^{\bigast,(k+1)} , \pi^{\bigast,(k+1)} \rangle_{\mathcal{A}} \nonumber \\
    &+ \langle \widehat{Q}_{\infty}^{\bigast,(k+1)} , \pi^{\bigast,(k+1)} - \widehat{\pi}^{\bigast,(k+1)} \rangle_{\mathcal{A}} \big)(s,a) \nonumber \\
    &\leq   \bar{\iota}^{k+1}_\infty(s,a) + \gamma P_{(k+1)} \left( \langle  Q_{\infty}^{\bigast,(k+1)} - \widehat{Q}_{\infty}^{\bigast,(k+1)}, \pi^{\bigast,(k+1)} \rangle_{\mathcal{A}} \right)(s,a) \label{eqn:up1_3}
\end{align}
where \eqref{eqn:up1_1} and \eqref{eqn:up1_2} hold by the definition of \jj{Bellman} equation (\eqref{eqn:bellmanOptimal_step} and \eqref{eqn:futurebellmanOptimal_step}). Equation \eqref{eqn:up1_3} holds by $\langle \widehat{Q}_{\infty}^{\bigast,(k+1)} , \pi^{\bigast,(k+1)} - \widehat{\pi}^{\bigast,(k+1)} \rangle_{\mathcal{A}} (s,a) \leq 0$ since $\widehat{\pi}^{\bigast,(k+1)}$ is the optimal policy of $\widehat{Q}^{\bigast,(k+1)}_{\infty}$.
We now define the matrix operator $(\mathbb{P}\circ \pi)(s,a) : \mathbb{R}^{|\mathcal{S}||\mathcal{A}|} \rightarrow \mathbb{R}^{|\mathcal{S}||\mathcal{A}|}$ as the transition matrix \jj{that captures how the} state\jj{-}action pair transition\jj{s} from $(s,a)$ to $(s^\prime,a^\prime)$ when following the policy $\pi$ in an environment with \jj{the} transition probability $\mathbb{P}$. Also, define the one-vector $\mathbbm{1}_{(s,a)} \in \mathbb{R}^{|\mathcal{S}|\mathcal{A}|}$ \jj{such} that the $(s,a)^{\text{th}}$ entity is one and \jj{the remaining entries are} zero. Then\jj{,} the equation \eqref{eqn:up1_0} \jj{becomes} the same as the $(s,a)^{\text{th}}$ entity of the vector $\mathbbm{1}_{(s,a)} \cdot (Q_{\infty}^{\bigast,(k+1)} - \widehat{Q}_{\infty}^{\bigast,(k+1)})(s,a)$. Also, the \jj{right-hand side} of equation \eqref{eqn:up1_3} can be represented as
\begin{align*}
    P_{(k+1)} \left( \langle  Q_{\infty}^{\bigast,(k+1)} - \widehat{Q}_{\infty}^{\bigast,(k+1)}, \pi^{\bigast,(k+1)} \rangle_{\mathcal{A}} \right)(s,a) &= (P_{(k+1)} \circ \pi^{\bigast,(k+1)}) \\
    &\cdot \left(\mathbbm{1}_{(s,a)} \cdot (Q_{\infty}^{\bigast,(k+1)} - \widehat{Q}_{\infty}^{\bigast,(k+1)})\right)(s,a)  \\
    &= (\mathbb{P}^{k+1}_{\pi^{\bigast}}) \left(\mathbbm{1}_{(s,a)} \cdot(Q_{\infty}^{\bigast,(k+1)} - \widehat{Q}_{\infty}^{\bigast,(k+1)}) \right) (s,a)
\end{align*}
where we denote $P_{(k+1)} \circ \pi^{\bigast,(k+1)} \coloneqq \mathbb{P}^{(k+1)}_{\pi^{\bigast}}$ for \jj{notational} simplicity.

Then, we can reformulate the inequality (between \eqref{eqn:up1_0} and \eqref{eqn:up1_3}) into a vector form \jj{which} holds element-wise for all $s,a$:  
\begin{align*}
    \left(\mathbbm{1}_{(s,a)} \cdot \left(Q^{\bigast,(k+1)}_{\infty} - \widehat{Q}^{\bigast,(k+1)}_{\infty}\right) \right)(s,a) \leq& \mathbbm{1}_{(s,a)} \cdot 
 \bar{\iota}^{(k+1)}_\infty (s,a) \\
 &+ \gamma(\mathbb{P}^{(k+1)}_{\pi^{\bigast}}) \left( \mathbbm{1}_{(s,a)} \cdot(Q_{\infty}^{\bigast,(k+1)} - \widehat{Q}_{\infty}^{\bigast,(k+1)}) \right) (s,a)
\end{align*}
Then, rearranging the above inequality yields \jj{that}
\begin{align}
    \mathbbm{1}_{(s,a)} \cdot \left(Q^{\bigast,(k+1)}_{\infty} - \widehat{Q}^{\bigast,(k+1)}_{\infty} \right) (s,a) 
    &\leq (\mathbb{I}- \gamma  \mathbb{P}^{k+1}_{\pi^{\bigast}})^{-1} \mathbbm{1}_{(s,a)} \cdot \bar{\iota}^{k+1}_\infty  (s,a) \label{eqn:up1_4} \\
    &= \frac{1}{1-\gamma} \bar{\iota}^{k+1}_\infty  (s,a) \nonumber
\end{align}
Now, note that $(\mathbb{I}- \gamma  \mathbb{P}^{(k+1)}_{\pi^{\bigast}})^{-1}$ can be expanded with \jj{an} infinite summation of \jj{the} matrix operator $P_{(k+1)} \circ \pi^{\bigast,(k+1)}$ as $ (\mathbb{I}- \gamma  \mathbb{P}^{(k+1)}_{\pi^{\bigast}})^{-1} = \mathbb{I} + \gamma  \mathbb{P}^{(k+1)}_{\pi^{\bigast}} + (\gamma  \mathbb{P}^{(k+1)}_{\pi^{\bigast}})^2 + ...$\jj{s. Since}, $\mathbbm{1}_{(s,a)}$ can be viewed as the Dirac delta state\jj{-}action distribution that always yields $(s,a)$, \jj{it holds that} $\nu_{(s,a)}^{\pi^{\bigast,(k+1)},(k+1)} = (\mathbb{I}- \gamma  \mathbb{P}^{(k+1)}_{\pi^{\bigast}})^{-1} \mathbbm{1}_{(s,a)}$, where $\nu$ is the unnormalized occupancy measure of $(s,a)$ \jj{in light of} Definition \eqref{def:Nu}.
Then taking the $l_1$ norm over the inequality \eqref{eqn:up1_4} yields the \jj{that}
\begin{align}
     \left| \left| \mathbbm{1}_{(s,a)} \cdot \left(Q^{\bigast,(k+1)}_{\infty} - \widehat{Q}^{\bigast,(k+1)}_{\infty}\right) (s,a) \right| \right|_1
    &\leq  \left| \left| (\mathbb{I}- \gamma  \mathbb{P}^{k+1}_{\pi^{\bigast}})^{-1} \mathbbm{1}_{(s,a)} \cdot \bar{\iota}^{k+1}_\infty  (s,a) \right| \right|_1  \nonumber\\
    &=  \left| \left| (\mathbb{I}- \gamma  \mathbb{P}^{k+1}_{\pi^{\bigast}})^{-1} \mathbbm{1}_{(s,a)} \right| \right|_1  \cdot \left| \bar{\iota}^{k+1}_\infty  (s,a) \right| \nonumber \\
    &= \frac{1}{1-\gamma} \left| \bar{\iota}^{k+1}_\infty (s,a) \right| \label{eqn:up1_5}
\end{align}
Equation \eqref{eqn:up1_5} holds since $\nu_{(s,a)}^{\pi^{\bigast,(k+1)},(k+1)}$ is \jj{an} unnormalized probability distribution.

Then, for \jj{every} $(s,a,h) \in \mathcal{S} \times \mathcal{A} \times [H]$, \jj{it follows from} combining the terms $\circled{1.1},\circled{1.2}$ and $\circled{1.3}$ \jj{that}
\begin{equation}
    Q^{\bigast,(k+1)}_{h}(s,a) - \widehat{Q}^{\bigast,(k+1)}_{h}(s,a)  \leq \frac{\gamma^{H-h}}{1-\gamma}(r_{\text{max}} + \hat{r}_{\text{max}}) + \frac{1}{1-\gamma} \left| \bar{\iota}^{(k+1)}_\infty  (s,a) \right| \nonumber
\end{equation}

\textbf{1.3. Combining the terms $\circled{1.1},\circled{1.2}$ and $\circled{1.3}$. }

Finally, \jj{an} upper bound \jj{on} $\circled{1}$ is \jj{derived as}
\begin{align}
    \circled{1}&=\sum_{k=1}^{K-1} \left(V^{\pi^{\bigast,(k+1)},(k+1)}(s_0)-\widehat{V}^{\widehat{\pi}^{\bigast,(k+1)},(k+1)}(s_0) \right) \nonumber \\
    & \leq \sum_{k=1}^{K-1} \left\vert \left\vert Q^{\bigast,(k+1)} - \widehat{Q}^{\bigast,(k+1)} \right\vert \right\vert_{\infty} \nonumber \\
    & =\sum_{k=1}^{K-1} \cdot \frac{\gamma^{H}}{1-\gamma}(r_{\text{max}} + \hat{r}_{\text{max}}) + \frac{1}{1-\gamma} \sum_{k=1}^{K-1} || \bar{\iota}_{\infty}^{k+1} ||_{\infty} \nonumber \\ 
    & = (K-1)\cdot \frac{\gamma^{H}}{1-\gamma}(r_{\text{max}} + \hat{r}_{\text{max}}) + \frac{1}{1-\gamma}  \bar{\iota}_{\infty}^{K} \label{eqn:up1_final}
\end{align}
where we have defined $\bar{\iota}_{\infty}^{K} \coloneqq \sum_{k=1}^{K-1} \left| \left| \bar{\iota}_{\infty}^{(k+1)} \right| \right|_\infty$ in \jj{Theorem} \ref{theorem1}.

\textbf{2. Upper bound \jj{on} \circled{2}. }

The gap between $\widehat{V}^{\bigast,(k+1)}(s_0)$ and $\widehat{V}^{\widehat{\pi}^{(k+1)},(k+1)}(s_0)$ comes from \jj{the} optimization error between \jj{the} optimal policy $\widehat{\pi}^{\bigast,(k+1)}$ and \jj{the} policy $\widehat{\pi}^{(k+1)}$\jj{, which} are both driven from \jj{the} same MDP $\widehat{\mathcal{M}}_{(k+1)}$ . We also separate \jj{this} gap into three terms\jj{:}
\begin{align}
    \circled{2}\text{'s}~(k)^{\text{th}}~\text{term} &=  \widehat{V}^{\bigast,(k+1)}(s_0) - \widehat{V}^{\widehat{\pi}^{(k+1)},(k+1)}(s_0) \nonumber \\
    &= \left(  \widehat{V}^{\bigast,(k+1)}(s_0) -  \widehat{V}_{\infty}^{\bigast,(k+1)}(s_0) \right)  + \left( \widehat{V}_{\infty}^{\bigast,(k+1)}(s_0)-\widehat{V}_{\infty}^{\widehat{\pi}^{(k+1)},(k+1)}(s_0) \right) + \nonumber \\
    &+ \left(   \widehat{V}_{\infty}^{\widehat{\pi}^{(k+1)},(k+1)}(s_0) - \widehat{V}^{\widehat{\pi}^{(k+1)},(k+1)}(s_0) \right) \label{eqn:lemmaeqn6} \\ 
    & \leq \underbrace{\left( \widehat{V}_{\infty}^{\bigast,(k+1)}(s_0)-\widehat{V}_{\infty}^{\widehat{\pi}^{(k+1)},(k+1)}(s_0) \right)}_{\circled{2.1}} + \frac{2\gamma^{H} \widehat{r}_{\text{max}}}{1-\gamma} \label{eqn:lemmaeqn7}
\end{align}
where \jj{the subscript} $\infty$ in the \jj{notations} $\widehat{V}_{\infty}^{\pi,(k+1)}(s_0)$ \jj{and} $\widehat{V}_{\infty,\tau}^{\pi,(k+1)}(s_0)$ \jj{indicate the} forecasted value function \jj{and the} forecasted entropy-regularized value function when $H=\infty$ (infinite horizon MDPs). Equation (\ref{eqn:lemmaeqn6}) holds since $\widehat{V}^{\pi,(k+1)}(s) - \widehat{V}_{\infty}^{\pi,(k+1)}(s) = \mathbb{E}_{\widehat{\mathcal{M}}_{(k+1)},\pi}\left[ \sum_{h=H}^{\infty} \gamma^{h} \widehat{r}_{h}^{(k+1)}~\Big\vert~ s = \widehat{s}^{(k+1)}_{0} \right] \leq \frac{\gamma^H}{1-\gamma}\widehat{r}_{\text{max}}$ holds for \jj{all} $\pi \in \Pi$.

\textbf{2.1. Upper bound \jj{on} \circled{2} - NPG without entropy regularization (\texttt{Alg}). }
The term $\circled{2.1}$ in \eqref{eqn:lemmaeqn7} can be bounded \jj{as} 
\begin{align}
    \circled{2.1} =& \widehat{V}_{\infty}^{\bigast,(k+1)}(s_0)-\widehat{V}_{\infty}^{\widehat{\pi}^{(k+1)},(k+1)}(s_0) \nonumber \\
     \leq& \frac{\log{|\mathcal{A}|}}{\eta G} + \frac{1}{(1-\gamma)^2 G} \label{eqn:withoutentropy}
\end{align}
\jj{due to} Theorem 5.3 \jj{in} \cite{agarwal2021theory}. Now, combining \ref{eqn:lemmaeqn7} and \ref{eqn:withoutentropy} offers \jj{an} upper bound of the term $\circled{2}$'s $(k)^{\text{th}}$ term as follows\jj{:}
\begin{align*}
    \circled{2}\text{'s}~(k)^{\text{th}}~\text{term} &= \widehat{V}^{\widehat{\pi}^{\bigast,(k+1)},(k+1)}(s_0)-\widehat{V}^{\widehat{\pi}^{(k+1)},(k+1)}(s_0)  \\ 
    &\leq \frac{1}{(1-\gamma)^2 G}+ \frac{\log{|\mathcal{A}|}}{\eta G}  + \frac{2\gamma^{H} \widehat{r}_{\text{max}}}{1-\gamma} 
\end{align*}
\jj{Hence,}
\begin{align}
    \circled{2} &= \sum_{k=1}^{K-1}\left( \widehat{V}^{\widehat{\pi}^{\bigast,(k+1)},(k+1)}(s_0)-\widehat{V}^{\widehat{\pi}^{(k+1)},(k+1)}(s_0) \right) \nonumber \\ 
    &\leq (K-1) \left( \frac{1}{(1-\gamma)^2 G}+ \frac{\log{|\mathcal{A}|}}{\eta G}  + \frac{2\gamma^{H} \widehat{r}_{\text{max}}}{1-\gamma} \right) \label{eqn:up2_withoutregularzation}
\end{align}

\textbf{2.2. Upper bound \jj{on} \circled{2} - NPG with entropy regularization ($\texttt{Alg}_\tau$). }

\jj{The} term $\circled{2.1}$ in \eqref{eqn:lemmaeqn7} can be further bounded as follows\jj{:} 
\begin{align}
    \circled{2.1} =& \widehat{V}_{\infty}^{\bigast,(k+1)}(s_0)-\widehat{V}_{\infty}^{\widehat{\pi}^{(k+1)},(k+1)}(s_0) \nonumber \\
    = &  \left( \widehat{V}_{\infty}^{\bigast,(k+1)}(s_0)-\widehat{V}_{\infty,\tau}^{\bigast,(k+1)}(s_0) \right) + \left( \widehat{V}_{\infty,\tau}^{\bigast,(k+1)}(s_0)-\widehat{V}_{\infty,\tau}^{\widehat{\pi}^{(k+1)},(k+1)}(s_0) \right) \nonumber \\
     + & \left( \widehat{V}_{\infty,\tau}^{\widehat{\pi}^{(k+1)},(k+1)}(s_0)-\widehat{V}_{\infty}^{\widehat{\pi}^{(k+1)},(k+1)}(s_0) \right) \nonumber \\
  \leq & \left\vert \left\vert \widehat{V}_{\infty}^{\bigast,(k+1)}(s_0)-\widehat{V}_{\infty,\tau}^{\bigast,(k+1)}(s_0) \right\vert \right\vert_{\infty} + \left\vert \left\vert \widehat{V}_{\infty,\tau}^{\bigast,(k+1)}(s_0)-\widehat{V}_{\infty,\tau}^{\widehat{\pi}^{(k+1)},(k+1)}(s_0) \right\vert \right\vert_{\infty} \nonumber \\
    + & \left\vert \left\vert \widehat{V}_{\infty,\tau}^{\widehat{\pi}^{(k+1)},(k+1)}(s_0)-\widehat{V}_{\infty}^{\widehat{\pi}^{(k+1)},(k+1)}(s_0) \right\vert \right\vert_{\infty} \nonumber \\
    \leq & \underbrace{\left\vert \left\vert \widehat{V}_{\infty,\tau}^{\bigast,(k+1)}(s_0)-\widehat{V}_{\infty,\tau}^{\widehat{\pi}^{(k+1)},(k+1)}(s_0) \right\vert \right\vert_{\infty}}_{\circled{2.2}} + \frac{2\tau \log \vert \mathcal{A} \vert}{1-\gamma}  \label{eqn:lemmaeqn8}
\end{align}
where \eqref{eqn:lemmaeqn8} holds since $\left\vert \left\vert \widehat{V}_{\infty}^{\pi,(k+1)}(s_0)-\widehat{V}_{\infty,\tau}^{\pi,(k+1)}(s_0) \right\vert \right\vert_{\infty} = \tau \max_s \left\vert \mathcal{H}(s,\pi) \right\vert \leq \tau  \frac{\log \left\vert \mathcal{A} \right\vert}{1-\gamma}$ holds for \jj{all} $\pi$.

\jj{We now bound} the term $\circled{2.2}$ \jj{in} (\ref{eqn:lemmaeqn8}). With the policy-update rule of \texttt{ProST-T} (Algorithm \ref{alg:FT-MBPO} in Appendix \ref{section:FT-MBPO}), suppose \jj{that} for \jj{a} given $g \in [\Delta_\pi]$, we \jj{have obtained an} \textit{inexact} soft $Q$-function value of the policy $\widehat{\pi}_{(g)}$ as $\widetilde{Q}^{\widehat{\pi}_{(g)}}_{\tau}$\jj{, where} $\widehat{Q}^{\widehat{\pi}_{(g)}}_{\tau}$ denotes \jj{an} \textit{exact} soft forecated $Q$-function value \jj{and} $g$ is the iteration index. The approximation gap $|\widetilde{Q}^{\widehat{\pi}_{(g)}}_{\tau} - \widehat{Q}^{\widehat{\pi}_{(g)}}_{\tau} |$ \jj{results from} computing $Q$ using \jj{a finite number of samples}. For \jj{a} hyperparameter $\delta$, \jj{let} the maximum of \jj{the} approximation gap over $(s,a)$ is smaller than $\delta$, namely $||\widetilde{Q}_{\tau}^{\widehat{\pi}_{(g)}}- \widehat{Q}_{\tau}^{\widehat{\pi}_{(g)}} ||_\infty \leq \delta$ holds. Then, for iteration $g=1,2,..,\Delta_\pi$, the policy-update rule of \texttt{ProST-T} can be written \jj{as} 
\begin{align*}
    &\widehat{\pi}_{(g+1)}(\cdot|s) =\frac{1}{Z_{(g)}} \cdot  \left( \widehat{\pi}_{(g)} (\cdot|s) \right)^{1-\frac{\eta\tau}{1-\gamma}} \exp{\left( \frac{\eta\widetilde{Q}_{\tau}^{\widehat{\pi}_{(g)}}(s,a)}{1-\gamma} \right)} \\ 
    & \text{where}~~ ||\widetilde{Q}_{\tau}^{\widehat{\pi}_{(g)}}(s,a)- \widehat{Q}_{\tau}^{\widehat{\pi}_{(g)}}(s,a) ||_\infty \leq \delta~~\text{for}~~ \forall (s,a) \in \mathcal{S} \times \mathcal{A}
\end{align*}
 where $Z_{(g)}(s) = \sum_{a\in \mathcal{A}} \left( \widehat{\pi}_{(g)}(a|s) \right)^{1-\frac{\eta\tau}{1-\gamma}} \exp{\left( (\eta\widehat{Q}_{\tau}^{\widehat{\pi}_{(g)}}(s,a)) / (1-\gamma) \right)}$\jj{.}

\jj{In light of} Theorem 2 \jj{in} \cite{cen2022fast}, when the learning rate \jj{is such that} $0 \leq \eta \leq (1-\gamma)/\tau$, then the approximate entropy-regularized NPG method satisfies the linear convergence theorem for \jj{every} $g\in [\Delta_\pi]$:
\begin{equation}
    \label{eqn:1}
    \vert\vert \widehat{Q}^{\bigast,(k+1)}_{\tau} - \widehat{Q}^{\widehat{\pi}_{(g)}}_{\tau} \vert\vert_{\infty} \leq \gamma \left[ (1-\eta\tau)^{g-1} C_1 + C_2\right]
\end{equation}
\begin{equation}
    \label{eqn:2}
    \vert\vert \log \widehat{\pi}^{\bigast,(k+1)} - \log \widehat{\pi}_{(g)} \vert\vert_{\infty} \leq 2\tau^{-1} \left[ (1-\eta\tau)^{g-1} C_1 + C_2 \right]
\end{equation}
where
\begin{align}
    C_{1} & := ||\widehat{Q}^{\bigast,(k+1)}_{\tau} - \widehat{Q}^{\widehat{\pi}_{(0)}}_{\tau}||_{\infty} + 2\tau \left( 1- \frac{\eta\tau}{1-\gamma}\right) \vert\vert \log \widehat{\pi}^{\bigast,(k+1)} - \log \widehat{\pi}_{(0)} \vert\vert_{\infty} \nonumber \\
     &= ||\widehat{Q}^{\bigast,(k+1)}_{\tau} - \widehat{Q}^{\pi^{(k)}}_{\tau}||_{\infty} + 2\tau \left( 1- \frac{\eta\tau}{1-\gamma}\right) \vert\vert \log \widehat{\pi}^{\bigast,(k+1)} - \log \widehat{\pi}^{(k)} \vert\vert_{\infty} \label{eqn:3} \\ 
    C_{2} & := \frac{2\delta}{1-\gamma} \left( 1+\frac{\gamma}{\eta\tau} \right)     \label{eqn:4}
\end{align}
The equation (\ref{eqn:3}) holds since the policy that the agent executes at the wall-clock time $\mathfrak{t}_k$ (episode $k$), i.e.\jj{,} $\pi^{(k)}$, is same as the initial policy of the policy iteration, i.e.\jj{,} $\widehat{\pi}_{(0)}$\jj{,} at the wall-clock time $\mathfrak{t}_k$. Also, the policy that the agent executes at the wall-clock time $\mathfrak{t}_{k+1}$, i.e.\jj{,} $\widehat{\pi}^{(k+1)}$, is same as the policy after $\Delta_\pi$ steps of \jj{the} soft policy iteration, i.e.\jj{,} $\widehat{\pi}_{(\Delta_\pi)}$ at the wall-clock time $\mathfrak{t}_{k+1}$.

\jj{Now}, the term $\circled{2.2}$ can be bounded as follows\jj{:}

\begin{align}
    \circled{2.2}& = ||\widehat{V}^{\bigast,(k+1)}_{\tau} - \widehat{V}^{\widehat{\pi}^{(k+1)}}_{\tau}||_{\infty} \nonumber \\
    &= ||\widehat{V}^{\bigast,(k+1)}_{\tau} - \widehat{V}^{\widehat{\pi}_{(\Delta_\pi)}}_{\tau}||_{\infty} \nonumber \\
    & \leq ||\widehat{Q}^{\bigast,(k+1)}_{\tau} - \widehat{Q}^{\widehat{\pi}_{(\Delta_\pi)}}_{\tau}||_{\infty} + \tau \vert\vert \log \widehat{\pi}^{\bigast,(k+1)} - \log \widehat{\pi}_{(g)} \vert\vert_{\infty} \nonumber \\
    & \leq \left( \gamma + 2 \right) \left[ (1-\eta\tau)^{\Delta_\pi-1} C_1 + C_2\right]
    \label{eqn:lemmaeqn10}
\end{align}

Combining (\ref{eqn:lemmaeqn7},\ref{eqn:lemmaeqn8} and \ref{eqn:lemmaeqn10}) offers \jj{an} upper bound \jj{on} the term $\circled{2}$'s $k^{(th)}$ term as follows, 
\begin{align}
    \circled{2}\text{'s}~(k)^{th}~\text{term} &= \widehat{V}^{\widehat{\pi}^{\bigast,(k+1)},(k+1)}(s_0)-\widehat{V}^{\widehat{\pi}^{(k+1)},(k+1)}(s_0)  \nonumber \\ 
    &\leq \left( \gamma + 2 \right) \left[ (1-\eta\tau)^{\Delta_\pi-1} C_1 + C_2\right] +  \frac{2\gamma^{H} \widehat{r}_{\text{max}}}{1-\gamma} + \frac{2\tau \log \vert \mathcal{A} \vert}{1-\gamma}
    \label{eqn:bound2}
\end{align}
\jj{Hence,}
\begin{align}
    \circled{2} &= \sum_{k=1}^{K-1}\left( \widehat{V}^{\widehat{\pi}^{\bigast,(k+1)},(k+1)}(s_0)-\widehat{V}^{\widehat{\pi}^{(k+1)},(k+1)}(s_0) \right) \nonumber \\ 
    &\leq (K-1) \left( \left( \gamma + 2 \right) \left[ (1-\eta\tau)^{\Delta_\pi-1} C_1 + C_2\right] +  \frac{2\gamma^{H} \widehat{r}_{\text{max}}}{1-\gamma} + \frac{2\tau \log \vert \mathcal{A} \vert}{1-\gamma} \right)
    \label{eqn:up2_final}
\end{align}
where \eqref{eqn:bound2} and \eqref{eqn:up2_final} \jj{hold} when $0 \leq \eta \leq (1-\gamma)/\tau$

\textbf{3. Upper bound \jj{on} \circled{3}. }

\jj{By recalling} Definition \eqref{def:estimatedModelPredictionError}, note that $\iota^{(k+1)}_h(\widehat{s}^{(k+1)}_h,\widehat{a}^{(k+1)}_h)$ is \jj{an} \textit{empirical} estimated model prediction error, \jj{measuring} the gap between $\mathcal{M}_{(k+1)}$ and $\widehat{\mathcal{M}}_{(k+1)}$. Specifically, at episode $k$, the \texttt{ProST} algorithm creates the future MDP $\widehat{\mathcal{M}}_{(k+1)}$ and evaluates $\widehat{V}$ \jj{and} $\widehat{Q}$ using $\widehat{\pi}^{(k+1)}$\jj{. Subsequently} at episode $k+1$, the agent uses $\widehat{\pi}^{(k+1)}$ to rollout a trajectory $\{ s^{(k+1)}_0,a^{(k+1)}_0,s^{(k+1)}_1,a^{(k+1)}_1,...,s^{(k+1)}_{H-1},a^{(k+1)}_{H-1},s^{(k+1)}_{H}\}$. \jj{Based on this observation, one can write} 
\begin{align}
    \iota^{(k+1)}_h(s^{(k+1)}_h,a^{(k+1)}_h) =&  R_{(k+1)}(s^{(k+1)}_h,a^{(k+1)}_h) + \gamma (P_{(k+1)} \widehat{V}^{\widehat{\pi}^{(k+1)},(k+1)}_{h+1})(s^{(k+1)}_h,a^{(k+1)}_h) \nonumber\\
    -&\widehat{Q}^{\widehat{\pi}^{(k+1)},(k+1)}_{h}(s^{(k+1)}_h,a^{(k+1)}_h) \nonumber \\ 
    =&  R_{(k+1)}(s^{(k+1)}_h,a^{(k+1)}_h) + \gamma (P_{(k+1)} \widehat{V}^{\widehat{\pi}^{(k+1)},(k+1)}_{h+1})(s^{(k+1)}_h,a^{(k+1)}_h) \nonumber \\ 
    &- Q^{\widehat{\pi}^{(k+1)},(k+1)}_{h}(s^{(k+1)}_h,a^{(k+1)}_h) + Q^{\widehat{\pi}^{(k+1)},(k+1)}_{h}(s^{(k+1)}_h,a^{(k+1)}_h) \nonumber \\
    &- \widehat{Q}^{\widehat{\pi}^{(k+1)},(k+1)}_{h}(s^{(k+1)}_h,a^{(k+1)}_h) \nonumber \\ 
    =&  \gamma P_{(k+1)} (\widehat{V}^{\widehat{\pi}^{(k+1)},(k+1)}_{h+1}-V^{\widehat{\pi}^{(k+1)},(k+1)}_{h+1})(s^{(k+1)}_h,a^{(k+1)}_h) \nonumber\\ 
    &+ Q^{\widehat{\pi}^{(k+1)},(k+1)}_{h}(s^{(k+1)}_h,a^{(k+1)}_h) - \widehat{Q}^{\widehat{\pi}^{(k+1)},(k+1)}_{h}(s^{(k+1)}_h,a^{(k+1)}_h) \label{eqn:ddd}
\end{align}
Equation \eqref{eqn:ddd} holds \jj{due to} \eqref{eqn:futurebellman_step}. Now, \jj{we} define the operator $\widehat{\mathfrak{I}}^{(k+1)}$ for \jj{a} function $f:\mathcal{S}\times\mathcal{A} \rightarrow \mathbb{R}$ as follows\jj{:}
\begin{align*}
    (\widehat{\mathfrak{I}}^{(k+1)}f)(s) \coloneqq \langle f(s,\cdot), \widehat{\pi}^{(k+1)}(\cdot|s) \rangle_{\mathcal{A}}
\end{align*}

Recall that $\widehat{V}^{\widehat{\pi}^{(k+1)},(k+1)}_h(s) = \big< \widehat{Q}^{\widehat{\pi}^{(k+1)},(k+1)}_h,\widehat{\pi}^{(k+1)} \big>_\mathcal{A}$ and $V^{\widehat{\pi}^{(k+1)},(k+1)}_h(s) = \big< Q^{\widehat{\pi}^{(k+1)},(k+1)}_h,\widehat{\pi}^{(k+1)} \big>_\mathcal{A}$ \jj{in light of} \eqref{eqn:futurebellman_step} and \eqref{eqn:bellman_step}. Then, the gap between $\widehat{V}_{h+1}^{\widehat{\pi}^{(k+1)},(k+1)}(s^{(k+1)}_h)$ and $V^{\widehat{\pi}^{(k+1)},(k+1)}_{h+1}(s^{(k+1)}_h)$ can be expanded as
\begin{align*}
    & \widehat{V}_{h}^{\widehat{\pi}^{(k+1)},(k+1)}(s^{(k+1)}_h) - V^{\widehat{\pi}^{(k+1)},(k+1)}_{h}(s^{(k+1)}_h) \\
    & = \left( \widehat{\mathfrak{I}}^{(k+1)}\left(  \widehat{Q}_{h}^{\widehat{\pi}^{(k+1)},(k+1)} - Q^{\widehat{\pi}^{(k+1)},(k+1)}_{h} \right) \right)(s^{(k+1)}_h) \\ 
    & = \left( \widehat{\mathfrak{I}}^{(k+1)}\left(  \widehat{Q}_{h}^{\widehat{\pi}^{(k+1)},(k+1)} - Q^{\widehat{\pi}^{(k+1)},(k+1)}_{h} \right) \right)(s^{(k+1)}_h) - \iota^{(k+1)}_h(s^{(k+1)}_h,a^{(k+1)}_h)  \\ 
    &\text{  }+ \gamma P_{(k+1)} (\widehat{V}^{\widehat{\pi}^{(k+1)},(k+1)}_{h+1}-V^{\widehat{\pi}^{(k+1)},(k+1)}_{h+1})(s^{(k+1)}_h,a^{(k+1)}_h) \\  
    &\text{  }+ \left( Q^{\widehat{\pi}^{(k+1)},(k+1)}_{h} - \widehat{Q}^{\widehat{\pi}^{(k+1)},(k+1)}_{h}\right) (s^{(k+1)}_h,a^{(k+1)}_h) \\ 
\end{align*}
Now, we define two sequences $\{D^{(k+1)}_{h,1}\}$ \jj{and} $\{D^{(k+1)}_{h,1}\}$ \jj{, where} $(k,h) = (0,0), (0,1),...,(K-1,H)$. We define $D^{(k+1)}_{h,1}$ \jj{and} $D^{(k+1)}_{h,2}$ as
\begin{align*}
    D^{(k+1)}_{h,1} \coloneqq& \gamma^h \left( \widehat{\mathfrak{I}}^{(k+1)}\left(  \widehat{Q}_{h}^{\widehat{\pi}^{(k+1)},(k+1)} - Q^{\widehat{\pi}^{(k+1)},(k+1)}_{h} \right) \right)(s^{(k+1)}_h)  \\ 
    &- \gamma^h \left( \widehat{Q}^{\widehat{\pi}^{(k+1)},(k+1)}_{h} - Q^{\widehat{\pi}^{(k+1)},(k+1)}_{h} \right) (s^{(k+1)}_h,a^{(k+1)}_h) \\
    D^{(k+1)}_{h,2} \coloneqq& \gamma^{h+1} P_{(k+1)} (\widehat{V}^{\widehat{\pi}^{(k+1)},(k+1)}_{h+1}-V^{\widehat{\pi}^{(k+1)},(k+1)}_{h+1})(s^{(k+1)}_h,a^{(k+1)}_h) \\
    &- \gamma^{h+1}  \left( \widehat{V}^{\widehat{\pi}^{(k+1)},(k+1)}_{h+1}-V^{\widehat{\pi}^{(k+1)},(k+1)}_{h+1} \right) (s^{(k+1)}_{h+1})
\end{align*}

Therefore, we have the following recursive formula over $h$\jj{:}
\begin{align*}
    &\gamma^h\left(\widehat{V}_{h}^{\widehat{\pi}^{(k+1)},(k+1)} - V^{\widehat{\pi}^{(k+1)},(k+1)}_{h}\right)(s^{(k+1)}_h) \\
    &= D^{(k+1)}_{h,1} + D^{(k+1)}_{h,2} + \gamma^{h+1} \left( \widehat{V}^{\widehat{\pi}^{(k+1)},(k+1)}_{h+1}-V^{\widehat{\pi}^{(k+1)},(k+1)}_{h+1} \right)\left( s^{(k+1)}_{h+1}\right) - \gamma^h \iota^{(k+1)}_h(s^{(k+1)}_h,a^{(k+1)}_h)
\end{align*}
\jj{The} summation over $h=0,1,..,H-1$ yields that
\begin{align*}
     &\widehat{V}_{0}^{\widehat{\pi}^{(k+1)},(k+1)}(s^{(k+1)}_0) - V^{\widehat{\pi}^{(k+1)},(k+1)}_{0}(s^{(k+1)}_0) \\
    &= \sum_{h=0}^{H-1} \left(D^{(k+1)}_{h,1} + D^{(k+1)}_{h,2} \right) -  \sum_{h=0}^{H-1} \gamma^h \iota^{(k+1)}_h(s^{(k+1)}_h,a^{(k+1)}_h).
\end{align*}
Now, for every $(k,h) \in [K] \times [H]$, we define $\mathcal{F}^{(k)}_{h,1}$ as a $\sigma-$algebra generated by state-action sequences $\{(s^{\tau}_i,a^{\tau}_i)\}_{(\tau,i) \in [k-1]\times[H]} \cup \{ (s_i^k,a_i^k) \}_{i \in [h]}$  and define $\mathcal{F}^{(k)}_{h,2}$ as \jj{a} $\sigma$-algebra generated by $\{(s^{\tau}_i,a^{\tau}_i)\}_{(\tau,i) \in [k-1]\times[H]} \cup \{ (s_i^k,a_i^k) \}_{i \in [h]} \cup \{ s^{(k)}_{h+1}\}$. A filtration  $\{\mathcal{F}^{(k)}_{h,m}\}_{(k,h,m) \in [K] \times [H] \times [2]}$ is a sequence of $\sigma$- algebras in terms of the time index $t(k,h,m) = 2(k-1)H + 2h + m$ such that $\mathcal{F}^{(k)}_{h,m} \subset \mathcal{F}^{k^\prime}_{h^\prime,m^\prime}$ for  every $t(k,h,m) \leq t((k^\prime),h^\prime,m^\prime)$. The \jj{estimates} $\widehat{V}^{\pi,(k+1)}_h$ \jj{and} $\widehat{Q}^{\pi,(k+1)}_h$ are $\mathcal{F}^{(k+1)}_{1,1}$ measurable since they are forecasted from \jj{the} past $k$ historical trajectories. Now, since $D^{(k+1)}_{h,1} \in \mathcal{F}^{(k+1)}_{h,1}$ and $D^{(k+1)}_{h,2} \in \mathcal{F}^{(k+1)}_{h,2}$ \jj{hold,} $\mathbb{E}[D^{(k+1)}_{h,1} | \mathcal{F}^{(k+1)}_{h-1,2}]=0$ and $\mathbb{E}[D^{(k+1)}_{h,2} | \mathcal{F}^{(k+1)}_{h,1}]=0$. Notice that $t(k,0,2) = t(k-1,H,2)$ and $\mathcal{F}^{(k)}_{0,2}=\mathcal{F}^{(k-1)}_{H,2}$ for $\forall k \geq 2$. Therefore, \jj{one} can define \jj{a} martingale sequence adapted to the filtration $\{\mathcal{F}^{(k)}_{h,m}\}_{(k,h,m) \in [K] \times [H] \times [2]}$\jj{:}
\begin{align*}
    s^{(k+1)}_{h,j} = \sum_{k^\prime=1}^{k}\sum_{h^\prime=0}^{H-1} \left(D^{k^\prime}_{h^\prime,1} + D^{k^\prime}_{h^\prime,2} \right) + \sum_{h^\prime=0}^{h} \left(D^{(k+1)}_{h^\prime,1} + D^{(k+1)}_{h^\prime,2} \right) + \sum_{(k^\prime,h^\prime,j)\in[K]\times[H]\times[2]}D^{k^\prime}_{h^\prime,j}
\end{align*}
\jj{Let}
\begin{align*}
    \sum_{k=1}^{K-1}\sum_{h=0}^{H-1}  \left(  D^{(k+1)}_{h,1} + D^{(k+1)}_{h,2} \right) = S^{K-1}_{H,2}
\end{align*}
Since $\gamma^h\widehat{Q}_h^{\widehat{\pi}^{(k+1)},(k+1)},\gamma^{h+1}\widehat{V}_{h+1}^{\widehat{\pi}^{(k+1)},(k+1)} \in [0,\hat{r}_{\text{max}}/(1-\gamma)]$ and $\gamma^h Q_h^{\widehat{\pi}^{(k+1)},(k+1)},\gamma^{h+1} V_{h+1}^{\widehat{\pi}^{(k+1)},(k+1)} \in [0,r_{\text{max}}/(1-\gamma)]$\jj{, it holds that} $|D^{(k+1)}_{h,1}|,~|D^{(k+1)}_{h,s}| \leq (r_{\text{max}} \vee \hat{r}_{\text{max}})/(1-\gamma)$ for $\forall (k,h) \in [K-1] \times [H]$. Then, by the Azuma-Hoeffding inequlaity, the following inequality holds\jj{:}
\begin{equation*}
    \mathbb{P} \left( |S^{K-1}_{H,2}| \leq s \right) \geq 2 \exp \left( \frac{-s^2}{16 \left( \frac{r_{\text{max}} \vee \hat{r}_{\text{max}}}{1-\gamma}\right)^2 \cdot (K-1)H }\right)
\end{equation*}
For any $p \in (0,1)$, if we set $s = 4( r_{\text{max}} \vee \hat{r}_{\text{max}})(1-\gamma)^{-1}\sqrt{ (K-1)H \log(4/p)}$, then the inequality holds with probability at least $1-p/2$. \jj{The} term $\circled{3}$ can be bounded as 
\begin{align}
    \circled{3} &= \sum_{k=1}^{K-1} \sum_{h=0}^{H-1} \left(D^{(k+1)}_{h,1} + D^{(k+1)}_{h,2} \right) - \sum_{k=1}^{K-1} \sum_{h=0}^{H-1} \gamma^h \iota^{(k+1)}_h(s^{(k+1)}_h,a^{(k+1)}_h) \nonumber \\
    &\leq  \frac{4( r_{\text{max}} \vee \hat{r}_{\text{max}})}{1-\gamma} \sqrt{ (K-1)H \log(4/p)} - \iota^{K}_H \label{eqn:up3_final}
\end{align}

\textbf{4. Upper bound \jj{on} dynamic regret. }

\textbf{4.1. Upper bound \jj{on} dynamic regret - without entropy regularization. }

\jj{For without} entropy-regularized case, combining the equations \eqref{eqn:up1_final}, \eqref{eqn:up2_withoutregularzation} \jj{and} \eqref{eqn:up3_final} \jj{leads to the following} upper bound \jj{on the} dynamic regret \jj{for a} future policy $\{ \widehat{\pi}\}$ that holds with probability \jj{at least} $1-p/2$\jj{:}
\begin{align*}
    &\mathfrak{R} \left( \{ \widehat{\pi}^{(k+1)}\}_{1:K-1},K) \right) \nonumber \\
    &= \circled{1} +\circled{2} +\circled{3} \\ 
    &\leq (K-1)\cdot \frac{\gamma^{H}}{1-\gamma}(r_{\text{max}} + \hat{r}_{\text{max}}) + \frac{1}{1-\gamma}  \bar{\iota}_{\infty}^{K} \\
    & + (K-1) \left( \frac{1}{(1-\gamma)^2 \Delta_\pi}+ \frac{\log{|\mathcal{A}|}}{\eta \Delta_\pi}  + \frac{2\gamma^{H} \widehat{r}_{\text{max}}}{1-\gamma}  \right) \\
    & + \frac{4( r_{\text{max}} \vee \hat{r}_{\text{max}})}{1-\gamma} \sqrt{(K-1)H\log(4/p)} - \iota^{K}_H 
\end{align*}

Taking \jj{an} upper bound \jj{on} $r_{\text{max}}$ \jj{and} $\hat{r}_{\text{max}}$ using $(r_{\text{max}} \vee \hat{r}_{\text{max}})$ yields the following upper bound that holds with probability \jj{at least} $1-p/2$\jj{:} 
\begin{align*}
    &\mathfrak{R} \left( \{ \widehat{\pi}^{(k+1)}\}_{1:K-1},K) \right) \nonumber \\
    &\leq (K-1) \bigg( \frac{1}{(1-\gamma)^2 \Delta_\pi}+ \frac{\log{|\mathcal{A}|}}{\eta \Delta_\pi}  +  \frac{4\gamma^{H} (\widehat{r}_{\text{max}} \vee r_{\text{max}})}{1-\gamma} \\ 
    &+ \frac{4( r_{\text{max}} \vee \hat{r}_{\text{max}})}{1-\gamma} \sqrt{\frac{H\log(4/p)}{K-1}}\bigg) + \frac{1}{1-\gamma} \bar{\iota}_{\infty}^{K} - \iota^{K}_H 
\end{align*}

\textbf{4.2. Upper bound \jj{on} dynamic regret - with entropy regularization. }

\jj{For} the entropy-regularized case, combining the equations \eqref{eqn:up1_final}, \eqref{eqn:up2_final}, \eqref{eqn:up3_final} \jj{leads to the following} upper bound \jj{on the} dynamic regret \jj{for a} future policy $\{ \widehat{\pi}\}$ that holds with probability at least $1-p/2$:

\begin{align*}
    &\mathfrak{R} \left( \{ \widehat{\pi}^{(k+1)}\}_{1:K-1},K) \right) \nonumber \\
    &= \circled{1} +\circled{2} +\circled{3} \\ 
    &\leq (K-1)\cdot \frac{\gamma^{H}}{1-\gamma}(r_{\text{max}} + \hat{r}_{\text{max}}) + \frac{1}{1-\gamma}  \bar{\iota}_{\infty}^{K} \\
    & + (K-1) \left( \left( \gamma + 2 \right) \left[ (1-\eta\tau)^{\Delta_\pi-1} C_1 + C_2\right] +  \frac{2\gamma^{H} \widehat{r}_{\text{max}}}{1-\gamma} + \frac{2\tau \log \vert \mathcal{A} \vert}{1-\gamma} \right) \\
    & + \frac{4( r_{\text{max}} \vee \hat{r}_{\text{max}})}{1-\gamma} \sqrt{(K-1)H\log(4/p)} - \iota^{K}_H 
\end{align*}

\hyunin{Then, the following holds} with probability at least $1-p/2$: 
\begin{align*}
    &\mathfrak{R} \left( \{ \widehat{\pi}^{(k+1)}\}_{1:K-1},K) \right) \nonumber \\
    &\leq (K-1) \bigg( \left( \gamma + 2 \right) \left[ (1-\eta\tau)^{\Delta_\pi-1} C_1 + C_2\right] +  \frac{4\gamma^{H} (\widehat{r}_{\text{max}} \vee r_{\text{max}})}{1-\gamma} + \frac{2\tau \log \vert \mathcal{A} \vert}{1-\gamma} \\ 
    &+ \frac{4( r_{\text{max}} \vee \hat{r}_{\text{max}})}{1-\gamma} \sqrt{\frac{H\log(4/p)}{K-1}}\bigg) + \frac{1}{1-\gamma} \bar{\iota}_{\infty}^{K} - \iota^{K}_H 
\end{align*}

\textbf{4.3. Upper bound of Theorem \ref{theorem1}.}

Then, combining \textbf{4.1}, \textbf{4.2} provides the expression, 
 \begin{equation*}
     \mathfrak{R} \left( \{ \widehat{\pi}^{(k+1)}\}_{1:K-1},K) \right) \leq \mathfrak{R}_{I} + \mathfrak{R}_{II}
 \end{equation*}
\jj{where $\mathfrak{R}_{II}=\mathfrak{R}_{\texttt{Alg}}$ if we use $\texttt{Alg}$ as the baseline algorithm and $\mathfrak{R}_{II}=\texttt{Alg}_\tau$ if we use $\mathfrak{R}_{\texttt{Alg}_{\tau}}$ as the baseline algorithm:}
    \begin{align*}
        &\mathfrak{R}_{I} = \frac{1}{1-\gamma} \bar{\iota}_{\infty}^{K} - \iota^{(k)}_H + C_p \sqrt{K-1} \\
        &\mathfrak{R}_ {\texttt{Alg}} = C_{\texttt{Alg}}(\Delta_\pi) \cdot (K-1)  \\ 
        &\mathfrak{R}_{\texttt{Alg}_{\tau}} = C_{\texttt{Alg}_{\tau}}(\Delta_\pi) \cdot (K-1)
    \end{align*}
\jj{where the corresponding} constants are 
    \begin{align*}
        &C_p = \frac{4( r_{\text{max}} \vee \hat{r}_{\text{max}})}{1-\gamma} \sqrt{H\log(4/p)},~~ C_{\texttt{Alg}}(\Delta_\pi) = \left(\frac{1}{(1-\gamma)^2}+ \frac{\log{|\mathcal{A}|}}{\eta}\right)\cdot \frac{1}{\Delta_\pi}  +  \frac{4\gamma^{H} (\widehat{r}_{\text{max}} \vee r_{\text{max}})}{1-\gamma} \\
        &C_{\texttt{Alg}_{\tau}} (\Delta_\pi) = \left( \gamma + 2 \right) \left[ (1-\eta\tau)^{\Delta_\pi-1} C_1 + C_2\right] +  \frac{4\gamma^{H} (\widehat{r}_{\text{max}} \vee r_{\text{max}})}{1-\gamma} + \frac{2\tau \log \vert \mathcal{A} \vert}{1-\gamma} 
    \end{align*}

\end{proof} 

\begin{lemma}[\textbf{Conditions on $\Delta_\pi$ \jj{and} $H$ to guarantee the optimal threshold $2\epsilon$ of \circled{2} without entropy regularization}] \jj{We decompose the term \circled{2} as} 
\begin{align*}
     \circled{2}\text{'s}~(k)^{th}~\text{term} =   \underbrace{\frac{1}{(1-\gamma)^2 \Delta_\pi}+ \frac{\log{|\mathcal{A}|}}{\eta \Delta_\pi}}_{\circled{2}-\circled{a}~\leq~\epsilon}  + \underbrace{\frac{2\gamma^{H} \widehat{r}_{\text{max}}}{1-\gamma}}_{\circled{2}-\circled{b}~\leq~\epsilon} 
\end{align*}
\jj{To guarantee that the terms $\circled{2}-\circled{a}$ and $\circled{2}-\circled{b}$ are each less than or equal to $\epsilon$, it suffices to satisfy the following conditions for $\tau, \eta, \Delta_\pi$ and $H$:}
\begin{align*}
    \circled{2} - \circled{a} &:  \Delta_\pi \geq \left( \frac{1}{(1-\gamma)^2}+ \frac{\log{|\mathcal{A}|}} {\eta} \right) \cdot \frac{1}{\epsilon} \\
    \circled{2} -\circled{b} &: H \geq \frac{\log (\frac{1-\gamma}{2\widehat{r}_{\text{max}}}\epsilon)}{\log (\gamma)} \quad \text{or} \quad H \geq \frac{1}{1-\gamma} \log \left( \frac{2\widehat{r}_{\text{max}}}{(1-\gamma)\epsilon} \right)  
\end{align*}
\label{lemma:optHyp_without_entropy}
\end{lemma}

\begin{lemma}[\textbf{Conditions on $\tau, \Delta_\pi, H$ to guarantee the optimal threshold $4\epsilon$ of \circled{2} with entropy regularization}] \jj{We decompose the term \circled{2} as} 
\begin{align*}
     \circled{2}\text{'s}~(k)^{th}~\text{term} =  \underbrace{\left( \gamma +2 \right) \left[ (1-\eta\tau)^{\Delta_\pi-1} C_1 \right]}_{\circled{2}-\circled{a}~\leq~\epsilon} +  \underbrace{\left( \gamma + 2 \right) C_2}_{\circled{2}-\circled{b}~\leq~\epsilon } +\underbrace{ \frac{2\gamma^{H} \widehat{r}_{\text{max}}}{1-\gamma}}_{\circled{2}-\circled{c}~\leq~\epsilon } + \underbrace{\frac{2\tau \log \vert \mathcal{A} \vert}{1-\gamma}}_{\circled{2}-\circled{d}~\leq ~\epsilon }
\end{align*}
\jj{To guarantee that the terms $\circled{2}-\circled{b},\circled{2}-\circled{c}$ and $\circled{2}-\circled{d}$ are each less than or equal to $\epsilon$, it suffices to satisfy the following conditions for $\tau, \eta, \Delta_\pi$ and $H$:}

\begin{align}
    \circled{2} - \circled{b} &: \delta \leq \frac{\epsilon}{(\gamma+2) \cdot \frac{2}{1-\gamma} \cdot (1+\frac{\gamma}{\eta \tau})} \label{eqn:DeltaBound}\\ 
    \circled{2} -\circled{c} &: H \geq \frac{\log (\frac{1-\gamma}{2\widehat{r}_{\text{max}}}\epsilon)}{\log (\gamma)} \quad \text{or} \quad H \geq \frac{1}{1-\gamma} \log \left( \frac{2\widehat{r}_{\text{max}}}{(1-\gamma)\epsilon} \right) \label{eqn:Hbound}\\ 
    \circled{2} - \circled{d} &: \tau \leq \frac{1-\gamma}{2 \log |\mathcal{A}|}\epsilon \label{eqn:TauBound}
\end{align}

and the term $\circled{2}- \circled{a}$ offers the lower bound of iteration $\Delta_\pi$ as follows.
\begin{align}
    \circled{2}-\circled{a} &: \Delta_\pi \geq \frac{\log \left( \frac{\epsilon}{C_1(\gamma+2)} \right)}{\log (1-\eta \tau)} +1 \quad \text{or} \quad \Delta_\pi \geq \frac{1}{\eta \tau} \log \left( \frac{C_1(\gamma+2)}{\epsilon} \right) +1  \label{eqn:Gbound}
\end{align}

The inequalities (\ref{eqn:Hbound}) \jj{and} (\ref{eqn:Gbound}) \jj{results} from applying \jj{the first-order} Taylor series on $\log(\gamma)$ \jj{and} $\log(1-\eta\tau)$ since $\gamma \in (0,1]$ \jj{and} $\eta \in (0,(1-\gamma)/\tau]$. The \jj{inequalities} (\ref{eqn:DeltaBound}) \jj{and} (\ref{eqn:Gbound}) implies that if the learning rate $\eta$ is fixed in the admissible range, then the iteration complexity scales inversely proportional \jj{to} $\tau$, and the upper bound \jj{on} $\delta$, which we will denote it as $\delta_{\text{max}}$, also scales proportional \jj{to} $\tau$.

Now, the \jj{best guaranteed convergence can be} achieved when $\eta^* = (1-\gamma)/\tau$ \jj{(associated with the value of}  $\eta$ that minimizes the equation (\ref{eqn:3}))\jj{, for which conditions} of hyperparameters $\Delta_{\pi,\eta^*}$ and $\delta_{\eta^*}$ are 

\begin{align*}
    \circled{2}-\circled{a} &: \Delta_{\pi,\eta^*} \geq \frac{1}{1-\gamma} \log \left( \frac{||\widehat{Q}^{\bigast,(k+1)}_{\tau} - \widehat{Q}^{\widehat{\pi}^{(0)}}_{\tau}||_{\infty} (\gamma+2)}{\epsilon} \right) +1 \\
    \circled{2}-\circled{b} &: \delta_{\eta^*} \leq \frac{\epsilon (1-\gamma)^2}{2(\gamma+2)}.
\end{align*}

When $\eta^* = (1-\gamma)/\tau$, the iteration complexity is now proportional to the effective horizon $1/(1-\gamma)$ modulo some log factor, \jj{where the} iteration complexity and $\delta_{\text{max}}$ \jj{are} now independent of the choice of the regularization parameter $\tau$\jj{.}

\label{lemma:optHyp_with_entropy}

\end{lemma}

\begin{lemma}[\textbf{Sample complexity to guarantee the optimal threshold $4\epsilon$ of $\circled{2}$ }]
\label{lemma:samplecomplexity}
 \jj{We define} $\delta_{\text{max}}$ as \jj{right-hand side} of the equation (\ref{eqn:DeltaBound})\jj{. If} we have the number of samples per state-action pairs \jj{is at} least the order of 
\begin{equation*}
    \frac{1}{(1-\gamma)^3\delta^2_{\text{max}}}
\end{equation*}
 up to some logarithmic factor, \jj{then} $\delta \leq \delta_{\text{max}}$ holds with high probability and we can guarantee the optimal threshold $4\epsilon$ with high probability for the upper bound of $\circled{2}$\jj{, provided} (\ref{eqn:Hbound}), (\ref{eqn:TauBound}) and (\ref{eqn:Gbound}) hold.
\end{lemma}

\begin{proof}[\textbf{Proof of Theorem \ref{theorem2}}]

\textbf{1. ProST-T $\iota^{(k)}_{H}$ : }
        
    The \textit{empirical} estimated model prediction error $\iota^{(k+1)}_{h}(s^{(k+1)}_h,a^{(k+1)}_h)$ is represented as follows (Definition \eqref{def:estimatedModelPredictionError})\jj{:}
    \begin{align}
        -\iota^{(k+1)}_h(s^{(k+1)}_h,a^{(k+1)}_h) =&  -R_{(k+1)}(s^{(k+1)}_h,a^{(k+1)}_h) - \gamma(P_{(k+1)} \widehat{V}^{\widehat{\pi}^{(k+1)},(k+1)}_{h+1})(s^{(k+1)}_h,a^{(k+1)}_h) \nonumber \\ 
        &+ \widehat{Q}^{\widehat{\pi}^{(k+1)},(k+1)}_{h}(s^{(k+1)}_h,a^{(k+1)}_h) \label{eqn:l0} \\
        =&  -R_{(k+1)}(s^{(k+1)}_h,a^{(k+1)}_h) - \gamma (P_{(k+1)} \widehat{V}^{\widehat{\pi}^{(k+1)},(k+1)}_{h+1})(s^{(k+1)}_h,a^{(k+1)}_h) \nonumber \\ 
        &+ \widehat{R}_{(k+1)}(s^{(k+1)}_h,a^{(k+1)}_h) + \gamma (\widehat{P}_{(k+1)} \widehat{V}^{\widehat{\pi}^{(k+1)},(k+1)}_{h+1})(s^{(k+1)}_h,a^{(k+1)}_h) \label{eqn:l1}\\
        =&  \left( \widehat{R}_{(k+1)} - R_{(k+1)}\right)(s^{(k+1)}_h,a^{(k+1)}_h) \nonumber \\
        &+ \gamma \left(\left( \widehat{P}_{(k+1)} - P_{(k+1)} \right)  \widehat{V}^{\widehat{\pi}^{(k+1)},(k+1)}_{h+1}\right)(s^{(k+1)}_h,a^{(k+1)}_h) \nonumber \\ 
        \leq & \bar{\Delta}^{Bonus,r}_{(k),h} + \gamma  \left\vert \left\vert \left( \widehat{P}_{(k+1)}  - P_{(k+1)} \right)(\cdot~|~s^{(k+1)}_h,a^{(k+1)}_h) \right\vert \right\vert_1 \left\vert \left\vert \widehat{V}^{\widehat{\pi}^{(k+1)},(k+1)}_{h+1} (\cdot)\right\vert \right\vert_{\infty} \nonumber  \\
        \leq & \bar{\Delta}^{Bonus,r}_{(k),h} + \gamma  \bar{\Delta}^{p}_{k,h} \frac{\gamma^{H-h}\hat{r}_{\text{max}}}{1-\gamma} \label{eqn:l2} \\
        \leq & \bar{\Delta}^{r}_{(k),h} + 2\Gamma^{(k)}_w(s^{(k+1)}_h,a^{(k+1)}_h) + \gamma  \bar{\Delta}^{p}_{(k),h} \frac{\gamma^{H-h}\hat{r}_{\text{max}}}{1-\gamma} \label{eqn:l3}
    \end{align}
    The equation \eqref{eqn:l1} holds due to the future \jj{Bellman} equation \eqref{eqn:futurebellman_step}, the equation \eqref{eqn:l2} holds since $ \left\vert \left\vert \widehat{V}^{\widehat{\pi}^{(k+1)},(k+1)}_{h+1} (\cdot)\right\vert \right\vert_{\infty} \leq \sum_{h^\prime = h+1}^{H} \gamma^{h^\prime-(h+1)} \hat{r}_{\text{max}} \leq \gamma^{H-h} \hat{r}_{\text{max}}/(1-\gamma) $, and the equation \eqref{eqn:l3} holds since $\Delta^{Bonus,r}_{(k)}(s,a) \leq \left\vert \left(R_{(k+1)} - \widetilde{R}_{(k+1)} \right) (s,a) \right\vert + |2\Gamma^{(k)}_w(s,a)| = \Delta^{r}_{(k)}(s,a) + 2\Gamma^{(k)}_w(s,a)$ \jj{for all} $(s,a)$. \jj{The} summation \jj{of the} \hh{empirical model} prediction error over all episodes and all steps \jj{can be bounded as}  
    \begin{align}
        -\iota^{K}_H = \sum_{k=1}^{K-1} \sum_{h=0}^{H-1} - \gamma^{h}\iota^{(k+1)}_h(s^{(k+1)}_h,a^{(k+1)}_h) \leq \underbrace{\bar{\Delta}^{r}_{K}}_{\circled{1}} + \underbrace{\sum_{k=1}^{K-1} \sum_{h=0}^{H-1} 2\Gamma^{(k)}_w(s^{(k+1)}_h,a^{(k+1)}_h)}_{\circled{2}} + \frac{\gamma \hat{r}_{\text{max}}}{1-\gamma} \underbrace{\bar{\Delta}^{p}_{K}}_{\circled{3}} \label{eqn:ftmbpo_empiricalModelErrorUpperbound}
    \end{align}
    We use Lemma \ref{lemma2} to bound the term $\circled{1}$, Lemma \ref{lemma3} and \eqref{eqn:gamma_Lambda_relationship} to bound the term $\circled{2}$, and Lemma \ref{lemma5} (or Lemma \ref{lemma4}) to bound the term $\circled{3}$\jj{:}
    \begin{align}
        \circled{1} &\leq  wHB_r(\Delta_\pi) + \lambda r_{\text{max}} \cdot (K-1) \sqrt{\frac{H}{w}} \sqrt{\log{\left( \frac{\lambda+wH}{\lambda} \right)}} \label{eqn:term1} \\
        \circled{2} &\leq 2\beta (K-1)\sqrt{\frac{H}{w}} \sqrt{\log{ \left( \frac{\lambda+wH}{\lambda} \right)} } \label{eqn:term2} \\
        \circled{3} &\leq  \left( |\mathcal{S}| \sqrt{\frac{H^2}{2}\log{\left( \frac{H}{\delta \lambda} \right) }} +\lambda \right) (K-1) \sqrt{\frac{H}{w}} \sqrt{\log{\left( \frac{\lambda+wH}{\lambda} \right)} } +wHB_p(\Delta_\pi) \label{eqn:term3}
    \end{align}
    where the \jj{inequality} \eqref{eqn:term3} holds with \jj{probability at least} $1-\delta$\jj{, where} $\delta \in (0,1)$. Now, \jj{combining} \eqref{eqn:term1}, \eqref{eqn:term2} \jj{and} \eqref{eqn:term3} \jj{that}
    \begin{align} 
        -\iota^{K}_H =& -\sum_{k=1}^{K-1} \sum_{h=0}^{H-1} \iota^{(k+1)}_h(s^{(k+1)}_h,a^{(k+1)}_h)\nonumber  \\
        \leq& \underbrace{\bar{\Delta}^{r}_{K}}_{\circled{1}} + \underbrace{\sum_{k=1}^{K-1} \sum_{h=0}^{H-1} 2\Gamma^{(k)}_w(s^{(k+1)}_h,a^{(k+1)}_h)}_{\circled{2}} + \frac{\gamma \hat{r}_{\text{max}}}{1-\gamma} \underbrace{\bar{\Delta}^{p}_{K}}_{\circled{3}} \nonumber \\
         \leq& wHB_r(\Delta_\pi) + \lambda r_{\text{max}} \cdot (K-1) \sqrt{\frac{H}{w}} \sqrt{\log{\left( \frac{\lambda+wH}{\lambda} \right)}} + 2\beta (K-1)\sqrt{\frac{H}{w}} \sqrt{\log{ \left( \frac{\lambda+wH}{\lambda} \right)} }  \nonumber \\
         &+ \frac{\gamma \hat{r}_{\text{max}}}{1-\gamma} \left( \left( |\mathcal{S}| \sqrt{\frac{H^2}{2}\log{\left( \frac{H}{\delta \lambda} \right) }} +\lambda \right) (K-1) \sqrt{\frac{H}{w}} \sqrt{\log{\left( \frac{\lambda+wH}{\lambda} \right)} } +wHB_p(\Delta_\pi) \right) \nonumber \\
        \leq&  wH \left( B_r(\Delta_\pi) + \frac{\gamma \hat{r}_{\text{max}}}{1-\gamma} B_p(\Delta_\pi) \right) \nonumber \\
        &+ (K-1) \sqrt{H}\left( \lambda r_{\text{max}} + 2\beta + \frac{\gamma \hat{r}_{\text{max}}}{1-\gamma} \left( |\mathcal{S}| \sqrt{\frac{H^2}{2}\log{\left( \frac{H}{\delta \lambda} \right) }} +\lambda \right)  \right) \sqrt{\frac{1}{w}} \sqrt{\log{\left( \frac{\lambda+wH}{\lambda} \right)}} \label{eqn:citedelta}
    \end{align}
    
    \textbf{2. ProST-T $\bar{\iota}^{K}_{\infty}$ : }
    
    Recall that $\bar{\iota}^{K}_{\infty} = \sum_{k=1}^{K-1} \bar{\iota}^{(k+1)}_{\infty}$. For \jj{the} same $\delta$ that we used in the previous proof of [1.\texttt{ProST-T} $\iota^{(k)}_H$] (see equation \eqref{eqn:citedelta}),  $\bar{\iota}^{k}_{\infty}$ can be bounded as follows with probability \jj{at least} $1-\delta$\jj{:}
    \begin{align}
        \bar{\iota}^{(k+1)}_{\infty} =& R_{(k+1)} + \gamma P_{(k+1)}\widehat{V}^{\bigast,(k+1)}_{\infty} - \widehat{Q}^{\bigast,(k+1)}_{\infty} \nonumber \\
        =&  R_{(k+1)} + \gamma P_{(k+1)} \widehat{V}^{\bigast,(k+1)}_{\infty} - \left( \widehat{R}_{(k+1)} + \gamma \widehat{P}_{(k+1)} \widehat{V}^{\bigast,(k+1)}_{\infty}  \right) \label{eqn:them2_2} \\
        =&  R_{(k+1)} + \gamma P_{(k+1)} \widehat{V}^{\bigast,(k+1)}_{\infty} - \left( \widetilde{R}_{(k+1)} + 2\Gamma^{(k)}_w(s,a) + \gamma \widehat{P}_{(k+1)} \widehat{V}^{\bigast,(k+1)}_{\infty}  \right) \label{eqn:them2_3} \\
        =&  R_{(k+1)} + \gamma P_{(k+1)} \widehat{V}^{\bigast,(k+1)}_{\infty} - \left( \widetilde{R}_{(k+1)} + 2\beta (\Lambda^{(k)}_w(s,a))^{1/2} + \gamma \widehat{P}_{(k+1)} \widehat{V}^{\bigast,(k+1)}_{\infty}  \right) \label{eqn:them2_4} \\
        =&  \left( R_{(k+1)} - \widetilde{R}_{(k+1)} \right) - \beta (\Lambda^{(k)}_w(s,a))^{1/2} + \gamma \left( P_{(k+1)} -\widehat{P}_{(k+1)} \right) \widehat{V}^{\bigast,(k+1)}_{\infty}  \nonumber \\
        & - \beta (\Lambda^{(k)}_w(s,a))^{1/2}  \label{eqn:them2_5} \\
        \leq&  | R_{(k+1)} - \widetilde{R}_{(k+1)} | - \beta (\Lambda^{(k)}_w(s,a))^{1/2} + \gamma || P_{(k+1)} -\widehat{P}_{(k+1)} ||_1 ||\widehat{V}^{\bigast,(k+1)}_{\infty} ||_{\infty}- \beta (\Lambda^{(k)}_w(s,a))^{1/2}  \nonumber \\
        \leq& \left( B^{(k-w+1:k)}_{r}(\Delta_\pi) + \lambda \Lambda^{(k)}_{w}(s,a) r_{\text{max}}\right)  - \beta (\Lambda^{(k)}_w(s,a))^{1/2} \label{eqn:them2_7} \\ 
        &+ \gamma \cdot \left( B^{(k-w+1:k)}_{p}(\Delta_\pi) +  (\Lambda^{(k)}_w(s,a))^{1/2} \cdot |\mathcal{S}| \cdot \sqrt{\frac{H^2}{2}\log{\left( \frac{H}{\delta \lambda} \right) }} + \lambda \Lambda^{(k)}_w (s,a) \right) \cdot \frac{\hat{r}_{\text{max}}}{1-\gamma} \nonumber \\
        &- \beta (\Lambda^{(k)}_w(s,a))^{1/2} \label{eqn:them2_8} \\
        \leq& \left( B^{(k-w+1:k)}_r(\Delta_\pi) + \lambda (\Lambda^{(k)}_{w}(s,a))^{1/2} r_{\text{max}}\right) - \beta (\Lambda^{(k)}_w(s,a))^{1/2} \label{eqn:them2_9} \\ 
        &+ \gamma \cdot \left( B^{(k-w+1:k)}_{p}(\Delta_\pi) +  (\Lambda^{(k)}_w(s,a))^{1/2} \cdot |\mathcal{S}| \cdot \sqrt{\frac{H^2}{2}\log{\left( \frac{H}{\delta \lambda} \right) }} + \lambda (\Lambda^{(k)}_w (s,a))^{1/2} \right) \cdot \frac{\hat{r}_{\text{max}}}{1-\gamma} \nonumber \\
        &- \beta (\Lambda^{(k)}_w(s,a))^{1/2} \label{eqn:them2_10} \\
        \leq& B^{(k-w+1:k)}_r(\Delta_\pi) +\gamma B^{(k-w+1:k)}_{p}(\Delta_\pi) \nonumber \\ 
        &+  \underbrace{\left( \lambda r_{\text{max}} - \beta  +\gamma |\mathcal{S}| \cdot \sqrt{\frac{H^2}{2}\log{\left( \frac{H}{\delta \lambda} \right) }} + \frac{\lambda\hat{r}_{\text{max}}}{1-\gamma} -\beta \right)}_{\leq 0 } (\Lambda^{(k)}_w(s,a))^{1/2} \label{eqn:them2_11_1} \\
        \leq& B^{(k-w+1:k)}_r(\Delta_\pi) +\gamma B^{(k-w+1:k)}_{p}(\Delta_\pi) \label{eqn:them2_12_1} 
    \end{align}
    The equation \eqref{eqn:them2_2} holds by the  future \jj{Bellman} equation \eqref{eqn:futurebellmanOptimal_step} when $H=\infty$, the equations \eqref{eqn:them2_3} and \eqref{eqn:them2_4} hold by the definition of $\widehat{R}_{(k+1)}$ together with \eqref{eqn:gamma_Lambda_relationship}. The \jj{inequalities} \eqref{eqn:them2_7} \jj{and} \eqref{eqn:them2_8} hold by Lemma \ref{lemma1}, Lemma \ref{lemma4}, \eqref{eqn:forecasting reward model error} and \eqref{eqn:forecasting transition probability model error}. The \jj{inequalities} \eqref{eqn:them2_9} \jj{and} \eqref{eqn:them2_10} hold since $0 \leq \Lambda^{(k)}_w(s,a) < 1$. Now, the \jj{inequality} \eqref{eqn:them2_12_1} holds \textit{if} the under-brace term of equation \eqref{eqn:them2_11_1} is equal or smaller than zero. That gives us \jj{an} additional \jj{condition on} $\beta$  to obtain the final inequality \eqref{eqn:them2_12_1}\jj{. Since} $\hat{r}_{\text{max}}$ is defined as $\tilde{r}_{\text{max}} + \frac{2\beta}{\sqrt{\lambda}}$ where $\widetilde{r}_{\text{max}}$ is a constant \jj{and} $\hat{r}_{\text{max}}$ is still function of $\beta,\lambda$ (equation \eqref{eqn:r_max_and_r_tilde})\jj{, the condition is}
    \begin{equation*}
        \lambda r_{\text{max}} - \beta  +\gamma |\mathcal{S}| \cdot \sqrt{\frac{H^2}{2}\log{\left( \frac{H}{\delta \lambda} \right) }} + \frac{\lambda}{1-\gamma}\cdot \left( \widetilde{r}_{\text{max}} + \frac{2\beta}{\sqrt{\lambda}} \right) -\beta \leq 0 
    \end{equation*}
    \jj{or equivalently,}
    \begin{equation}
        \beta \geq \left( 2 + \frac{2\sqrt{\lambda}}{1-\gamma}\right)^{-1} \left( \lambda r_{\text{max}}+ \gamma |\mathcal{S}| \cdot \sqrt{\frac{H^2}{2}\log{\left( \frac{H}{\delta \lambda} \right) }} \right) \label{eqn:betacondition}
    \end{equation}
    Since \eqref{eqn:them2_12_1} holds for all $(s,a)$ if $\beta$ satisfies \eqref{eqn:betacondition}, $\sum_{k=1}^{K-1} \bar{\iota}^{K}_{\infty}=||\bar{\iota}^{k}_{\infty}||_{\infty}$ is bounded as 
    \begin{align*}
        \bar{\iota}^{K}_{\infty} \leq \sum_{k=1}^{K-1} \left( B^{(k-w+1:k)}_r(\Delta_\pi) +\gamma B^{(k-w+1:k)}_{p}(\Delta_\pi) \right) \leq w(B_r(\Delta_\pi) +\gamma B_p (\Delta_\pi))
    \end{align*}
    because $\sum_{k=1}^{K-1}  B^{(k-w+1:k)}_{p}(\Delta_\pi) = \sum_{\mathcal{E}=1}^{\floor{\frac{K-1}{w}}} \sum_{k=(\mathcal{E}-1)w}^{\mathcal{E} w } B^{(k-w+1:k)}_{p}(\Delta_\pi) \leq wB_p (\Delta_\pi)$ holds and in the same way $\sum_{k=1}^{K-1}  B^{(k-w+1:k)}_{p}(\Delta_\pi) \leq w B_r (\Delta_\pi)$ holds.

    Then, the model prediction errors $-\iota^{K}_H, \bar{\iota}^{K}_{\infty}$ when utilizing the forecaster $f$ as \texttt{SW-LSE} are 
    \begin{align*}
        -\iota^{K}_H \leq& wH \left( B_r(\Delta_\pi) + \frac{\gamma \hat{r}_{\text{max}}}{1-\gamma} B_p(\Delta_\pi) \right) \\
        &+ (K-1) \sqrt{H}\left( \lambda r_{\text{max}} + 2\beta + \frac{\gamma \hat{r}_{\text{max}}}{1-\gamma} \left( |\mathcal{S}| \sqrt{\frac{H^2}{2}\log{\left( \frac{H}{\delta \lambda} \right) }} +\lambda \right)  \right) \sqrt{\frac{1}{w}} \sqrt{\log{\left( \frac{\lambda+wH}{\lambda} \right)}}, \\
        \bar{\iota}^{K}_{\infty}  \leq& w(B_r(\Delta_\pi) +\gamma B_p (\Delta_\pi))
    \end{align*}
    
    Finally, the \jj{term} $\mathfrak{R}_{I}$ can be bounded \jj{as} 
    \begin{align*}
        \mathfrak{R}_{I} &= \frac{1}{1-\gamma}\bar{\iota}_{\infty}^{K} - \iota^{K}_H + C_p \sqrt{K-1} \\
        & \leq \frac{1}{1-\gamma} \left( w(B_r(\Delta_\pi) +\gamma B_p (\Delta_\pi)) \right) +  wH \left( B_r(\Delta_\pi) + \frac{\gamma \hat{r}_{\text{max}}}{1-\gamma} B_p (\Delta_\pi) \right)  \\
        &+(K-1) \sqrt{H}\left( \lambda r_{\text{max}} + 2\beta + \frac{\gamma \hat{r}_{\text{max}}}{1-\gamma} \left( |\mathcal{S}| \sqrt{\frac{H^2}{2}\log{\left( \frac{H}{\delta \lambda} \right) }} +\lambda \right)  \right) \sqrt{\frac{1}{w}} \sqrt{\log{\left( \frac{\lambda+wH}{\lambda} \right)}} \\
        &+ C_p \sqrt{K-1} \\ 
        & \leq \left( \left(\frac{1}{1-\gamma}+H \right) B_r(\Delta_\pi) + \frac{(1+H\hat{r}_{\text{max}})\gamma}{1-\gamma} B_p(\Delta_\pi) \right) w \\
        &+(K-1) \sqrt{H}\left( \lambda r_{\text{max}} + 2\beta + \frac{\gamma \hat{r}_{\text{max}}}{1-\gamma} \left( |\mathcal{S}| \sqrt{\frac{H^2}{2}\log{\left( \frac{H}{\delta \lambda} \right) }} +\lambda \right)  \right) \sqrt{\frac{1}{w}} \sqrt{\log{\left( \frac{\lambda+wH}{\lambda} \right)}} \\
        &+ C_p \sqrt{K-1}  
    \end{align*}
     Now, \jj{let} $B(\Delta_\pi)$ \jj{be a} conic combination of $B_r(\Delta_\pi)$ \jj{and} $B_p(\Delta_\pi)$ \jj{as} 
    \begin{align}
        B(\Delta_\pi) &= \left(\frac{1}{1-\gamma}+H \right) B_r(\Delta_\pi) + \frac{(1+H\hat{r}_{\text{max}})\gamma}{1-\gamma} B_p(\Delta_\pi) \nonumber \\ 
        &\add{\leq} \left(\frac{1}{1-\gamma}+H \right) \Delta_{\pi}^{\alpha_r} B_r(1) + \frac{(1+H\hat{r}_{\text{max}})\gamma}{1-\gamma} \Delta_{\pi}^{\alpha_p} B_p(1) \nonumber \\ 
        &=  C_{B_r} \Delta_{\pi}^{\alpha_r}+ C_{B_p} \Delta_{\pi}^{\alpha_p} \label{eq:1_a}
    \end{align}
    where $C_{B_r} =  \left(\frac{1}{1-\gamma}+H \right)B_r(1) $ \jj{and} $C_{B_p}=\frac{(1+H\hat{r}_{\text{max}})\gamma}{1-\gamma}B_p(1)$ \jj{are} constants related \jj{to the} total variation budget with reward and transition probability\jj{.} 

    Recall the \jj{definitions} of $B_r(\Delta_\pi)$ \jj{and} $B_p(\Delta_\pi)$\jj{, as well as the inequalities} $B_r(\Delta_\pi) \leq \Delta_{\pi}^{\alpha_r} B_r(1)$ \jj{and} $B_p(\Delta_\pi) \add{~\leq~} \Delta_{\pi}^{\alpha_p} B_p(1)$. We denote $B_p(1)$ \jj{and} $B_r(1)$ as time-elapsing variation \jj{budgets} for one policy iteration. We also let the constant $C_k$ be defined \jj{as}
    
    \begin{align*}
        C_k &=(K-1) \sqrt{H}\left( \lambda r_{\text{max}} + 2\beta + \frac{\gamma \hat{r}_{\text{max}}}{1-\gamma} \left( |\mathcal{S}| \sqrt{\frac{H^2}{2}\log{\left( \frac{H}{\delta \lambda} \right) }} +\lambda \right)  \right).
    \end{align*}

Then, \jj{an} upper bound \jj{on} $\mathfrak{R}_{I}$ \jj{can be obtained as}
\begin{align*}
    \mathfrak{R}_{I}  \leq B(\Delta_\pi) w + C_k  \sqrt{\frac{1}{w} \log{\left( \frac{\lambda+wH}{\lambda} \right)}} +  C_p \sqrt{K-1}.  
\end{align*}
 \end{proof}
 
\begin{proof}[\textbf{Proof of Proposition \ref{proposition2}}]
    Now, \jj{we} set the sliding window length $w$ \jj{that} is adaptive to $\Delta_\pi$ as follows\jj{:}
\begin{align*}
    &\widetilde{w}(\Delta_\pi) = \left( \frac{C_k}{B(\Delta_\pi)} \right)^{2/3}.
\end{align*}
Then,
\begin{align*}
    &B(\Delta_\pi) \widetilde{w}(\Delta_\pi) +  C_k \sqrt{\frac{1}{\widetilde{w}(\Delta_\pi)}} \sqrt{\log{\left( \frac{\lambda+\widetilde{w}(\Delta_\pi) H}{\lambda} \right)}} \\
    =&  C_{k}^{2/3} B(\Delta_\pi)^{1/3} +  C_{k}^{2/3} B(\Delta_\pi)^{1/3}  \sqrt{\log{\left( 1 + \frac{H}{\lambda} \left( \frac{C_k}{B(\Delta_\pi)} \right)^{2/3} \right)}}.
\end{align*}

Since $C_k$ is linear to $K-1$, the function $\mathfrak{R}_{I}$ satisfies that
\begin{equation}
   \mathfrak{R}_{I} = \mathcal{O} \left( B(\Delta_\pi)^{1/3} \left(K-1\right)^{2/3} \cdot \sqrt{\log{\left(  \frac{K-1}{B(\Delta_\pi)} \right)}}\right).
\end{equation}
Now, \jj{by utilizing} \eqref{eq:1_a}, if $B(\Delta_\pi) \add{\leq C_{B_r} \Delta_{\pi}^{\alpha_r}+ C_{B_p} \Delta_{\pi}^{\alpha_p}}  = o(K)$ holds, then $\mathfrak{R}_{I}$ is sublinear to $K$. The corresponding condition is $B_r(1) +  \frac{\hat{r}_{\text{max}}}{1-\gamma} B_p(1)= o(K)$ \jj{with} $\Delta_\pi<K$ since
\begin{align*}
    &C_{B_r} \Delta_{\pi}^{\alpha_r}+ C_{B_p} \Delta_{\pi}^{\alpha_p} = o(K) \\ 
    &\left(C_{B_r} + C_{B_p} \right) \cdot \Delta_{\pi}^{\max \left( \alpha_r,\alpha_p \right)}  = o(K) \\
    &\left( \left( \frac{1}{1-\gamma} + H \right)B_r(1) + \left( \frac{1+H\hat{r}_{\text{max}}}{1-\gamma} +  \right)B_p(1)  \right)  \cdot \Delta_{\pi}^{\max \left( \alpha_r,\alpha_p \right)}  = o(K) \\
    &\left( \frac{1}{1-\gamma} \left( B_r(1) + B_p(1) \right) + H  \left( B_r(1) +  \frac{\hat{r}_{\text{max}}}{1-\gamma} B_p(1) \right)\right) \cdot  \Delta_{\pi}^{\max \left( \alpha_r,\alpha_p \right)}  = o(K).
\end{align*}
This completes the proof.
\end{proof}

\begin{proof}[\textbf{Proof of Theorem \ref{theorem3}}]
    We first prove multiple statements below. We denote the upper bound \jj{on} $\mathfrak{R}_I$ as $\mathfrak{R}^{\text{max}}_I$, and that of $\mathfrak{R}_{II}$ as $\mathfrak{R}^{\text{max}}_{II}$.
    
    \textbf{1. The \jj{upper bound on} $\mathfrak{R}_{II}(\Delta_\pi)$ \jj{(i.e.,} $\mathfrak{R}^{\text{max}}_{II}$) is a non-increasing function, the \jj{upper bound on} $\mathfrak{R}_{I}(\Delta_\pi)$ \jj{(i.e.,} $\mathfrak{R}^{\text{max}}_{I}$) is \jj{a} non-decreasing function , and both are convex in \jj{the} region $\Delta_\pi \in \mathbb{N}_{I} \cap \mathbb{N}_{II}$}
    \begin{align*}
        \frac{\partial \mathfrak{R}^{\text{max}}_{II}(\Delta_\pi)}{\partial \Delta_\pi} &= \frac{\partial}{\partial \Delta_\pi} \left( C_1 (K-1)(\gamma+2) \left[ (1-\eta\tau)^{\Delta_\pi-1} \right] \right) \\
         &= \log{\left(1-\eta \tau \right)} C_1 (K-1)(\gamma+2) \left[ (1-\eta\tau)^{\Delta_\pi-1} \right]
         \leq 0 \\
         \frac{\partial^2 \mathfrak{R}^{\text{max}}_{II}(\Delta_\pi)}{\partial^2 \Delta_\pi} &= \frac{\partial^2}{\partial^2 \Delta_\pi} \left( C_1 (K-1)(\gamma+2) \left[ (1-\eta\tau)^{\Delta_\pi-1} \right] \right) \\
         &= \left( \log{\left(1-\eta \tau \right)} \right)^2 C_1 (K-1)(\gamma+2) \left[ (1-\eta\tau)^{\Delta_\pi-1} \right]
         \geq 0 
    \end{align*}
    since $\Delta_\pi \in \mathbb{N}_{I} \cap \mathbb{N}_{II}$ satisfies $\Delta_\pi>1$ and $\log (1-\eta\tau) \leq 0$ holds \jj{under} the hyperparameter assumption $0\leq \eta \leq (1-\gamma)/\tau$\jj{, it follows from} the Proposition \ref{proposition1} \jj{that}
        \begin{align*}
        \frac{\partial \mathfrak{R}^{\text{max}}_{I}(\Delta_\pi)}{\partial \Delta_\pi} &= \frac{\partial}{\partial \Delta_\pi} \left( C_{B_r}\Delta_{\pi}^{\alpha_r}+C_{B_p} \Delta_{\pi}^{\alpha_p} \right) \\
         &= \alpha_r C_{B_r} \Delta_{\pi}^{\alpha_r-1} + \alpha_p C_{B_p} \Delta_{\pi}^{\alpha_p-1} 
         \geq 0 \\
        \frac{\partial^2 \mathfrak{R}^{\text{max}}_{I}(\Delta_\pi)}{\partial^2 \Delta_\pi} &= \frac{\partial^2}{\partial^2 \Delta_\pi} \left( C_{B_r}\Delta_{\pi}^{\alpha_r}+C_{B_p} \Delta_{\pi}^{\alpha_p} \right) \\
         &= \alpha_r (\alpha_r-1) C_{B_r} \Delta_{\pi}^{\alpha_r-2} + \alpha_p (\alpha_p-1) C_{B_p} \Delta_{\pi}^{\alpha_p-2} 
         \geq 0 
    \end{align*}
    \jj{when} $\alpha_r,\alpha_p \geq 1$.
    
    \textbf{2. \add{Suboptimal} $\Delta^*_\pi$}
    
    \jj{We} slightly relax the \jj{upper bound} $\mathfrak{R}_{I}(\Delta_\pi) \add{~\leq~} C_{B_r}\Delta_{\pi}^{\alpha_r}+C_{B_p} \Delta_{\pi}^{\alpha_p}$ to $\mathfrak{R}_{I}(\Delta_\pi) = \left( C_{B_r} + C_{B_p} \right) \Delta_{\pi}^{\max{\left(\alpha_r, \alpha_p\right)}}$ and \jj{obtain}  $\Delta_{\pi}^*$ \jj{in} the worst case by optimizing $\mathfrak{R}^{\text{max}}_{I}(\Delta_\pi)+ \mathfrak{R}^{\text{max}}_{II}(\Delta_\pi)$.
    \begin{enumerate}
        \item $\boldsymbol{\max{(\alpha_r,\alpha_p)}}=0$ :
        this means \jj{that} $\mathfrak{R}^{\text{max}}_{I}(\Delta_\pi) = C_{B_r} + C_{B_p}$\jj{, where} $\mathfrak{R}^{\text{max}}_{I}$ is now independent \jj{of} $\Delta_\pi$. Then, \jj{an infinite number} $\Delta_\pi$ guarantees \jj{a} small dynamic regret $\mathfrak{R}_{I}$, \jj{which} also leads to \jj{a} small $\mathfrak{R}$. \jj{It can be checked that} $\mathfrak{R}_{II}$ without entropy regularization decreases \jj{with the scale of} $1/\Delta_\pi$, and $\mathfrak{R}_{II}$ with entropy regularization decreases \jj{with the scale of} $\exp{(\Delta_\pi)}$. This also matches with the \jj{existing results on achieving a faster convergence with an} entropy regularization\jj{.}      
    \end{enumerate}
    
    For the remaining case, \jj{we} first compute the gradient of the term $\mathfrak{R}^{\text{max}}_{I}(\Delta_\pi) + \mathfrak{R}^{\text{max}}_{II}(\Delta_\pi)$ \jj{when}  $\mathfrak{R}^{\text{max}}_{II}(\Delta_\pi)$ comes from entropy-regularized case\jj{:}
    \begin{align*}
        &\frac{\partial \left( \mathfrak{R}^{\text{max}}_{I}(\Delta_\pi) + \mathfrak{R}^{\text{max}}_{II}(\Delta_\pi) \right)}{\partial \Delta_\pi} \\
        &= \max{(\alpha_r,\alpha_p)} \left( \alpha_r C_{B_r}  + \alpha_p C_{B_p} \right) \Delta_{\pi}^{\max{(\alpha_r,\alpha_p)}-1} - \log{\left(\frac{1}{1-\eta \tau}\right)} C_1 (K-1)(\gamma+2) \left[ (1-\eta\tau)^{\Delta_\pi-1} \right] \\ 
        &= k_I \Delta_{\pi}^{\max{(\alpha_r,\alpha_p)}-1} - k_{II} \left[ (1-\eta\tau)^{\Delta_\pi-1} \right]
    \end{align*}
    \jj{when} $\mathfrak{R}^{\text{max}}_{II}(\Delta_\pi)$ \jj{is for the case without entropy regularization}, the gradient of \jj{the} dynamic regret upper bound is given as
    \begin{align*}
        &\frac{\partial \left( \mathfrak{R}^{\text{max}}_{I}(\Delta_\pi) + \mathfrak{R}^{\text{max}}_{II}(\Delta_\pi) \right)}{\partial \Delta_\pi} \\
        &=  \max{(\alpha_r,\alpha_p)} \left( \alpha_r C_{B_r}  + \alpha_p C_{B_p} \right) \Delta_{\pi}^{\max{(\alpha_r,\alpha_p)}-1} - \left(\frac{1}{(1-\gamma)^2}+ \frac{\log{|\mathcal{A}|}}{\eta}\right)\cdot \frac{1}{\Delta_{\pi}^2} \\
        &= k_{I} \Delta_{\pi}^{\max{(\alpha_r,\alpha_p)}-1} - k_{II} \frac{1}{\Delta_{\pi}^2} 
    \end{align*}
    \begin{enumerate}
    \setcounter{enumi}{1}
        \item $\boldsymbol{\max{(\alpha_r,\alpha_p)} = 1}$: \jj{The relation} $(1-\eta\tau)^{\Delta_\pi-1} =k_{I}/k_{II}$ should be satisfied for \jj{the} entropy regularized case and $\Delta_{\pi}^{-2} = k_{I}/k_{II}$ should be satisfied \jj{in the case without entropy regularization}, respectively. Then, \jj{it holds that} $\Delta_{\pi}^* = \log_{1-\eta\tau}(k_{I}/k_{II})+1$ for \jj{the} entropy regularized case and $\Delta_{\pi}^* = \sqrt{k_{II}/k_{I}}$ \jj{without regularization}. 
    \end{enumerate}
    
    Now, for the case of \jj{the} entropy regularized case, if $k_{II} = (1-\eta\tau)k_I$ is satisfied, $\partial \left( \mathfrak{R}^{\text{max}}_{I}(\Delta_\pi) + \mathfrak{R}^{\text{max}}_{II}(\Delta_\pi) \right) / \partial \Delta_\pi =0$ is equal to solving $\Delta_{\pi}^{\max{(\alpha_r,\alpha_p)}-1} = (1-\eta\tau)^{\Delta_\pi}$. Now, we use the Lambert W function to \jj{find} $\Delta_\pi$ as follows\jj{:}
    \begin{align*}
        &\Delta_{\pi}^{\max{(\alpha_r,\alpha_p)}-1} = (1-\eta\tau)^{\Delta_\pi} \\ 
        &\left(\max{(\alpha_r,\alpha_p)}-1 \right) \log{\Delta_\pi} = \Delta_\pi \log{(1-\eta\tau)} \\
        &\Delta_{\pi}^{-1} \cdot \log{\Delta_\pi} = \frac{ \log{(1-\eta\tau)}}{\max{(\alpha_r,\alpha_p)}-1} \\ 
        & - \log{\Delta_\pi} \cdot e^{ -\log{\Delta_\pi}} = - \frac{ \log{(1-\eta\tau)}}{\max{(\alpha_r,\alpha_p)}-1} \\ 
        & W \left[ - \log{\Delta_\pi} \cdot e^{ -\log{\Delta_\pi}} \right] = W \left[ - \frac{ \log{(1-\eta\tau)}}{\max{(\alpha_r,\alpha_p)}-1} \right] \\ 
        & W \left[ - \log{\Delta_\pi} \cdot e^{ -\log{G}} \right] = W \left[-  \frac{ \log{(1-\eta\tau)}}{\max{(\alpha_r,\alpha_p)}-1} \right] \\
        & - \log{\Delta_\pi}  = W \left[-  \frac{ \log{(1-\eta\tau)}}{\max{(\alpha_r,\alpha_p)}-1} \right] \\
         & \Delta_{\pi}^*  = \exp{\left(-W \left[-  \frac{ \log{(1-\eta\tau)}}{\max{(\alpha_r,\alpha_p)}-1} \right]\right)} =\exp{\left(-W \left[x \right]\right)} \\
    \end{align*}

    \begin{enumerate}
    \setcounter{enumi}{2}
        \item $\boldsymbol{0 < \max{(\alpha_r,\alpha_p)} < 1}$ : 
        \begin{itemize}
            \item Without Entropy-regularization: $\Delta_{\pi}^{*} = (k_{I}/k_{II})^{1/(\max{(\alpha_r,\alpha_p)}+1)}$
            \item With Entropy-regularization: \jj{Since} $x = -  \frac{ \log{(1-\eta\tau)}}{\max{(\alpha_r,\alpha_p)}-1} <0 $, \jj{a} small $|x|$ \jj{will have a} large $-W(x) >0$ value, which leads to \jj{a large} $\Delta_{\pi}^{*}$. 
        \end{itemize}

        \item $\boldsymbol{ \max{(\alpha_r,\alpha_p)} > 1}$ : 
        \begin{itemize}
            \item  Without Entropy-regularization: $\Delta_{\pi}^* = (k_{I}/k_{II})^{1/(\max{(\alpha_r,\alpha_p)}+1)}$
            \item With Entropy-regularization: \jj{It holds that} $x >0$ and $-W(x) <0$. Then $\Delta_{\pi}^* < 1 $\jj{, which} means \jj{that one} iteration is enough. 
        \end{itemize}
    \end{enumerate}
\end{proof}

From the proof of \jj{Theorem}  \ref{theorem2}, we will develop Lemma \ref{corollary1}, Lemma \ref{corollary2} \jj{and} Lemma \ref{corollary3} \jj{to upper-bound} two model prediction errors $-\iota^{(k)}_h$ \jj{and} $\bar{\iota}_{\infty}^k$\jj{.}

\begin{lemma}[Upper bound \jj{on} $-\iota^{(k+1)}_h(s^{(k+1)}_h,a^{(k+1)}_h)$ by $\bar{\Delta}^{r}_{k,h},~\bar{\Delta}^{p}_{k,h}$]
    \label{corollary1}
    It holds that
    \begin{equation*}
         -\iota^{(k+1)}_h(s^{(k+1)}_h,a^{(k+1)}_h) \leq \bar{\Delta}^{r}_{k,h} + 2\Gamma^{(k)}_w(s,a) + \gamma  \bar{\Delta}^{p}_{k,h} \frac{\gamma^{H-h}\hat{r}_{\text{max}}}{1-\gamma}
    \end{equation*}
\end{lemma}
\begin{proof}[\textbf{Proof of Lemma \ref{corollary1}}]
    \jj{It follows from} \eqref{eqn:l0}, \eqref{eqn:l1}, \eqref{eqn:l2} and \eqref{eqn:l3}.
\end{proof}

\begin{lemma}[Upper bound \jj{on} $-\iota^{(k+1)}_h(s,a)$ by $\Delta^{r}_{(k)},~\Delta^{p}_{(k)}$]
For \jj{every} $(s,a) \in \mathcal{S} \times \mathcal{A}$, \jj{it holds that}
    \label{corollary2}
    \begin{equation*}
         -\iota^{(k+1)}_h(s,a)   
        \leq    \Delta^{r}_{(k)}(s,a) + \gamma \Delta^{p}_{(k)}(s,a)\frac{\gamma^{H-h}\hat{r}_{\text{max}}}{1-\gamma} +  2\Gamma^{(k)}_w(s,a)    
    \end{equation*}
\end{lemma}
\begin{proof}[\textbf{Proof of Lemma \ref{corollary2}}]
    \begin{align*}
        -\iota^{(k+1)}_h(s,a) =&  -R_{(k+1)}(s,a) - \gamma(P_{(k+1)} \widehat{V}^{\widehat{\pi}^{(k+1)},(k+1)}_{h+1})(s,a) + \widehat{Q}^{\widehat{\pi}^{(k+1)},(k+1)}_{h}(s,a) \\
        =&  -R_{(k+1)}(s,a) - \gamma(P_{(k+1)} \widehat{V}^{\widehat{\pi}^{(k+1)},(k+1)}_{h+1})(s,a) \nonumber \\ 
        &+ \widehat{R}_{(k+1)}(s,a) + \gamma(\widehat{P}_{(k+1)} \widehat{V}^{\widehat{\pi}^{(k+1)},(k+1)}_{h+1})(s,a) \\
        =&  \left( \widehat{R}_{(k+1)} - R_{(k+1)}\right)(s,a) + \gamma \left(\left( \widehat{P}_{(k+1)} - P_{(k+1)} \right)  \widehat{V}^{\widehat{\pi}^{(k+1)},(k+1)}_{h+1}\right)(s,a) \\ 
        \leq & \Delta^{r}_{(k)}(s,a) + 2\Gamma^{(k)}_w(s,a) + \gamma  \left\vert \left\vert \left( \widehat{P}_{(k+1)}  - P_{(k+1)} \right)(\cdot~|~s,a) \right\vert \right\vert_1 \left\vert \left\vert \widehat{V}^{\widehat{\pi}^{(k+1)},(k+1)}_{h+1} (\cdot)\right\vert \right\vert_{\infty} \\
        \leq & \Delta^{r}_{(k)}(s,a) + 2\Gamma^{(k)}_w(s,a) +  \gamma \Delta^{p}_{(k)}(s,a)\frac{\gamma^{H-h}\hat{r}_{\text{max}}}{1-\gamma} 
    \end{align*}
\end{proof}

\begin{lemma}[Upper bound \jj{on} $\bar{\iota}^{k}_{\infty}$ by $\Delta^{r}_{(k)},~\Delta^{p}_{(k)}$]
For \jj{every} $(s,a) \in \mathcal{S}\times \mathcal{A}$\jj{, it holds that}
    \label{corollary3}
    \begin{equation*}
         \bar{\iota}^{k+1}_{\infty}(s,a)   
        \leq   \Delta^{r}_{(k)}(s,a) +  \Delta^{p}_{(k)}(s,a) \frac{\gamma \hat{r}_{max}}{1-\gamma} -  2\Gamma^{(k)}_w(s,a) 
    \end{equation*}
\end{lemma}
\begin{proof}[\textbf{Proof of Lemma \ref{corollary3}}]
    \jj{It results from} \eqref{eqn:them2_5},
    \begin{align*}
          \bar{\iota}^{k+1}_{\infty} &=  
          \left( R_{(k+1)} - \widetilde{R}_{(k+1)} \right) - \beta (\Lambda^{(k)}_w(s,a))^{1/2} + \gamma \left( P_{(k+1)} -\widehat{P}_{(k+1)} \right) \widehat{V}^{\bigast,(k+1)}_{\infty} - \beta (\Lambda^{(k)}_w(s,a))^{1/2} \\
            & \leq 
          \left| R_{(k+1)} - \widetilde{R}_{(k+1)} \right| - \beta (\Lambda^{(k)}_w(s,a))^{1/2} + \gamma \left|\left| P_{(k+1)} -\widehat{P}_{(k+1)} \right|\right|_1  \left|\left| \widehat{V}^{\bigast,(k+1)}_{\infty}\right|\right|_{\infty} - \beta (\Lambda^{(k)}_w(s,a))^{1/2} \\ 
            & \leq 
          \Delta^{r}_{(k)}(s,a) - \beta (\Lambda^{(k)}_w(s,a))^{1/2} + \gamma \Delta^{p}_{(k)}(s,a) \frac{\hat{r}_{max}}{1-\gamma} - \beta (\Lambda^{(k)}_w(s,a))^{1/2} \\
            & =
          \Delta^{r}_{(k)}(s,a) +  \Delta^{p}_{(k)}(s,a) \frac{\gamma \hat{r}_{max}}{1-\gamma} -  2\Gamma^{(k)}_w(s,a)
    \end{align*}
\end{proof}

\begin{lemma}[Upper bound \jj{on} $\Delta^{r}_{(k)}(s,a)$]
\label{lemma1}
For \jj{every} $(s,a) \in \mathcal{S} \times \mathcal{A}$, \jj{it holds that}
\begin{equation*}
    \Delta^{r}_{(k)}(s,a) \leq B^{(k-w:k)}_{r}(\Delta_\pi) + \lambda \Lambda^{(k)}_{w}(s,a) r_{\text{max}}
\end{equation*}
\end{lemma}
\begin{proof}[\textbf{Proof of Lemma \ref{lemma1}}]
    We directly utilize the proof of \jj{Lemma 35 in} \cite{ding2022provably}. \jj{For every} $(s,a) \in \mathcal{S} \times \mathcal{A}$, $\Delta^{r}_{(k)}(s,a)$ can be represented as
    \begin{align}
        &  \Delta^{r}_{(k)}(s,a) \\
        &= |R_{(k+1)}(s,a) - \widetilde{R}_{(k+1)}(s,a)| \label{eqn:lemma1_eq1} \\
        &= |o^r_{(k+1)}(s,a) - \widetilde{o}^r_{(k+1)}(s,a)| \label{eqn:lemma1_eq2} \\
        &= \left|\frac{\sum_{t=(1 \wedge k-w+1)}^{k} \sum_{h=0}^{H-1} \mathbbm{1} \left[ (s,a)=(s^t_h,a^t_h) \right] \cdot r^t_h}{\lambda + \sum_{t=(1 \wedge k-w+1)}^{k} n_{t}(s,a)} - o^r_{(k+1)}(s,a) \right|  \label{eqn:lemma1_eq3} \\
        &= \Lambda^{(k)}_{w}(s,a) \left| \sum_{t=(1 \wedge k-w+1)}^{k} \sum_{h=0}^{H-1} \mathbbm{1} \left[ (s,a)=(s^t_h,a^t_h) \right] \cdot r^t_h - \left( \lambda + \sum_{t=(1 \wedge k-w+1)}^{k} n_{t}(s,a) \right)o^r_{(k+1)}(s,a) \right|  \label{eqn:lemma1_eq4} \\
        &= \Lambda^{(k)}_{w}(s,a) \left| \sum_{t=(1 \wedge k-w+1)}^{k} \sum_{h=0}^{H-1} \left( \mathbbm{1} \left[ (s,a)=(s^t_h,a^t_h) \right] \left( r^t_h - o^r_{(k+1)}(s,a)\right) \right) -  \lambda \cdot o^r_{(k+1)}(s,a) \right| \label{eqn:lemma1_eq5} \\ 
        & \leq \Lambda^{(k)}_{w}(s,a) \left( \sum_{t=(1 \wedge k-w+1)}^{k} \sum_{h=0}^{H-1} \mathbbm{1} \left[ (s,a)=(s^t_h,a^t_h) \right] \cdot \left| r^t_h - o^r_{(k+1)}(s,a) \right| \right)+ \lambda \Lambda^{(k)}_w(s,a) \left| o^r_{(k+1)}(s,a) \right| \label{eqn:lemma1_eq6} \\ 
        & \leq \Lambda^{(k)}_{w}(s,a) \left( \sum_{t=(1 \wedge k-w+1)}^{k} n_t(s,a) \left(  \left| r^t(s,a) - o^r_{(k+1)}(s,a)\right| \right) \right) + \lambda \Lambda^{(k)}_{w}(s,a) r_{\text{max}} \label{eqn:lemma1_eq7} \\
        &\leq \max_{(1 \wedge k-w+1) \leq t\leq k}\left(  \left| r^t(s,a) - o^r_{(k+1)}(s,a)\right| \right) \Lambda^{(k)}_{w}(s,a) \left( \sum_{t=(1 \wedge k-w+1)}^{k} n_t(s,a) \right) + \lambda \Lambda^{(k)}_{w}(s,a) r_{\text{max}} \nonumber  \\
        & \leq \max_{(1 \wedge k-w+1) \leq t\leq k}\left(  \left| r^t(s,a) - o^r_{(k+1)}(s,a)\right| \right) +   \lambda \Lambda^{(k)}_{w}(s,a) r_{\text{max}} \nonumber \\ 
        & \leq B^{(k-w:k)}_{r}(\Delta_\pi) + \lambda \Lambda^{(k)}_{w}(s,a) r_{\text{max}}  \label{eqn:lemma1_eq8}
    \end{align}
    Equations \eqref{eqn:lemma1_eq2} and \eqref{eqn:lemma1_eq3} hold by the definition of $o^{r}_{k+1},\widetilde{o}^{r}_{k+1}$ (definition \eqref{eqn:r_p_estimation}), equation \eqref{eqn:lemma1_eq4} holds by the definition \eqref{def:Biggamma}, equation \eqref{eqn:lemma1_eq5} holds since $n_t(s,a) \coloneqq \sum_{h=0}^{H-1} \mathbbm{1} \left[ (s,a) = (s_h^t,a_h^t) \right]$, \jj{and inequality} \eqref{eqn:lemma1_eq8} holds since $\max_{(1 \wedge k-w+1) \leq t\leq k} \left(  \left| r^t(s,a) - o^r_{(k+1)}(s,a) \right| \right) \leq |r^{(1 \wedge k-w+1)}(s,a) - r^{(1 \wedge k-w+1)+1}(s,a) | + \dots +|r^{k}(s,a)-r^{k+1}(s,a)| = B^{(k-w:k)}_{r}(\Delta_\pi)$.
\end{proof}

\begin{lemma}[Upper bound \jj{on} $\bar{\Delta}^{r}_{K}$]
    \label{lemma2}
    For \jj{every} $(s,a) \in \mathcal{S} \times \mathcal{A}$, \jj{it holds that}
    \begin{align*}
         \bar{\Delta}^{r}_{K} \leq wHB_r(\Delta_\pi) + \lambda r_{\text{max}} \cdot (K-1) \sqrt{\frac{H}{w}} \sqrt{\log{\left( \frac{\lambda+wH}{\lambda} \right)}}
    \end{align*}
\end{lemma}
\begin{proof}[\textbf{Proof of Lemma \ref{lemma2}}]
    The total empirical forecasting model error up to $K-1$ is given as  
    \begin{align}
        \bar{\Delta}^{r}_{K} &= \sum_{k=1}^{K-1} \sum_{h=0}^{H-1} \bar{\Delta}^{r}_{k,h} \nonumber \\
        &= \sum_{k=1}^{K-1} \sum_{h=0}^{H-1} \Delta^{r}_{(k)}(s^{(k+1)}_{h},a^{(k+1)}_{h}) \nonumber \\
        & \leq \sum_{k=1}^{K-1}\sum_{h=0}^{H-1} \left( B^{(k-w:k)}_{r}(\Delta_\pi) + \lambda \Lambda^{(k)}_{w}(s^{(k+1)}_{h},a^{(k+1)}_{h}) r_{\text{max}}  \right) \label{eqn:lemma2_eq3} \\
        & = wHB_r(\Delta_\pi) + \lambda r_{\text{max}} \cdot \sum_{k=1}^{K-1} \sum_{h=0}^{H-1} \left(  \Lambda^{(k)}_w(s^{(k+1)}_{h},a^{(k+1)}_{h})  \right) \label{eqn:lemma2_eq4} \\
       &\leq  wHB_r(\Delta_\pi) + \lambda r_{\text{max}} \cdot \sum_{k=1}^{K-1} \sum_{h=0}^{H-1} \left( \sqrt{ \Lambda^{(k)}_w(s^{(k+1)}_{h},a^{(k+1)}_{h})}  \right) \nonumber \\
        & \leq wHB_r(\Delta_\pi) + \lambda r_{\text{max}} \cdot (K-1) \sqrt{\frac{H}{w}} \sqrt{\log{\left( \frac{\lambda+wH}{\lambda} \right)}} \label{eqn:lemma2_eq6}
    \end{align}
    The \jj{inequality} \eqref{eqn:lemma2_eq3} holds by \jj{Lemma} \ref{lemma1}, the equation \eqref{eqn:lemma2_eq4} holds since $\sum_{k=1}^{K-1}  B^{(k-w:k)}_{r}(\Delta_\pi)=\sum_{\mathcal{E}=1}^{{\floor{\frac{K-1}{w}}}} \sum_{k=(\mathcal{E}-1)w}^{\mathcal{E} w } B^{(k-w:k)}_{r}(\Delta_\pi) \leq w B_r(\Delta_\pi) $, and the \jj{inequality} \eqref{eqn:lemma2_eq6} holds by \jj{Lemma} \ref{lemma3}.
\end{proof}

\begin{lemma}[Upper bound \jj{on} the term $\sum_{k=1}^{K-1} \sum_{h=0}^{H-1} \sqrt{  \Lambda^{(k)}_w(s^{(k+1)}_{h},a^{(k+1)}_{h})}$]
    \label{lemma3}
    \jj{It holds that}
    \begin{equation*}
        \sum_{k=1}^{K-1} \sum_{h=0}^{H-1} \left(  \sqrt{\Lambda^{(k)}_w(s^{(k+1)}_{h},a^{(k+1)}_{h})}  \right) \leq (K-1)\sqrt{\frac{H}{w}} \sqrt{\log{ \left( \frac{\lambda+wH}{\lambda} \right)} }
    \end{equation*}
\end{lemma}
\begin{proof}[\textbf{Proof of lemma \ref{lemma3}}]
    We denote $\bar{\Lambda}_{w}^{k} = \lambda \mathbb{I} + \sum_{t=(1 \wedge k-w+1)}^{k} \sum_{h=0}^{H-1} \varphi(s_h^t,a_h^t)\varphi(s_h^t,a_h^t)^\top$. Also, we denote $(\bar{\Lambda}_{w}^{k})^{(1)}= \lambda \mathbb{I} +  \varphi(s_h^{(1 \wedge k-w+1)},a_h^{(1 \wedge k-w+1)})\varphi(s_h^{(1 \wedge k-w+1)},a_h^{(1 \wedge k-w+1)})^\top$ Then, for \jj{every} $(s,a) \in \mathcal{S} \times \mathcal{A} $, $\Lambda^{(k)}_w(s,a) = \varphi(s,a)(\bar{\Lambda}^k_w)^{-1}\varphi(s,a)^\top$ holds. Now, the following term can be bounded \jj{as}
    \begin{align}
        &\sum_{k=1}^{K-1} \sum_{h=0}^{H-1} \sqrt{\Lambda^{(k)}_w (s^{(k+1)}_{h},a^{(k+1)}_{h})} \nonumber\\
        &=  \sum_{k=1}^{K-1} \sum_{h=0}^{H-1} \sqrt{\varphi(s^{(k+1)}_{h},a^{(k+1)}_{h}) (\bar{\Lambda}^k_w)^{-1}  \varphi(s^{(k+1)}_{h},a^{(k+1)}_{h})^\top} \nonumber \\   
        &= \sum_{\mathcal{E}=1}^{\floor{\frac{K-1}{w}}} \sum_{k=(\mathcal{E}-1)w+1}^{\mathcal{E} w} \sum_{h=0}^{H-1} \sqrt{\varphi(s^{(k+1)}_{h},a^{(k+1)}_{h}) (\bar{\Lambda}^k_w)^{-1}  \varphi(s^{(k+1)}_{h},a^{(k+1)}_{h})^\top} \nonumber \\
        &\leq \sum_{\mathcal{E}=1}^{\floor{\frac{K-1}{w}}} \sqrt{Hw} \sqrt{\sum_{k=(\mathcal{E}-1)w+1}^{\mathcal{E} w} \sum_{h=0}^{H-1} \varphi(s^{(k+1)}_{h},a^{(k+1)}_{h}) (\bar{\Lambda}^k_w)^{-1}  \varphi(s^{(k+1)}_{h},a^{(k+1)}_{h})^\top}   \label{eqn:lemma3_eq2_1}\\
        & \leq \sum_{\mathcal{E}=1}^{\floor{\frac{K-1}{w}}}  \sqrt{Hw} \sqrt{\log{ \left( \frac{\text{det}\left(\Lambda^{\mathcal{E}w+1}_w \right)}{\text{det} \left( (\Lambda^{(\mathcal{E}-1)w+2}_w)^{(1)} \right) }\right)}} \label{eqn:lemma3_eq3}\\
        &\leq  \floor{\frac{K-1}{w}}\sqrt{Hw} \sqrt{\log{ \left( \frac{\lambda+wH}{\lambda} \right)} }\label{eqn:lemma3_eq5} \\
        &\leq  (K-1)\sqrt{\frac{H}{w}} \sqrt{\log{ \left( \frac{\lambda+wH}{\lambda} \right)} } \nonumber
    \end{align}
\jj{The inequality} \eqref{eqn:lemma3_eq2_1} holds by \jj{the} Cauchy–Schwarz inequality,  \eqref{eqn:lemma3_eq3} holds by \jj{Lemmas} (D.1) and (D.2) \jj{in} \cite{jin2020provably}, \jj{and} \eqref{eqn:lemma3_eq5} holds since $ (\Lambda^{(\mathcal{E}-1)w+2}_w)^{(1)}  \geq \lambda$ \jj{and} $ \Lambda^{\mathcal{E}w+1}_w \leq \lambda + wH$. 
\end{proof}

\begin{lemma}[Upper bound \jj{on} $\Delta^{p}_{(k)}(s,a)$]
    \label{lemma4}
    For \jj{every} $(s,a) \in \mathcal{S} \times \mathcal{A}$ and given $\delta \in (0,1)$, the following holds with probability \jj{at least} $1-\delta$\jj{:}
    \begin{equation*}
        \Delta^{p}_{(k)}(s,a) \leq B^{(k-w+1:k)}_{p} +  (\Lambda^{(k)}_w(s,a))^{1/2} \cdot |\mathcal{S}| \cdot \sqrt{\frac{H^2}{2}\log{\left( \frac{H}{\delta \lambda} \right) }} + \lambda \Lambda^{(k)}_w (s,a)
    \end{equation*}
\end{lemma}
\begin{proof}[\textbf{Proof of lemma \ref{lemma4}}]
        For \jj{every} $(s,a) \in  \mathcal{S} \times \mathcal{A}$, \jj{one can write:}
\begin{align}
    & \Delta^{p}_{(k)}(s,a) \nonumber\\
    & = ||P_{(k+1)}(\cdot|s,a) - \widehat{P}_{(k+1)}(\cdot|s,a)||_1 \nonumber \\
    &= ||o^p_{(k+1)}(\cdot,s,a) - \widehat{o}^p_{(k+1)}(\cdot,s,a)||_1 \nonumber \\
    &= \sum_{s^\prime \in \mathcal{S}}\left|\frac{\sum_{t=k-w+1}^{k} n_t(s^\prime,s,a)}{\lambda + \sum_{t=k-w+1}^{k} n_{t}(s,a)} - o^p_{(k+1)}(s^\prime,s,a) \right|  \nonumber \\
    &= \Lambda^{(k)}_{w}(s,a)  \sum_{s^\prime \in \mathcal{S}} \left| \sum_{t=k-w+1}^{k} n_{t}(s^\prime,s,a) - \left( \lambda + \sum_{t=k-w+1}^{k} n_{t}(s,a) \right)o^p_{(k+1)}(s^\prime,s,a) \right| \nonumber \\
    & \leq  \Lambda^{(k)}_{w}(s,a) \sum_{s^\prime \in \mathcal{S}} \left( \left| \sum_{t=k-w+1}^{k} \left( n_{t}(s^\prime,s,a) -  n_{t}(s,a) o^p_{(k+1)}(s^\prime,s,a) \right) \right| + \left| \lambda o^p_{(k+1)}(s^\prime,s,a) \right| \right) \nonumber \\
    & \leq \Lambda^{(k)}_{w}(s,a)  \sum_{s^\prime \in \mathcal{S}} \left| \sum_{t=k-w+1}^{k} \left( n_{t}(s^\prime,s,a) -  n_{t}(s,a) o^p_{(k+1)}(s^\prime,s,a) \right) \right| + \lambda \Lambda^{(k)}_{w}(s,a) \label{eqn:up2_1}
\end{align}
Recall that $n_t(s^\prime,s,a),~n_t(s,a)$ is defined as
\begin{align}
    n_t(s^\prime,s,a) &= \sum_{h=0}^{H-1} \mathbbm{1}\left[ (s^\prime,s,a) = (s^t_{h+1},s^t_{h},a^t_{h})\right] \nonumber \\ 
    &=  \sum_{h=0}^{H-1} \mathbbm{1}\left[ (s,a) = (s^t_{h},a^t_{h})\right] \cdot \mathbbm{1} \left[ s^\prime = s^t_{h+1}\right] \label{eqn:n_tdef}\\ 
    n_t(s,a) &= \sum_{h=0}^{H-1} \mathbbm{1}\left[ (s,a) = (s^t_{h},a^t_{h})\right] \label{eqn:n_t2def}
\end{align}
where $\mathbbm{1}[\cdot]$ is an indicator function. \jj{Substituting}  \eqref{eqn:n_tdef} \jj{and} \eqref{eqn:n_t2def} into \eqref{eqn:up2_1} \jj{yields that}
\begin{align*}
    & \Lambda^{(k)}_w(s,a) \sum_{s^\prime \in \mathcal{S}} \left| \sum_{t=k-w+1}^{k} \left( n_{t}(s^\prime,s,a) -  n_{t}(s,a) o^p_{(k+1)}(s^\prime,s,a) \right) \right| \\
    & = \Lambda^{(k)}_w(s,a) \sum_{s^\prime \in \mathcal{S}} \left| \sum_{t=k-w+1}^{k} \left( \sum_{h=0}^{H-1} \mathbbm{1}\left[ (s,a) = (s^t_{h},a^t_{h})\right] \cdot \mathbbm{1} \left[ s^\prime = s^t_{h+1}\right] \right.\right.\\
    & \left.\left.-  \sum_{h=0}^{H-1} \mathbbm{1}\left[ (s,a) = (s^t_{h},a^t_{h})\right] \cdot o^p_{(k+1)}(s^\prime,s,a) \right) \right| \\
    & = \Lambda^{(k)}_w(s,a) \sum_{s^\prime \in \mathcal{S}} \left| \sum_{t=k-w+1}^{k} \left( \sum_{h=0}^{H-1} \mathbbm{1}\left[ (s,a) = (s^t_{h},a^t_{h})\right] \left(\mathbbm{1} \left[ s^\prime = s^t_{h+1}\right]- o^p_{(k+1)}(s^\prime,s,a) \right) \right) \right| \\
    & \leq \underbrace{\Lambda^{(k)}_w(s,a) \sum_{s^\prime \in \mathcal{S}} \bigg| \sum_{t=k-w+1}^{k}  \left( \sum_{h=0}^{H-1} \mathbbm{1}\left[ (s,a) = (s^t_{h},a^t_{h})\right] \left(  \mathbbm{1} \left[ s^\prime = s^t_{h+1}\right] - o^p_{t}(s^\prime,s,a) \right) \right) \bigg|}_{\circled{2.1}} \\ 
    & + \underbrace{\Lambda^{(k)}_w(s,a) \sum_{s^\prime \in \mathcal{S}}  \bigg| \sum_{t=k-w+1}^{k}  \left( \sum_{h=0}^{H-1} \mathbbm{1}\left[ (s,a) = (s^t_{h},a^t_{h})\right] \left(  o^p_{t}(s^\prime,s,a) - o^p_{(k+1)}(s^\prime,s,a) \right) \right) \bigg|}_{\circled{2.2}}
\end{align*}

The term \circled{2.1} can be upperbounded by utilizing the \jj{Lemmas (34) and (43)} in \cite{ding2022provably}. \jj{For every} $t \in [K]$ \jj{and} $s^\prime \in \mathcal{S}$, \jj{we define the} random variable $\eta^t(s^\prime) \coloneqq \sum_{h=0}^{H-1} \left( \mathbbm{1} \left[ s^\prime = s^t_{h+1}\right] - o^p_{t}(s^\prime,s^t_h,a^t_h) \right)$. 
\jj{Given} $s^\prime \in \mathcal{S}$\jj{, the} sequence $\{ \eta^\tau(s^\prime) \}_{\tau=1}^{\infty}$ is \jj{a} zero-mean and $H/2$-sub Gaussian random variable.
From the \jj{Lemma 43 in} \cite{ding2022provably}, \jj{we set} $Y = \lambda \mathbb{I}$ \jj{and} $X_t = \sum_{h=0}^{H-1} \mathbbm{1}\left[(s,a) = (s^t_h,a^t_h) \right]$. Then, for \jj{a} given $\delta \in (0,1)$, the following holds with probability \jj{at least}  $1-\delta$ for \jj{all} $(s,a) \in \mathcal{S} \times \mathcal{A}$\jj{:}
\begin{align}
    &\left\vert (\Lambda^{(k)}_w(s,a))^{1/2} \sum_{t=k-w+1}^k \left( \sum_{h=0}^{H-1} \mathbbm{1} \left[ (s,a) = (s^t_h,a^t_h) \right] \cdot \sum_{h=0}^{H-1} \mathbbm{1}[s^\prime=s^t_{h+1}] - o_t^p(s^\prime,s,a) \right) \right\vert \nonumber \nonumber \\
    &\leq \sqrt{\frac{H^2}{2}\log{\left( \frac{(\Lambda^{(k)}_w(s,a))^{-1/2} \cdot \lambda^{-1/2}}{\delta/H} \right)}} \nonumber \\ 
    &=\sqrt{\frac{H^2}{2}\log{\left( \frac{H}{\delta} \cdot \frac{1}{(\Lambda^{(k)}_w(s,a))^{1/2} \cdot \lambda^{1/2}} \right) }}  \nonumber\\
    & \leq \sqrt{\frac{H^2}{2}\log{\left( \frac{H}{\delta} \cdot \frac{1}{\lambda} \right) }} \label{eqn:up2_2}
\end{align}

\jj{As a result, the following inequality} holds with \jj{probability at least} $1-\delta$\jj{:}
\begin{align*}
    \circled{2.1}& \\
    =& (\Lambda^{(k)}_w(s,a))^{1/2} \sum_{s^\prime \in \mathcal{S} }\bigg\vert (\Lambda^{(k)}_w(s,a))^{1/2} \sum_{t=k-w+1}^k \bigg( \sum_{h=0}^{H-1} \mathbbm{1} \left[ (s,a) = (s^t_h,a^t_h) \right] \cdot \sum_{h=0}^{H-1} \mathbbm{1}[s^\prime=s^t_{h+1}] \\
    & - o_t^p(s^\prime,s,a) \bigg) \bigg\vert \\
    \leq& (\Lambda^{(k)}_w(s,a))^{1/2} \cdot |\mathcal{S}| \cdot \sqrt{\frac{H^2}{2}\log{\left( \frac{H}{\delta \lambda} \right) }} 
\end{align*}
The term \circled{2.2} can be bounded \jj{as}
\begin{align}
    \circled{2.2} & \leq \Lambda^{(k)}_w(s,a) \sum_{s^\prime \in \mathcal{S}} \sum_{t=k-w+1}^{k}  \sum_{h=0}^{H-1} \mathbbm{1}\left[ (s,a) = (s^t_{h},a^t_{h})\right] \bigg|  o^p_{t}(s^\prime,s,a) - o^p_{(k+1)}(s^\prime,s,a) \bigg| \nonumber \\ 
     & = \Lambda^{(k)}_w(s,a)  \sum_{t=k-w+1}^{k}  \sum_{h=0}^{H-1} \mathbbm{1}\left[ (s,a) = (s^t_{h},a^t_{h})\right] \sum_{s^\prime \in \mathcal{S}} \bigg|  o^p_{t}(s^\prime,s,a) - o^p_{(k+1)}(s^\prime,s,a) \bigg|  \nonumber\\ 
     & = \Lambda^{(k)}_w(s,a)  \sum_{t=k-w+1}^{k}  \sum_{h=0}^{H-1} \mathbbm{1}\left[ (s,a) = (s^t_{h},a^t_{h})\right] \bigg|\bigg|  o^p_{t}(\cdot,s,a) - o^p_{(k+1)}(\cdot,s,a) \bigg|\bigg|_{1}  \nonumber\\ 
     & \leq \max_{t\in[k-w+1,k]} \left( \bigg|\bigg|  o^p_{t}(\cdot,s,a) - o^p_{(k+1)}(\cdot,s,a) \bigg|\bigg|_{1} \right) \cdot \left( \Lambda^{(k)}_w(s,a)  \sum_{t=k-w+1}^{k}  \sum_{h=0}^{H-1} \mathbbm{1}\left[ (s,a) = (s^t_{h},a^t_{h})\right] \right)  \nonumber\\
     & \leq \max_{t\in[k-w+1,k]} \left( \bigg|\bigg|  o^p_{t}(\cdot,s,a) - o^p_{(k+1)}(\cdot,s,a) \bigg|\bigg|_{1} \right) \cdot 1  \nonumber \\ 
     & \leq B^{(k-w+1:k)}_{p}(\Delta_\pi) \label{eqn:up2_3}
\end{align}
Then, \jj{by combining} \eqref{eqn:up2_1}, \eqref{eqn:up2_2} \jj{and} \eqref{eqn:up2_3}, the term $\Delta^{p}_{(k)}(s,a)$ \jj{can be expressed as}
\begin{align*}
    \Delta^{p}_{(k)}(s,a) \leq B^{(k-w+1:k)}_{p}(\Delta_\pi) +  (\Lambda^{(k)}_w(s,a))^{1/2} \cdot |\mathcal{S}| \cdot \sqrt{\frac{H^2}{2}\log{\left( \frac{H}{\delta \lambda} \right) }} + \lambda \Lambda^{(k)}_w (s,a).
\end{align*}
\end{proof}

\begin{lemma}[Upper bound \jj{on} $\bar{\Delta}^{p}_{K}$]
    \label{lemma5}
    \jj{Given} $\delta \in (0,1)$, the following \jj{inequality holds} with \jj{probability at least} $1-\delta$\jj{:}
    \begin{equation*}
        \bar{\Delta}^{p}_{K} \leq  \left( |\mathcal{S}| \sqrt{\frac{H^2}{2}\log{\left( \frac{H}{\delta \lambda} \right) }} +\lambda \right) (K-1) \sqrt{\frac{H}{w}} \sqrt{\log{\left( \frac{\lambda+wH}{\lambda} \right)} } +wHB_p(\Delta_\pi) 
    \end{equation*}
\end{lemma}
\begin{proof}[\textbf{Proof of lemma \ref{lemma5}}]
    The total empirical forecasting transition probability model error $\bar{\Delta}^{p}_{K}$ can be represented as follows, 
    \begin{align*}
        \bar{\Delta}^{p}_{K} =& \sum_{k=1}^{K-1} \sum_{h=0}^{H-1} \bar{\Delta}^{p}_{k,h} \\
        =& \sum_{k=1}^{K-1} \sum_{h=0}^{H-1} \Delta^{p}_{(k)} (s^{(k+1)}_{h},a^{(k+1)}_{h}) \\
        \leq & \sum_{k=1}^{K-1}\sum_{h=0}^{H-1}  \left( (\Lambda^{(k)}_w(s^{(k+1)}_{h},a^{(k+1)}_{h}))^{1/2} |\mathcal{S}| \sqrt{\frac{H^2}{2}\log{\left( \frac{H}{\delta \lambda} \right) }}  \right) \\
        &+ \sum_{k=1}^{K-1} \sum_{h=0}^{H-1}  \left( \max_{t\in[k-w+1,k]} \bigg|\bigg|  o^p_{t}(\cdot,s^{(k+1)}_{h},a^{(k+1)}_{h}) - o^p_{(k+1)}(\cdot,s^{(k+1)}_{h},a^{(k+1)}_{h}) \bigg|\bigg|_{1}\right)  \\
        &+ \sum_{k=1}^{K-1} \sum_{h=0}^{H-1} \left( \lambda \Lambda^{(k)}_w(s^{(k+1)}_{h},a^{(k+1)}_{h}) \right) \\ 
        \leq & \left( |\mathcal{S}| \sqrt{\frac{H^2}{2}\log{\left( \frac{H}{\delta \lambda} \right)}} + \lambda \right) \sum_{k=1}^{K-1}\sum_{h=0}^{H-1}  \left( (\Lambda^{(k)}_w(s^{(k+1)}_{h},a^{(k+1)}_{h}))^{1/2}   \right) \\
        &+ \sum_{k=1}^{K-1} \sum_{h=0}^{H-1}  \left( \max_{t\in[k-w+1,k]} \bigg|\bigg|  o^p_{t}(\cdot,s^{(k+1)}_{h},a^{(k+1)}_{h}) - o^p_{(k+1)}(\cdot,s^{(k+1)}_{h},a^{(k+1)}_{h}) \bigg|\bigg|_{1}\right)  \\
        \leq& \left( |\mathcal{S}| \sqrt{\frac{H^2}{2}\log{\left( \frac{H}{\delta \lambda} \right) }} +\lambda \right) (K-1) \sqrt{\frac{H}{w}} \sqrt{\log{\left( \frac{\lambda+wH}{\lambda} \right)} } +wHB_p(\Delta_\pi)\label{eqn:ftmbpo_up3}
    \end{align*}
\end{proof}


\begin{proof}[\textbf{Proof of Theorem \ref{thm:3}} ]

\hh{Before introduing the proof, we first go over some details about Theorem \ref{thm:3} in the following paragraph.} 

The \texttt{W-LSE} involves solving \jj{the} following joint optimization problem over $\phi^r_f \in \mathbb{R}^{|\mathcal{S}||\mathcal{A}|},\phi^p_f\in \mathbb{R}^{|\mathcal{S}|^2|\mathcal{A}|}$ \jj{and} $q \in \mathbb{R}^N$ to obtain a minimum upper bound \jj{on} the dynamic regret: 
\begin{gather}
     \min_{\phi^\diamond_f,q} \mathcal{L}\left(\phi^\diamond_f,q~;~\square_{1:N} \right) 
     \text{ where }\mathcal{L} \left( \phi^\diamond_f,q~;~\square_{1:N} \right) = \sum_{t=1}^{N} q_{t}  \left( \widehat{\square}_{\phi^\diamond_f}^{k+1}-\square_{t} \right)^2 + \text{disc}(q) + \frac{1}{wH} \cdot \lambda ||\phi^\diamond_f||_2  
\end{gather}
where $\diamond = r\text{ or }p$. If $\diamond=r$, then $\square=R(s,a)$ and if $\diamond=p$, then $\square=P(s^\prime,s,a)$. \jj{Moreover,} ${\square}_{ \phi^\diamond_f}$ means that $\square$ is parameterized by $ \phi^\diamond_f$, and $\square_{1:N}$ are observed data of $\square$, and the disc$(q) \coloneqq \sup_{f\in\mathcal{F}}\left( \mathbb{E}[f(\widehat{\square}^{k+1}~|~\square_{1:N}] - \sum_{t=1}^{N}q_t \mathbb{E}[\widehat{\square}^{t}|\square_{1:t-1}] \right)$ measures the non-stationarity of the environment. disc$(q)$ could be measured and \jj{upper-bounded} by \jj{the} observed data. For example, if $\diamond=r$ and $\square=R$, then
$\phi^r_{f}$ parameterizes the future reward function $\widehat{R}_{\phi^r_{f}}^{k+1}$, $N$ is the total number of visits of $(s,a)$ up to episode $k$, $R_{1:N}(s,a)$ is the set of reward values $\{R_{1}(s,a),R_{1}(s,a),\dots,R_{N}(s,a) \} $ that the agent has received when visiting $(s,a)$. We demonstrate \jj{a} modified upper bound \jj{on} $\mathfrak{R}_{I}$ when utilizing $\texttt{W-LSE}$. \jj{To do so}, we define \jj{the} forecasting reward model error $\Delta_{r,k}^1(s,a)=\left\vert \left(R_{(k+1)} - \widetilde{R}_{(k+1)} \right) (s,a) \right\vert $ \jj{and the} forecasting transition probability model error as $\Delta^{p}_{(k)}(s,a)=\left\vert \left\vert \left( P_{(k+1)} - \widehat{P}_{(k+1)} \right)(\cdot~|~s,a) \right\vert \right\vert_1$ where $\widetilde{R}_{(k+1)}$ and $\widehat{P}_{(k+1)}$ are predicted reward,transition probability from function $g \circ f$  (Appendix \ref{appendix:Notation for theoretical analysis}).

    We now \jj{brought} the \jj{Theorem} 7 of \cite{kuznetsov2018theory} to offer \jj{an} upper bound \jj{on} the $l_2$-norm of the reward gap between $R_{(k+1)}(s,a)$ and $\widetilde{R}_{(k+1)}(s,a)$ as follows. \jj{To this end, we} denote $X_{k,h} = (s^{(k)}_h,a^{(k)}_h) \in \mathcal{S} \times \mathcal{A},~Y_{k,h} = R_{(k)}(s^{(k)}_h,a^{(k)}_h) \in \mathbb{R}$ and assume \jj{that the} environment \jj{provides the agent with a} noisy reward $\widehat{Y}_{k,h} = Y_{k,h}  + \eta$\hh{, where $\eta$ is sampled from a zero-mean Gaussian.} Define the kernel $\Psi(x) = \varphi(x) \in \mathbb{R}^{|\mathcal{S}| |\mathcal{A}|}$\jj{, where} $ \varphi(x)$ is the one-hot vector that we have defined \jj{in Section} \ref{appendix:Environment_setting}. Now, \jj{we set} $r(x)= c^\top \varphi(x)$ where the vector $c \in \mathbb{R}^{|\mathcal{S}| |\mathcal{A}|}$ is \jj{the} same as the estimated future reward vector $\widetilde{R}_{(k+1)}  \in \mathbb{R}^{|\mathcal{S}| |\mathcal{A}|}$ and $r(x)$ is \jj{the} same as the estimated future reward when $x=(s,a)$, namely $\widetilde{R}_{(k+1)}(s,a)$. Then, for data until episode $k$, i.e., $\mathcal{D}_{data} =  \{ (X_{1,0},\widehat{Y}_{1,0}),(X_{1,1},\widehat{Y}_{1,1}),..,(X_{k,H-1},\widehat{Y}_{k,H-1}) \} $, we denote $\mathcal{D}^{(s,a)}_{data} \coloneqq \{(X_{k,h},\widehat{Y}_{k,h})~|~X_{k,h}=(s,a) \text{ such that }  (X_{k,h},\widehat{Y}_{k,h}) \in \mathcal{D}_{data}\}$. We relabel $\mathcal{D}^{(s,a)}_{data}$ as $\{\left((s,a),\widehat{Y}_1\right),\left((s,a),\widehat{Y}_2\right),...,\left((s,a),\widehat{Y}_N \right)\}$ \jj{such that} $N(s,a) = \sum_{t=1}^{k}n_t(s,a)$ is the total number of visitations of $(s,a)$ until episode $k$ (Definition \eqref{eqn:n_t2def}). We \jj{use the shorthand notation} $N$ as $N(s,a)$, and $\sum_{t=1}^{N} q_t =1 $. \jj{For every} $(s,a) \in \mathcal{S} \times \mathcal{A}$, the following \jj{inequalities hold} with probability at least $1-\delta$ for all functions $r \in \{ x \rightarrow c^\top \Psi(x) : ||c||_2 \leq \Lambda \}$\jj{:}  
    
    \begin{align}
        \mathbb{E}[\left( r(s,a)-\widehat{Y}_{N+1} \right)^2|\mathcal{D}^{(s,a)}_{data}] \leq \sum_{t=1}^{N}   q_{t} \left(r(s,a)-\widehat{Y}_{t} \right)^2 + \text{disc}(q) + \frac{1}{wH} \cdot \lambda ||\bar{r}||_2 \label{eqn:thm3}
    \end{align}
    \jj{Take the expectation} over $\eta$ on both inequailty. 
    \begin{align*}
         \mathbb{E}_{\eta} \left[ \mathbb{E} \left[ \left( r(s,a)-\widehat{Y}_{N+1} \right)^2|\mathcal{D}^{(s,a)}_{data} \right] \right] & \leq \mathbb{E}_{\eta} \left[ \sum_{t=1}^{N}   q_{t} \left(r(s,a)-\widehat{Y}_{t} \right)^2 + \text{disc}(q) + \frac{1}{wH} \cdot \lambda ||\bar{r}||_2 \right], \\
         \mathbb{E} \left[ \left( r(s,a)-\widehat{Y}_{N+1} \right)^2|\mathcal{D}^{(s,a)}_{data} \right] & \leq \sum_{t=1}^{N}  \mathbb{E}_{\eta} \left[ q_{t} \left(r(s,a)-\widehat{Y}_{t} \right)^2 \right] + \text{disc}(q) + \frac{1}{wH} \cdot \lambda ||\bar{r}||_2 .
    \end{align*}
     The \jj{left-hand side} of \eqref{eqn:thm3} can be expressed \jj{as}
    \begin{align}
        \mathbb{E} \left[ (r(s,a) - \widehat{Y}_{N+1} -\eta)^2 \right] 
        &= \mathbb{E}_{\eta} \left[ (r(s,a) - Y_{N+1})^2\right] + \mathbb{E}_{\eta} \left[ \eta^2 \right] \nonumber \\
        & = (r(s,a) - Y_{N+1})^2 + \mathbb{E} \left[ \eta^2 \right]  \label{eq:thm3_3}
    \end{align}
    \jj{Also, the} term $ \sum_{t=1}^{N}  \mathbb{E}_{\eta}\left[ q_{t} \left(r(s,a)-\widehat{Y}_{t} \right)^2\right]$ of the \jj{right-hand side} of equation \eqref{eqn:thm3} \jj{can be written as} 
    \begin{align*}
        \sum_{t=1}^{N}  \mathbb{E}_{\eta}[ q_{t} \left(r(s,a)-\widehat{Y}_{t} \right)^2] &=  \sum_{t=1}^{N}  \mathbb{E}_{\eta}[ q_{t} \left(\left(r(s,a)-Y_{t} \right)^2 + \eta^2 \right)] \\
        &=  \sum_{t=1}^{N}  \mathbb{E}_{\eta}\left[ q_{t} \left(\left(r(s,a)-Y_{t} \right)^2 \right)\right] + \sum_{t=1}^{N} \mathbb{E}_{\eta} \left[  q_t \eta^2 \right] \\
        &=  \sum_{t=1}^{N}  q_{t} \left(\left(r(s,a)-Y_{t} \right)^2 \right) + \mathbb{E}_{\eta} \left[ \eta^2 \right] 
    \end{align*}
    \jj{By eliminating} $\mathbb{E}_{\eta}[\eta^2]$ \jj{from both sides, we obtain that}
    \begin{align}
       (r(s,a) - Y_{N+1})^2  \leq  \sum_{t=1}^{N}  q_{t} \left(\left(r(s,a)-Y_{t} \right)^2 \right) + \text{disc}(q) + \frac{1}{wH} \cdot \lambda ||\bar{r}||_2 \label{eqn:startfromhere}
    \end{align}
   Recall the definition of $r(s,a) = \widetilde{R}_{(k+1)}(s,a)$, $Y_t = R_t(s,a)$. Since $t$ \jj{matches} one of $(k,h)\in [K] \times [H]$ pairs, we can rewrite 
   \begin{align*}
        \sum_{t=1}^{N} q_{t} \left(r(s,a)-\widehat{Y}_{t} \right)^2 &=  \sum_{k^\prime=1}^{k-1} \sum_{h=0}^{H-1} q_{(k^\prime,h)}\left(r(s,a)-Y_{(k,h)}  \right)^2 \\ 
        &=  \sum_{k^\prime=1}^{k-1} \sum_{h=0}^{H-1} q_{(k^\prime,h)}\left(\widetilde{R}_{(k+1)}(s,a)-R^{k^\prime}_h(s,a) \right)^2 \\ 
        &= \sum_{k^\prime=1}^{k-1} \sum_{h=0}^{H-1} q_{(k^\prime,h)}\left(\widetilde{R}_{(k+1)}(s,a)-R^{k^\prime}(s,a) \right)^2
   \end{align*}
    where if $(s,a)$ is not visited at step $h$ of episode $k$, then \jj{the} corresponding $q_{(k^\prime,h)}$ is zero. \jj{As a result,}  
    \begin{align*}
        \Delta^{r}_{(k)}(s,a) \leq&  \sqrt{ \min_{q,\bar{r}} \left(\sum_{k^\prime=1}^{k-1} \sum_{h=0}^{H-1} q_{(k^\prime,h)}\left(\widetilde{R}_{(k+1)}(s,a)-R^{k^\prime}(s,a) \right)^2 + \text{disc}(q) + \frac{1}{wH} \cdot \lambda ||\bar{r}||_2 \right)} \\
        \leq&  \sqrt{ \min_{q,\bar{r}} \left( \left( \max_{1\leq k^\prime \leq k} \left( \widetilde{R}_{(k+1)}(s,a)-R^{k^\prime}(s,a)\right) \right)^2 \left( \sum_{k^\prime=1}^{k-1} \sum_{h=0}^{H-1} q_{(k^\prime,h)} \right) + \text{disc}(q) + \frac{1}{wH} \cdot \lambda ||\bar{r}||_2 \right)} 
    \end{align*}
    \jj{A similar analysis for} $\Delta^{p}_{(k)}$ \jj{leads to the following inequality for all} $(s^\prime,s,a) \in \mathcal{S} \times \mathcal{S} \times \mathcal{A}$\jj{:}
    \begin{align*}
         &\left| P_{(k+1)}(s^\prime~|~s,a) - \widehat{P}_{(k+1)}(s^\prime~|~s,a) \right| \\ &\leq  \sqrt{ \min_{q,\bar{p}} \left(\sum_{k^\prime=1}^{k-1} \sum_{h=0}^{H-1} q_{(k^\prime,h)}\left(\widehat{P}_{(k+1)}(s^\prime|s,a)-P^{k^\prime}(s^\prime|s,a) \right)^2 + \text{disc}(q) + \frac{1}{wH} \cdot \lambda ||\bar{p}||_2 \right)} 
    \end{align*}
    \jj{On the other hand,} 
    \begin{align*}
         \Delta^{p}_{(k)}(s,a) &\leq  \sum_{s^\prime \in \mathcal{S}} \sqrt{ \min_{q,\bar{p}} \left(\sum_{k^\prime=1}^{k-1} \sum_{h=0}^{H-1} q_{(k^\prime,h)}\left(\widehat{P}_{(k+1)}(s^\prime|s,a)-P^{k^\prime}(s^\prime|s,a) \right)^2 + \text{disc}(q) + \frac{1}{wH} \cdot \lambda ||\bar{p}||_2 \right)}  \\
        &\leq  |\mathcal{S}| \sqrt{ \min_{q,\bar{p}} \left(\sum_{k^\prime=1}^{k-1} \sum_{h=0}^{H-1} q_{(k^\prime,h)}\left|\left|\widehat{P}_{(k+1)}(\cdot|s,a)-P^{k^\prime}(\cdot|s,a) \right|\right|_{\infty}^2 + \text{disc}(q) + \frac{1}{wH} \cdot \lambda ||\bar{p}||_2 \right)} 
    \end{align*}
        Recall the Corollary \ref{corollary2}, Corollary\ref{corollary3} and $\mathfrak{R}_{I}$ definition. \jj{Aftering fixing} $(s,a)$, \jj{the term} $\mathfrak{R}_{I}(s,a)$ \jj{can be expressed as} 
    \begin{align*}
        \mathfrak{R}_{I} =& \frac{1}{1-\gamma}\sum_{k=1}^{K-1} \bar{\iota}_{\infty}^{k+1} + \sum_{k=1}^{K-1} \sum_{h=0}^{H-1} - \iota^{(k+1)}_{h} + C_p \sqrt{K-1} \\
         \leq& \frac{1}{1-\gamma} \sum_{k=1}^{K-1}  \left( \Delta^{r}_{(k)}(s,a) +  \Delta^{p}_{(k)}(s,a) \frac{\gamma \hat{r}_{max}}{1-\gamma} -  2\Gamma^{(k)}_w(s,a) \right) \\
        & + \sum_{k=1}^{K-1} \sum_{h=0}^{H-1} \left( \Delta^{r}_{(k)}(s,a) + \Delta^{p}_{(k)}(s,a)\frac{\gamma  \hat{r}_{max} }{1-\gamma} +  2\Gamma^{(k)}_w(s,a) \right) \\ 
        & +  C_p \sqrt{K-1} \\
         \leq& \frac{1}{1-\gamma} \sum_{k=1}^{K-1}  \left( \Delta^{r}_{(k)}(s,a) +  \Delta^{p}_{(k)}(s,a) \frac{\gamma}{1-\gamma}\left( \tilde{r}_{\text{max}} + \max(2\Gamma^{(k)}_w(s,a)) \right)  -  2\Gamma^{(k)}_w(s,a) \right) \\
        & + H \sum_{k=1}^{K-1} \left( \Delta^{r}_{(k)}(s,a) + \Delta^{p}_{(k)}(s,a)\frac{\gamma}{1-\gamma}\left( \tilde{r}_{\text{max}} + \max(2\Gamma^{(k)}_w(s,a)) \right) +  2\Gamma^{(k)}_w(s,a) \right) \\ 
        & +  C_p \sqrt{K-1} \\
         \leq& \sum_{k=1}^{K-1}  \bigg( \underbrace{\left( \frac{1}{1-\gamma} + H\right) \Delta^{r}_{(k)}(s,a) +  \frac{\gamma \tilde{r}_{\text{max}}}{1-\gamma}  \left( \frac{1}{1-\gamma} + H\right) \Delta^{p}_{(k)}(s,a)}_{\circled{1}} + \\
        & + \frac{\gamma}{1-\gamma}  \left( \frac{1}{1-\gamma} + H\right) \max(2\Gamma^{(k)}_w(s,a)) \Delta^{p}_{(k)}(s,a) \bigg) \\
        & + \sum_{k=1}^{K-1}  2 \left( -\frac{1}{1-\gamma} + H\right) \Gamma^{(k)}_w(s,a)  \\ 
        & +  C_p \sqrt{K-1}
    \end{align*}
    \jj{We set} the term $\circled{1}$ \jj{to be} $2(\frac{1}{1-\gamma} +H)\Gamma^{(k)}_w(s,a)$\jj{, which requires redefining} the exploration bonus term \jj{as}
    \begin{align*}
        \Gamma^{(k)}_w(s,a) = \frac{1}{2} \Delta^{r}_{(k)}(s,a) + \frac{\gamma \tilde{r}_{\text{max}}}{2(1-\gamma)}  \Delta^{p}_{(k)}(s,a).
    \end{align*}
    Also, note that $\Delta^{p}_{(k)}(s,a) = \sum_{s^\prime \in \mathcal{S}}\left| \left( \widehat{P}_{(k+1)}- P_{(k+1)} \right) (s^\prime|s,a)\right| \leq \left| \mathcal{S} \right|$. \jj{Therefore,} 
    \begin{align*}
        \mathfrak{R}_{I} 
        &\leq \sum_{k=1}^{K-1} \left( 4H \Gamma^{(k)}_w(s,a)+ \frac{2\gamma}{1-\gamma}  \left( \frac{1}{1-\gamma} + H\right) \max(\Gamma^{(k)}_w(s,a)) \left| \mathcal{S} \right|   \right) \\
        &\leq \sum_{k=1}^{K-1} \left( 4H + \frac{2\gamma}{1-\gamma}  \left( \frac{1}{1-\gamma} + H\right)  \right) \max(\Gamma^{(k)}_w(s,a))  \\
        &= \left( 4H + \frac{2\gamma \left| \mathcal{S} \right| }{1-\gamma}  \left( \frac{1}{1-\gamma} + H\right)  \right) \sum_{k=1}^{K-1}  \max(\Gamma^{(k)}_w(s,a))  \\
        & \leq \left( 4H + \frac{2 \gamma \left| \mathcal{S} \right|}{1-\gamma}  \left( \frac{1}{1-\gamma} + H\right)  \right) \sum_{k=1}^{K-1} \left(  \frac{1}{2} \max (\Delta^{r}_{(k)}(s,a) ) + \frac{\gamma \tilde{r}_{\text{max}}}{2(1-\gamma)}  \max( \Delta^{p}_{(k)}(s,a) ) \right)  \\
        & = \left( 4H + \frac{2\gamma \left| \mathcal{S} \right|}{1-\gamma}  \left( \frac{1}{1-\gamma} + H\right)  \right) \sum_{k=1}^{K-1} \left(  \frac{1}{2} \Delta^{r}_{(k)}(s,a)  + \frac{\gamma \tilde{r}_{\text{max}}}{2(1-\gamma)}  \Delta^{p}_{(k)}(s,a) \right)  \\
        & = \left( 4H + \frac{2\gamma \left| \mathcal{S} \right| }{1-\gamma}  \left( \frac{1}{1-\gamma} + H\right)  \right) \left(  \frac{1}{2} \sum_{k=1}^{K-1} \Delta^{r}_{(k)}(s,a)  + \frac{\gamma \tilde{r}_{\text{max}}}{2(1-\gamma)}  \sum_{k=1}^{K-1} \Delta^{p}_{(k)}(s,a) \right)  \\
    \end{align*}
    Note that above upper bound \jj{on} $\mathfrak{R}_{I}$ holds \jj{under the following conditions for} $\Delta^{r}_{(k)}(s,a)$ \jj{and} $\Delta^{p}_{(k)}(s,a)$\jj{:}
    \begin{align*}
         \Delta^{r}_{(k)}(s,a) &\leq  \sqrt{ \min_{q,\bar{r}} \left(\sum_{k^\prime=1}^{k-1} \sum_{h=0}^{H-1} q_{(k^\prime,h)}\left(\widetilde{R}_{(k+1)}(s,a)-R^{k^\prime}(s,a) \right)^2 + \text{disc}(q) + \frac{1}{wH} \cdot \lambda ||\bar{r}||_2 \right)}, \\
        \Delta^{p}_{(k)}(s,a) &\leq  \sum_{s^\prime \in \mathcal{S}} \sqrt{ \min_{q,\bar{p}} \left(\sum_{k^\prime=1}^{k-1} \sum_{h=0}^{H-1} q_{(k^\prime,h)}\left(\widehat{P}_{(k+1)}(s^\prime|s,a)-P^{k^\prime}(s^\prime|s,a) \right)^2 + \text{disc}(q) + \frac{1}{wH} \cdot \lambda ||\bar{p}||_2 \right)}. 
    \end{align*}
\end{proof}

\begin{proof}[\textbf{Proof of Remark \ref{rmk:W-LSE_modelerror}}]
    The proof starts \jj{with} \eqref{eqn:startfromhere}. \jj{Define}  
    \begin{align}
        q^{sw}_{t} &= 
        \begin{cases}
            \frac{1}{wH} & \text{if } t \in (k-w,k] \\
            0 & \text{otherwise} 
        \end{cases}, \nonumber \\
        r^{sw} &= \argmin_{\bar{r}} \left( \lambda ||\bar{r}||^2 + \sum_{t=1}^{N}  \left(r(s,a)-\widehat{Y}_{t} \right)^2 \right). \label{eqn:r_slidingwindow}
    \end{align}
    where $r_{sw}$ is the same reward estimation \jj{as in} \eqref{eqn:r_p_estimation}. Then the minimum \jj{of \eqref{eqn:thm3}} yields that
    \begin{align}
        &\min_{\bar{r},q} \left( \sum_{t=1}^{N}   q_{t} \left(r(s,a)-\widehat{Y}_{t} \right)^2 + \text{disc}(q) + \frac{1}{wH} \cdot \lambda ||\bar{r}||_2    \right) \label{eqn:thm3_1} \\ 
        &\leq \min_{\bar{r}} \left( \sum_{t=1}^{N}   q^{sw}_{t} \left(r(s,a)-\widehat{Y}_{t} \right)^2 + \text{disc}(q^{sw}) + \frac{1}{Hw} \cdot \lambda ||\bar{r}||_2 \right)\nonumber \\
        &\leq \frac{1}{Hw} \underbrace{\min_{\bar{r}} \left(  \sum_{t=1}^{N} (Hw) \cdot q^{sw}_{t} \left(r(s,a)- \widehat{Y}_{t} \right)^2  + \lambda ||\bar{r}||_2 \right)}_{\circled{1}}+~\text{disc}(q_{sw}). \label{eqn:thm3_2}
    \end{align}
    The term $\circled{1}$ is the optimization problem \jj{in} \eqref{eqn:r_slidingwindow} \jj{whose} minimizer is $r^{sw}$. \jj{An} inspection \jj{of} \eqref{eqn:thm3_1} \jj{and} \eqref{eqn:thm3_2} \jj{concludes} that the \jj{optimal solution} ($q^{*},\bar{r}^{*}$), \jj{namely the} minimizer of \eqref{eqn:thm3_1} \jj{provides a} smaller value than $(q^{sw},r^{sw})$. Since the \jj{right-hand side} \eqref{eqn:startfromhere} is same \jj{as} \eqref{eqn:thm3_1}, $(q^{*},\bar{r}^{*})$ \jj{provides} a tighter upper bound \jj{on the left-hand side} term of equation \eqref{eqn:thm3} than $q^{sw},r^{sw}$. 
    Therefore, \eqref{eq:thm3_3} \jj{implies} that the optimal \jj{solution} $(q^{*},\bar{r}^{*})$ gives a tighter upper bound \jj{on} $\Delta^{r}_{(k)}$ \jj{than using} $(q^{sw},\bar{r}^{sw})$. \jj{One can repeat \jj{the} above argument} for the upper bound \jj{on} $\Delta^{p}_{(k)}$. Then, by Corollary \ref{corollary2} and \ref{corollary3}, the tighter upper bound\jj{s on} $\Delta^{r}_{(k)}(s,a)$ \jj{and} $\Delta^{p}_{(k)}(s,a)$ \jj{provide} smaller upper \jj{bounds on} $-\iota^{(k)}_H,\bar{\iota}^{K}_\infty$ and \jj{lead to} a tighter upper bound on $\mathfrak{R}_{I}$.
\end{proof}
\section{Experimental design and results}
\label{section:Experimental design and results}

\subsection{Environment setting details}
\label{appendix:environment setting details}
\textbf{Reward function design. } 

All three environments share the same reward function structure and have an identical goal. The reward function $R$ consists of three parts $R = R_h + R_f - R_c$\jj{, where} $R_h$ is the healthy reward, $R_f = k_{f} (x_{t+1} - x_{t})/ \Delta t , k_{f}>0$ is the forward reward, and $R_c$ is the control cost. \jj{The agents} have a goal to run faster in the $+x$ direction, \jj{and therefore} the faster they run, the higher the forward reward $R_f$ \jj{is}. We \jj{modify} the environment to make the agent's desired directions change as the episode goes by. To be specific, we \jj{design} the forward reward $R_f$ to change as episodes progress \jj{in the form of}  $R^k_f = o_k \cdot k_{f} (x_{t+1} - x_{t})/ \Delta t$ where $o_k=a\text{sin}(w bk)$ and $k$ is a episode where $a,b>0$ are constants. \jj{A} positive $o_k$ causes the agent to desire \jj{a} forward $+x$ direction as an optimal policy, \jj{and a} negative $o_k$ causes it to desire \jj{a} backward $-x$ direction. We \jj{generate} different speeds of non-\jj{stationarity} by changing the frequency variable $w \in \{ 1,2,3,4,5 \} $.

\textbf{Non-stationary variable $o_k$ generator. } 
\begin{enumerate}
    \item Sine function: The non-stationary \jj{parameter} $o_k$ is designed as
$o_k = \sin{ \left( 2\pi w k / 37 \right)}$\jj{, where} $w$ is the integer speed of the environment change and $k$ is the \jj{episode number}. We \jj{change} $w$ \jj{in the set} $[1,2,3,4,5]$. We \jj{divide} $2 \pi w k$ by 37, \jj{a} prime number, to \jj{ensure} that the environment has various non-stationary modes and to avoid certain non-stationary \jj{parameters} appearing frequently.
    \item Real data: we \jj{bring} the stock price data to model \jj{a} non-stationary real \jj{dataset}.
\end{enumerate}
\begin{figure}[h]
    \centering
    \includegraphics[width=0.8\textwidth]{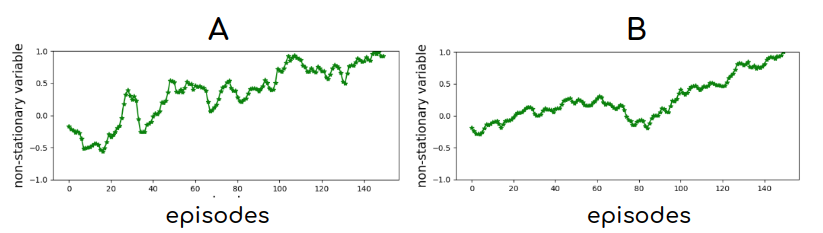}
    \caption{Nonstationary parameter from real data A,B}
    \label{fig:ns_variable}
\end{figure}

\textbf{Non-stationary \jj{parameter} $o_k$ generator (ablation study). }
$B(G)$ satisfies the property of the time-elapsing variation budget that $B(G)$ increases as $G$ increases. For the ablation study, we generate $o_k = \sin{(2\pi\cdot G \cdot k} / 37)$\jj{, where} $G \in \{ 38, 76, 114, 152, 190\}$. We estimated $B(G)$ as $\sum_{k=1}^{150} |o_{k+1}-o_k|$:
\begin{center}
\begin{tabular}{||c | c c c c c||} 
 \hline  
 & $G=38$ & $G=76$ & $G=114$ & $G=152$ & $G=190$ \\ [0.5ex] 
 \hline\hline
  $B(G)$ &  15.98 & 31.85 &  47.49 & 62.79 & 77.64 \\  [0.5ex] 
 \hline
\end{tabular}
\end{center}

\subsection{Hyperparameters and implementation details}
\label{appendix:hyperparameters and implementation details}
\textbf{Training Details. }

For the ARIMA model that serves as a forecatser $f$, we \jj{use the} auto$\_$arima function of pmdarima python package to find the optimal $p,q,d$. To compare the results between \texttt{ProST-G} and MBPO, we \jj{train} the MBPO and \texttt{ProST-G} with \jj{the} initial learning rate $lr=0.0003$ with the decaying parameter $0.999$. For \texttt{ProST-G}, We \jj{add the} uniform noise $\eta \sim \text{Unif}([-b,b])$ \jj{to the} non-stationary \jj{parameter} $o^k$ to generate \jj{the} noisy non-stationary parameter $\hat{o}_k = o_k +\eta$ with different noise bounds $b \in \{ 0.01,0.03,0.05 \}$. \hh{We denote} \text{Unif}$([-b,b])$ \hh{as continuous uniform distributions over the interval $[-b,b]$.}

To compare the \jj{results} between \texttt{ProST-G} \jj{and} ProOLS, ONPG, FTML, we \jj{train} these three baselines with eight different initial learning \jj{rates} $lr \in \{0.001,0.003,0.005,0.007,0.01,0.03,0.05,0.07 \}$.

\textbf{Hyper parameters.}

{\renewcommand{\arraystretch}{1.4}
\small
\begin{tabular}{ ||c|c||c|c|c||  }
    \hline 
    Letter &  hyper parameters & Swimmer-v2 & Half cheetah-v2 & Hopper-v2 \\ 
    \hline
    \hline
    $K$ & episodes  &  100   &  150 & 150   \\
    \hline
    $H$ & environment steps per episodes & \multicolumn{3}{c||}{100} \\
    \hline
    $G$ & policy updates per epochs & \multicolumn{3}{c||}{50} \\
    \hline
    $\widehat{H}$ & model rollout length & \multicolumn{3}{c||}{1 $\rightarrow$ 15 over episodes 20 $\rightarrow$150}\\ 
    \hline
    $N$ & iteration of policy update and policy evaluation & \multicolumn{3}{c||}{1} \\
    \hline
    $M$ & model rollout batch size ($D_{syn}$) & \multicolumn{3}{c||}{1e5} \\
    \hline
    $\tau$ & entropy regularization parameter & \multicolumn{3}{c||}{0.2} \\
    \hline
    $\gamma$ & reward discounting factor & \multicolumn{3}{c||}{0.99} \\
    \hline

\end{tabular}
}

Note that $\widehat{H}$ increases linearly within a certain range as episode goes by. We denote $h_{min} \rightarrow h_{max}$ over episodes $k_{min}\rightarrow k_{max}$ as $\widehat{H}(k) = \min( \max(h_{min}+(k-k_{min})/(k_{max}-k_{min})\cdot (h_{max}-h_{min}) ,h_{min}), h_{max})$.

\newpage
\subsection{Full results}
\label{subsection:Full results}
\subsubsection{Non-stationarity: sine wave}
\textbf{(1) Swimmer-v2}
\begin{figure}[h]
    \centering
    \includegraphics[width=0.95\textwidth]{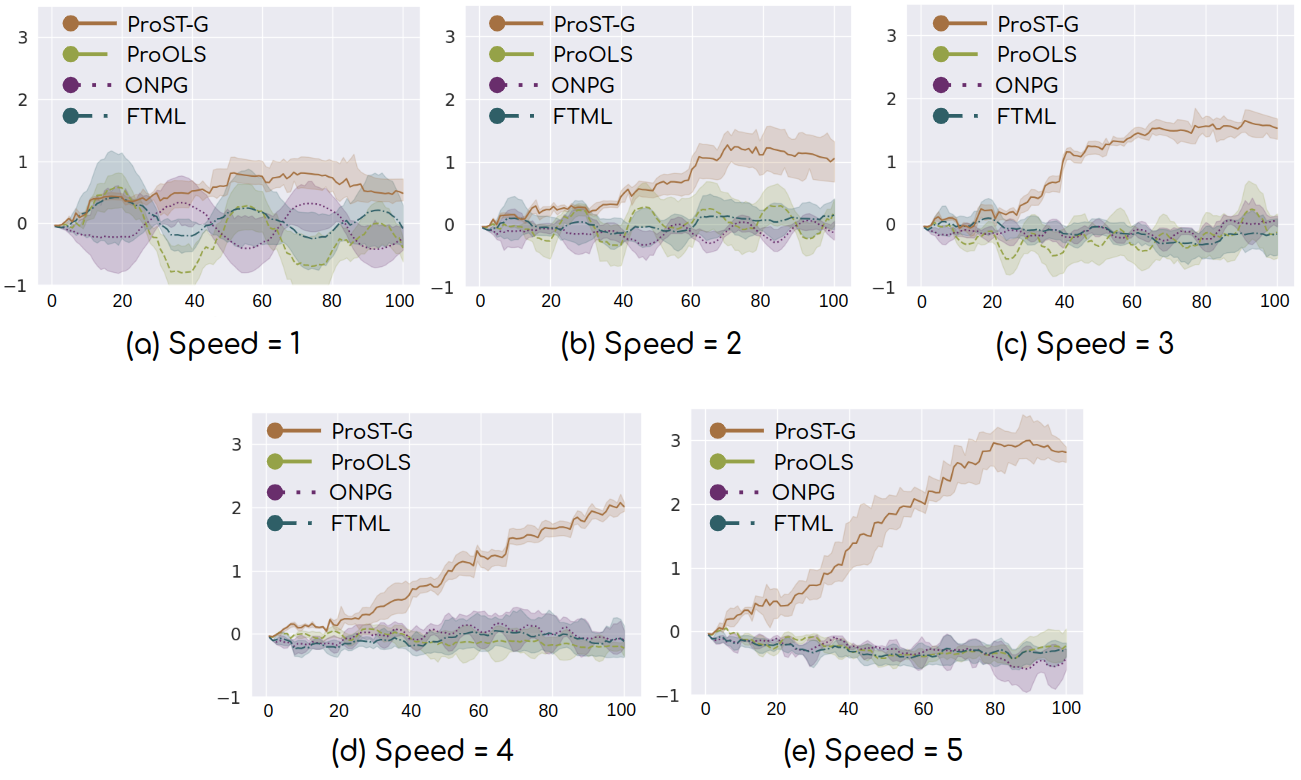}
     \caption{(a) $\sim $ (e) the average rewards of \texttt{ProST-G}, and the three baselines: ProOLS, ONPG, FTML for 5 different speeds ($x$-axis indicates the episode). The shaded area of \texttt{ProST-G} is 95\% confidence area among 3 different noise bounds, and the shaded areas of three baselines are the 95 \% confidence area among 8 different learning rates.}
\end{figure}

\begin{figure}[h]
    \centering
    \includegraphics[width=1.0\textwidth]{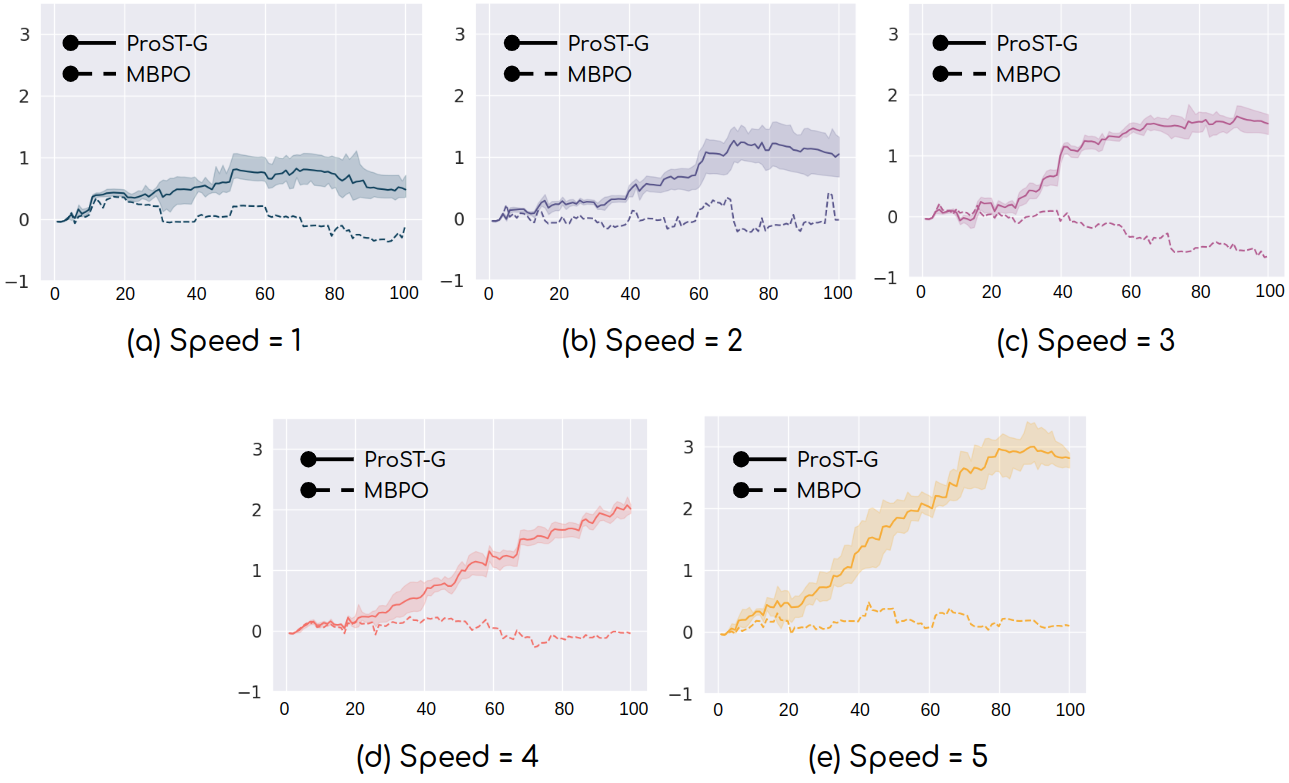}
    \caption{(a) $\sim $ (e) the average rewards of \texttt{ProST-G} and MBPO. The shaded area of \texttt{ProST-G} is 95\% confidence area among 3 different noise bounds.}
\end{figure}

\newpage

\textbf{(2) Halfcheetah-v2}

\begin{figure}[h]
    \centering
    \includegraphics[width=0.95\textwidth]{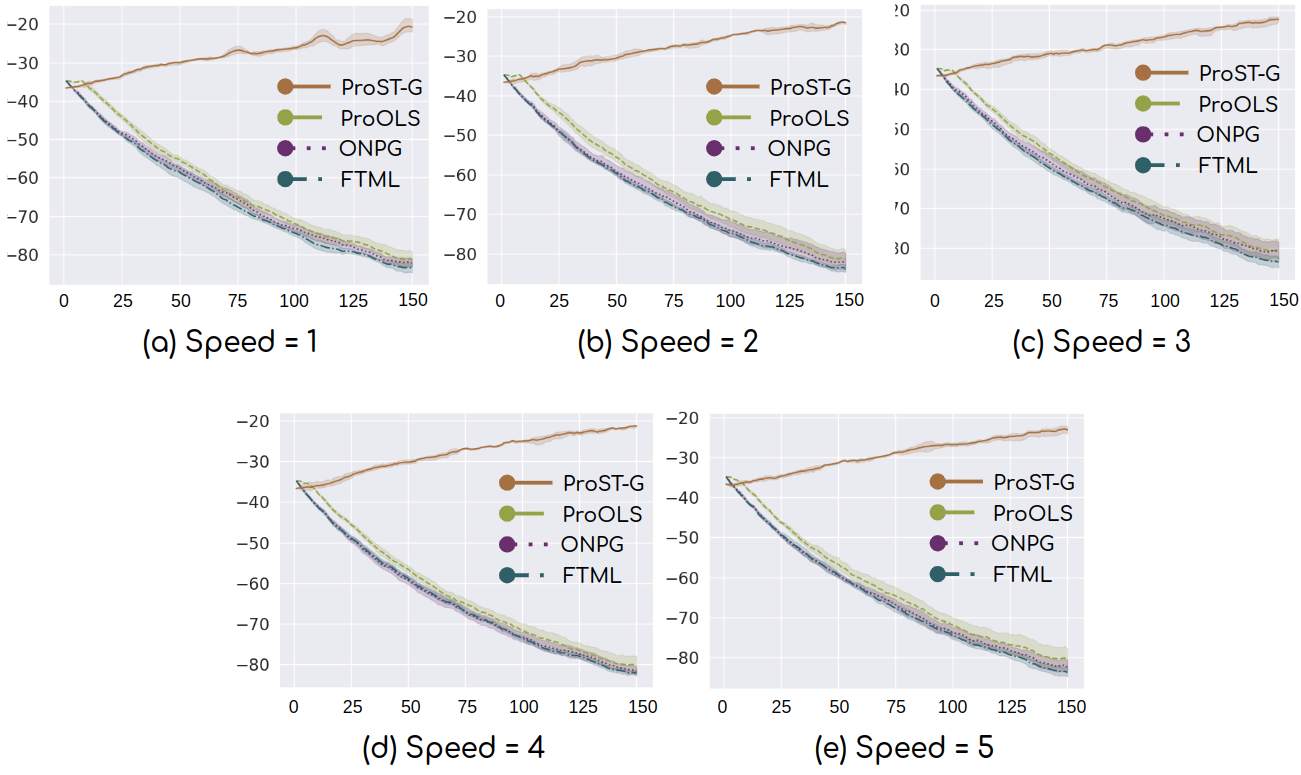}
     \caption{(a) $\sim $ (e) the average rewards of \texttt{ProST-G}, and the three baselines: ProOLS, ONPG, FTML for 5 different speeds ($x$-axis indicates the episode). The shaded area of \texttt{ProST-G} is 95\% confidence area among 3 different noise bounds, and the shaded ares of three baselines are the 95\% confidence areas among 8 different learning rates.}
\end{figure}

\begin{figure}[h]
    \centering
    \includegraphics[width=1.0\textwidth]{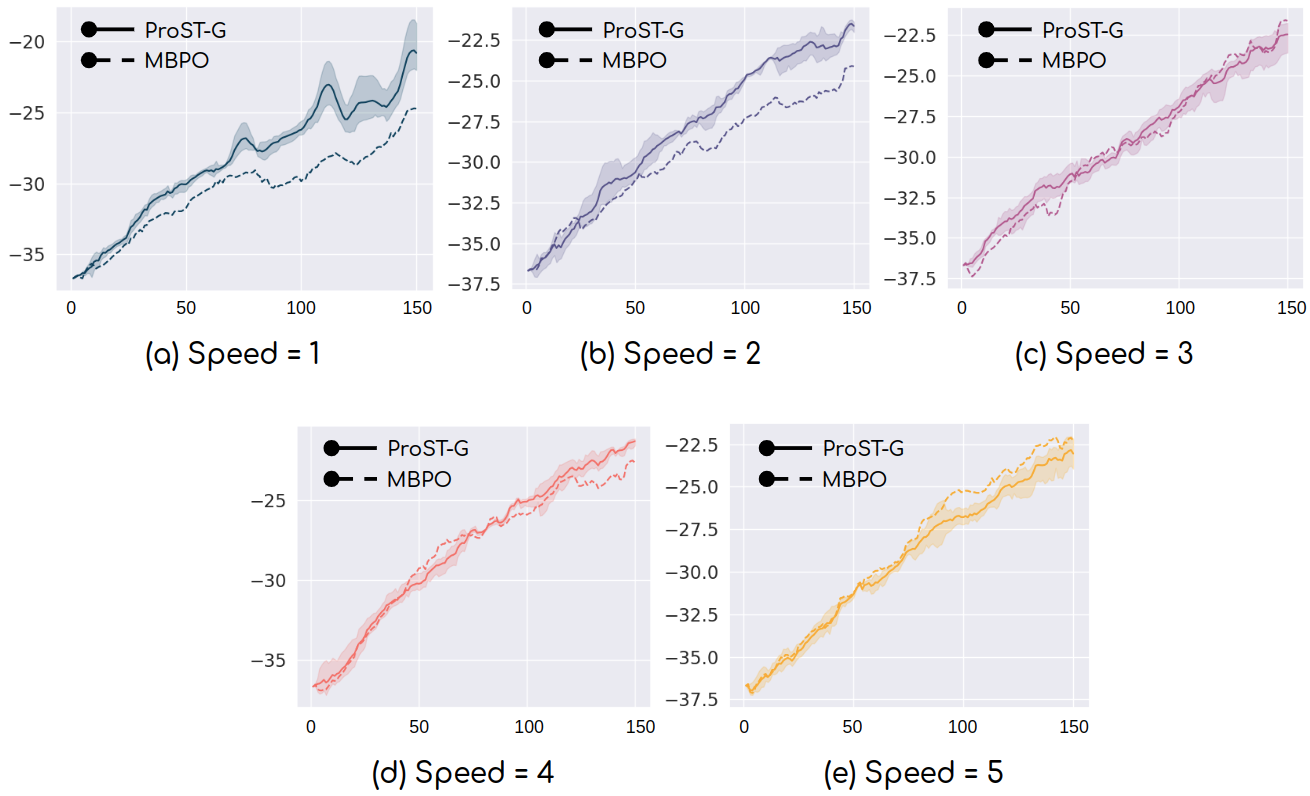}
    \caption{(a) $\sim $ (e) the average rewards of \texttt{ProST-G} and MBPO ($x$-axis indicates the episode). The shaded area of \texttt{ProST-G} is 95\% confidence area among 3 different noise bounds.}
\end{figure}

\newpage 

\textbf{(3) Hopper-v2}

\begin{figure}[h]
    \centering
    \includegraphics[width=0.95\textwidth]{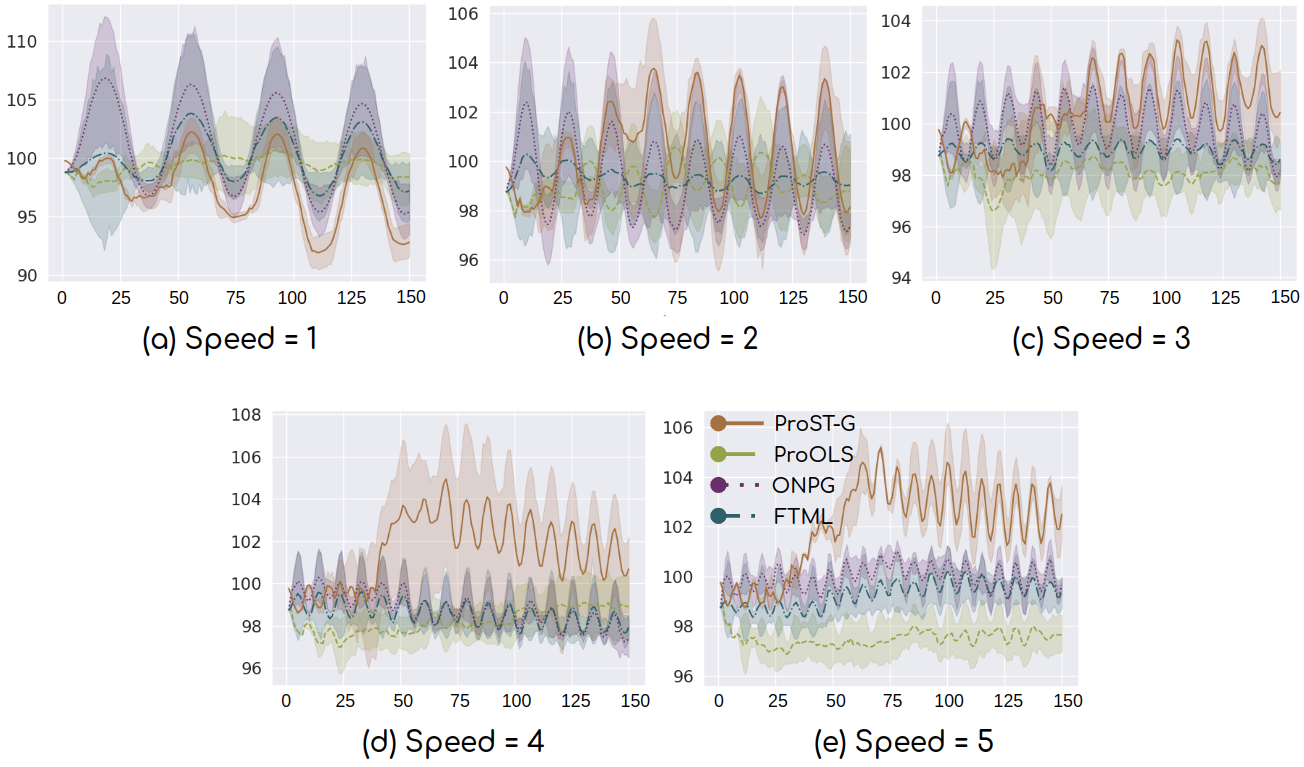}
     \caption{(a) $\sim $ (e) the average rewards of \texttt{ProST-G}, and the three baselines : ProOLS, ONPG, FTML for 5 different speeds ($x$-axis indicates the episode). The shaded area of \texttt{ProST-G} is 95\% confidence area among 3 different noise bounds, and the shaded areas of three baselines are the 95\% confidence areas among 8 different learning rates.}
\end{figure}

\begin{figure}[h]
    \centering
    \includegraphics[width=1.0\textwidth]{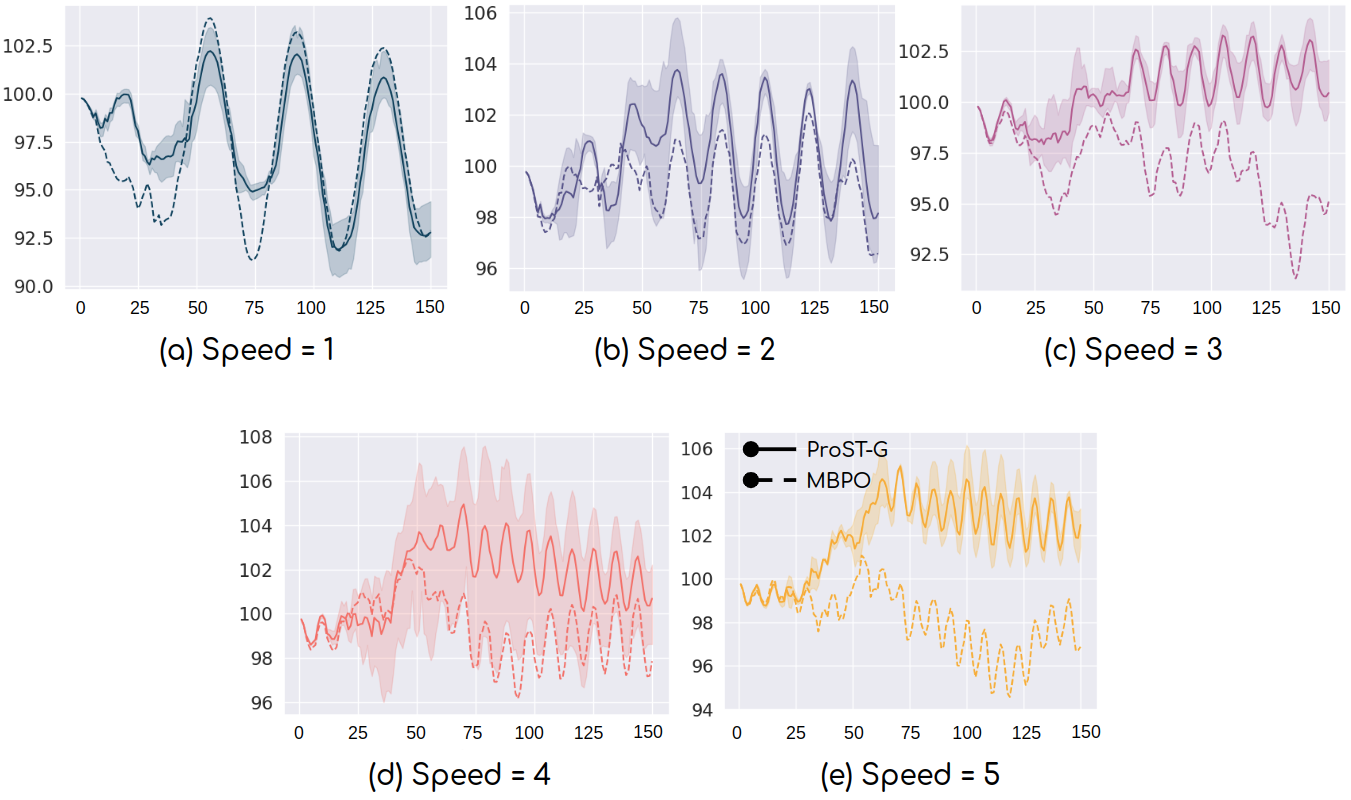}
    \caption{(a) $\sim $ (e) the average rewards of \texttt{ProST-G} and MBPO ($x$-axis indicates the episode). The shaded area of \texttt{ProST-G} is 95\% confidence area among 3 different noise bounds.}
\end{figure}

\newpage 

\subsubsection{Non-stationarity : real data}
The shaded area of \texttt{ProST-G} is 95\% confidence area among 3 different noise bounds, and the shaded ares of three baselines are the 95\% confidence area among 8 different learning rates.

\textbf{(1) Swimmer-v2}

\begin{figure}[h]
    \centering
    \includegraphics[width=1.0\textwidth]{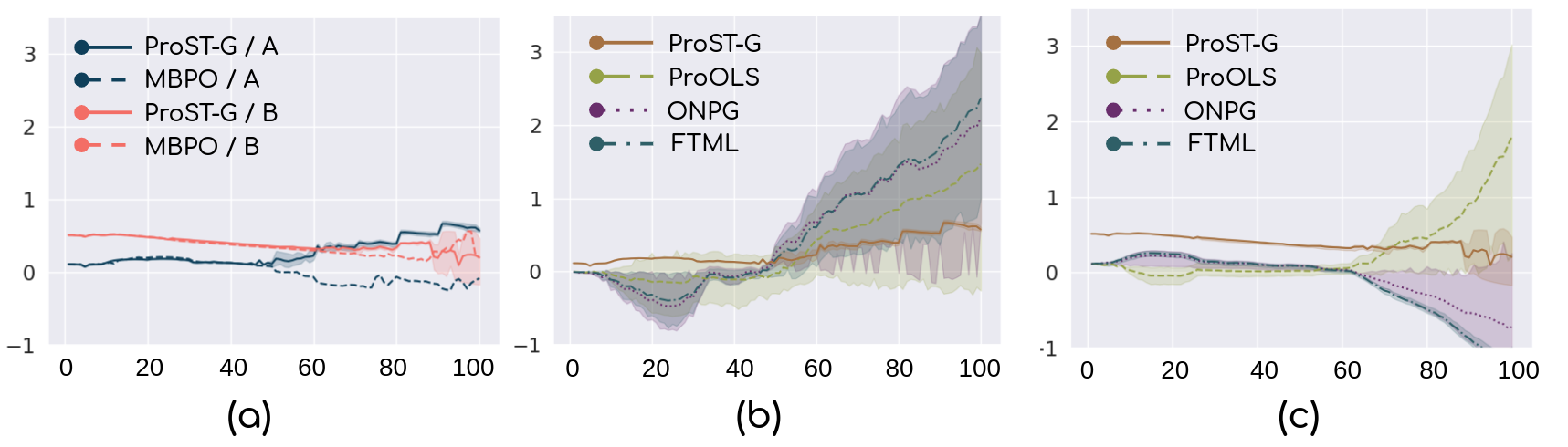}
    \caption{(a) average reward with \texttt{ProST-G} and MBPO on real data A,B ($x$-axis is episode). (b) average reward with ProST-G and three baselines on realdata A. (c) average reward with ProST-G and three baselines on realdata B.}
\end{figure}

\textbf{(2) Halfcheetah-v2}

\begin{figure}[h]
    \centering
    \includegraphics[width=1.0\textwidth]{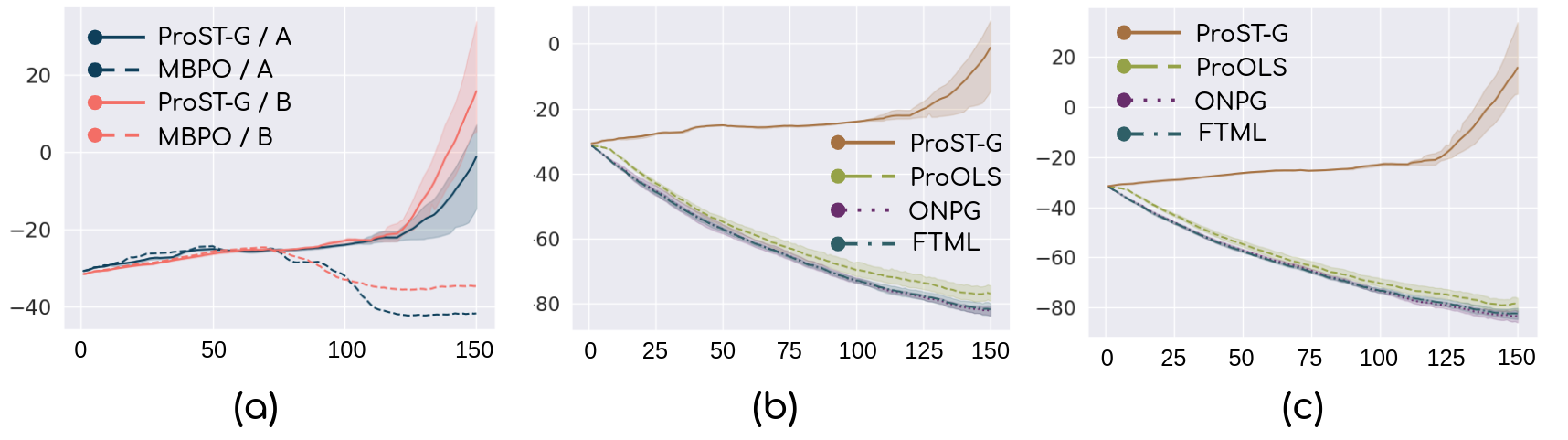}
    \caption{(a) average reward with \texttt{ProST-G} and MBPO on real data A,B ($x$-axis is episode). (b) average reward with ProST-G and three baselines on realdata A. (c) average reward with ProST-G and three baselines on realdata B.}
\end{figure}

\textbf{(3) Hopper-v2}

\begin{figure}[h]
    \centering
    \includegraphics[width=1.0\textwidth]{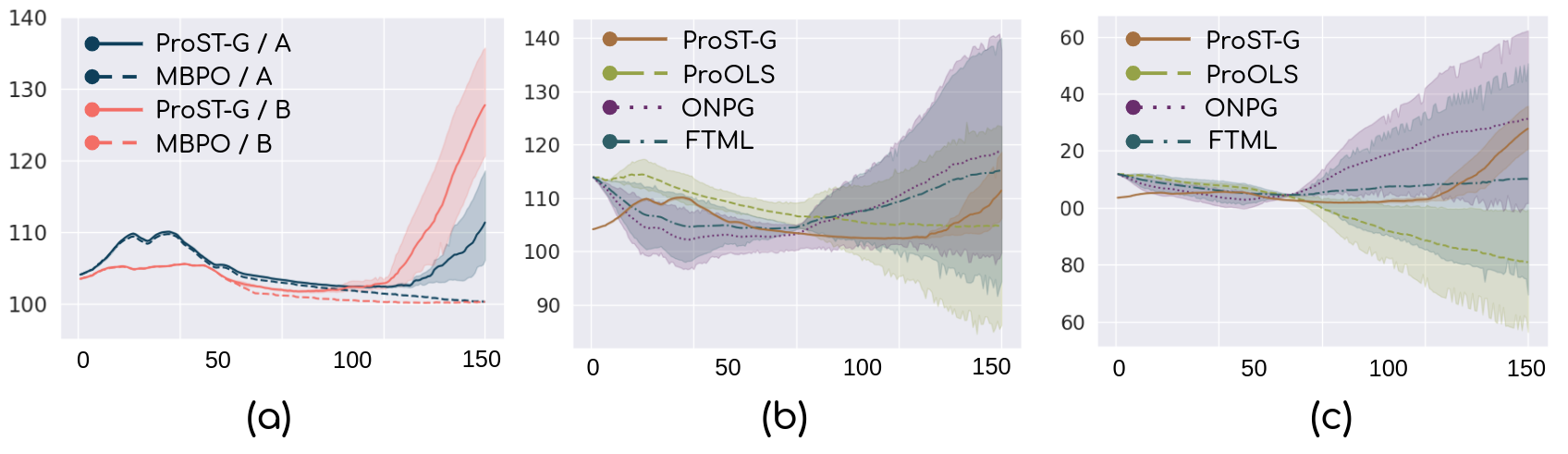}
    \caption{(a) average reward with \texttt{ProST-G} and MBPO on real data A,B ($x$-axis is episode). (b) average reward with ProST-G and three baselines on realdata A. (c) average reward with ProST-G and three baselines on realdata B.}
\end{figure}

\section{Algorithms}


\subsection[Meta-algorithm IgfA]{\texttt{ProST} framework}
\label{section:meta_algorithm}
\begin{algorithm}[H]
\SetNoFillComment
\textbf{Set} : $k_f=1$ \\
\textbf{Init} : policy $\pi^{0}$, forecaster $f_{\phi^0_{f}}$, model estimator $g_{\phi^0_{g}}$ , two dataset $\mathcal{D}_{env}, \mathcal{D}_{syn}$

\For{episode $k$}{
    Execute the agent with $\pi^{k}$ in a environment $\mathcal{M}_{k}$ and add a trajectory to $\mathcal{D}_{env}$. \\
    \tcc{MDP forecaster $g \circ f$}
    \tcc{(1) Observe and forecast:}
    Observe a noisy non-stationary parameter $\hat{o}_k$\\
    Update $f_{\phi_f},g_{\phi_g}$ using $\mathcal{D}_{env} $ and $\hat{o}_{k-(w-1):k}$.\\
    Use $f_{\phi^k_f},g_{\phi^k_g}$ to predict the future $\widehat{\mathcal{P}}^{k+1},\widehat{\mathcal{R}}^{k+1}$ and construct future MDP $\widehat{\mathcal{M}}_{k+1}$\\
    \tcc{Baseline $A$}
    \tcc{(2) Optimize:}
    Roll out synthetic trajectories in  $\widehat{\mathcal{M}}_{k+1}$ and add them to $\mathcal{D}_{syn}$ \\
    
    Use $\mathcal{D}_{syn}$to evaluate and update $\pi^{k}$ to $\widehat{\pi}^{k+1}$
}
\caption{\texttt{ProST} framework}
\label{alg:meta_algorithm}
\end{algorithm}

\subsection{\texttt{ProST-T} algorithm}
\label{section:FT-MBPO}
\begin{algorithm}[H]
\SetNoFillComment
\textbf{Set} : $k_f=1$ \\
\textbf{Init} : policy $\pi^{k}$ , forecaster $f_{\phi^k_f}$, tabular reward model $g^R_k$, tabular transition probability model $g^P_k$, forecasted state-action value $\widehat{Q}^{\boldsymbol{\cdot},k+1}$, empty dataset $\mathcal{D}_{env},\mathcal{D}_{syn}$ \\
Explore $w$ episodes and add $(\tau^{-k},\hat{o}_{-k})$ to $D_{env}$ where $k \in [w]$ before starts \\
\For{episodes $k=1,..,K$}{
    Rollout a trajectory $\tau_k$ using $\pi^{k}$and $\mathcal{D}_{env} = \mathcal{D}_{env} \cup \{ \tau_k \}$ \\
    Observe a noisy non-stationary parameter $\hat{o}_k$ \\
    \tcc{MDP forecaster $g \circ f$: (1) update $f,g$}
    Update $f_{\phi_f}$ : $\phi^k_f \leftarrow \argmin_{\phi} \mathcal{L}_f (\hat{o}_{k-(w-1):k} ; \phi )$ \\
    Update $g^P_k(s^\prime,s,a,o)$ \\
    Update $g^R_k(s,a,o)$ \\
    \tcc{MDP forecaster $g \circ f$: (2) predict $\widehat{\mathcal{P}}^{k+1},\widehat{\mathcal{R}}^{k+1}$}
    Forecast $1$ episode ahead non-stationary parameter:  $\hat{o}_{k+1} = f_{\phi^k_f}(\hat{o}_{k-(w-1):k})$ \\
    Forecast transition probability function: $\widehat{g}^P_{k+1} = g^P_{k}(\cdot,\hat{o}_{k+1})$ \\
    Forecast reward function: $\widehat{g}^R_{k+1} = g^R_{k}(\cdot,\hat{o}_{k+1})$ \\
    Reset $\mathcal{D}_{syn}$ to empty.\\
    \tcc{Baseline $A$:  NPG with entropy regularization}
    Set $\widehat{\pi}^{(0)} \leftarrow \pi^{k}$ \\
    \For{$g=0,..,G-1$}{
        Evaluate $Q_{\tau}^{\widehat{\pi}^{(g)}}$ using the rollouts from the future model $\widehat{g}^{P}_{k+1},\widehat{g}^{R}_{k+1}$ \\
        Update $\widehat{\pi}$ : $\widehat{\pi}^{(g+1)} \leftarrow 1/Z^{(t)} \cdot  \left( \widehat{\pi}^{(g)} \right)^{1-\frac{\eta\tau}{1-\gamma}} \exp{\left( (\eta\widehat{Q}_{\tau}^{\widehat{\pi}^{(g)}}) / (1-\gamma) \right)} $ \\
        where $Z^{(t)} = \sum_{a\in \mathcal{A}} \left( \widehat{\pi}^{(g)} \right)^{1-\frac{\eta\tau}{1-\gamma}} \exp{\left( (\eta\widehat{Q}_{\tau}^{\widehat{\pi}^{(g)}}) / (1-\gamma) \right)}$
    }
    Set $\pi^{k+1} \leftarrow \widehat{\pi}^{(G)}$
}
\caption{\texttt{ProST-T}}
\label{alg:FT-MBPO}
\end{algorithm}

\newpage

\subsection{\texttt{ProST-G} algorithm}
\label{section:F-MBPO}

\textbf{(1) Forecaster $f$.} 
We adopt the ARIMA model to forecast $\hat{o}_{k+1}$ from the noisy \jj{observation} $\widehat{o}_{k-(w-1):k}$. \jj{The} ARIMA model is one of the most general \jj{classes} of models for forecasting a time series\jj{, which} can be made to be stationary by taking a difference among the data. For given time series data $X_t$, we define ARIMA$(p,d,q)$ as given by $X_t - \alpha_1X_{t-1}- \dots -\alpha_{p}X_{t-p} = \epsilon_t + \theta_1 \epsilon_{t-1} + \cdots + \theta_q \epsilon_{t-q}$\jj{, where} $\alpha_{i}$\jj{'s} are the parameters of the autoregressive part of the model, the $\theta_i$\jj{'s} are the parameters of the moving average part\jj{, and }$\epsilon_t$\jj{'s} are the error terms that take $d$ times difference between $X_t$s, which we assume to be independent and follow a normal distribution with \jj{a} zero mean. 

\textbf{(2) Model predictor $g$.}
We use a bootstrap ensemble of dynamic models $\{ g^1_{\phi_g},g^2_{\phi_g},...,g^M_{\phi_g} \}$. Each ensemble model is a probabilistic neural network whose output is parameterized by the mean vector $\mu$ and the diagonal vector of the standard deviation $\text{Diag}(\Sigma)$ of a Gaussian distribution\jj{, namely} $g^i_{\phi_g}(s_{h+1},r_h | s_h,a_h,\widehat{o}_{k+1}) = \mathcal{N}(\mu^{i}_{\phi_{g}}(s_h,a_h),\Sigma^{i}_{\phi_{g}}(s_h,a_h))$. To efficiently handle uncertainty due to the non-stationary environment, we \jj{design} each neural network to be a probabilistic model to capture the aleatoric uncertainty, \jj{i.e.} the noise of the output, and learn multiple models as bootstrap ensemble to handle the epistemic uncertainty, \jj{i.e.} the uncertainty in the model parameters. Then we predict $s_{h+1}$ \jj{and} $r_{h}$ from a model uniformly chosen from its ensemble randomly that admits different transitions along a single model rollout to be sampled from different dynamics modes.

\textbf{(3) Baseline algorithm $A$.}
We adopt soft-actor critic (SAC) as our policy optimization algorithm. SAC alternates the policy evaluation step and the policy optimization step. For a given policy $\widehat{\pi}$, it estimates the forecasted $\widehat{Q}^{\widehat{\pi},k+1}$ value using \jj{the} Bellman backup operator and optimizes the policy that minimizes the expected KL-divergence between $\pi$ and the exponential of the difference $\widehat{Q}^{\widehat{\pi},k+1}-\widehat{V}^{\widehat{\pi},k+1}$ : $\mathbb{E}_{s \sim D_{syn}}[D_{KL}(\widehat{\pi}||\exp{(\widehat{Q}^{\widehat{\pi},k+1}-\widehat{V}^{\widehat{\pi},k+1})}]$.

\begin{algorithm}[H]
\SetNoFillComment
\textbf{Set} : $k_f=1$ \\
\textbf{Init} : policy $\pi^{k}$, forecaster $f_{\phi^k_{f}}$, model estimator $g_{\phi^k_{g}}$ , two dataset $\mathcal{D}_{env}, \mathcal{D}_{syn}$ \\

Explore $w$ episodes and add $(\tau^{-k},\hat{o}_{-k})$ to $D_{env}$ where $k \in [w]$ before starts \\
\For{episodes $k=1,..,K$}{
    Execute the agent with $\pi^{k}$ in a environment $\mathcal{M}_{k}$ and add a trajectory to $\mathcal{D}_{env}$. \\
    \tcc{MDP forecaster $g \circ f$: (1) update $f,g$}
    Observe a noisy non-stationary variable $\hat{o}_k$\\
    Optimize $f_{\phi^k_f}$ on $\hat{o}_{k-(w-1):k}$ \\
    Optimize $g_{\phi^k_g}$ on $\mathcal{D}_{env} $ \\
    \tcc{MDP forecaster $g \circ f$: (2) predict $f,g$}
    Forecast $\hat{o}_{k+1} = f_{\phi^{k}_f}(\hat{o}_{k-(w-1):k})$\\
    Forecast model : $\widehat{g}_{k+1} = g_{\phi^k_g}(\cdot,\hat{o}_{k+1})$ \\ 
    Reset $\mathcal{D}_{syn}$ to empty.\\
    \tcc{Baseline $A$: SAC}
    Set $\widehat{\pi}^{k+1} \leftarrow \pi^{k}$ \\ 
    \For{epochs $n=1,...,N$}{
        \For{model rollouts $m=1,..,M$}{
        Sample $\hat{s}^m_0$ uniformly from $D_{env}$.\\
        Perform a $\widehat{H}$-step model rollout using $\hat{a}^m_h = \widehat{\pi}^{k+1}(\hat{s}^m_h)$, $\hat{s}^{m}_{h+1} = \widehat{g}_{k+1}(\hat{s}^{m}_{h}, \hat{a}^{m}_{h})$ 
 and add a rollout  to $\mathcal{D}_{syn}$ . \\
        }
        
        \For{updates $g=1,..,G$}{
        Evaluate and update forecasted policy $\widehat{\pi}^{k+1}$ on $\mathcal{D}_{syn}$
        }
    }
    Set $\pi_{k+1}  \leftarrow \widehat{\pi}^{k+1}$ \\ 
}
\caption{\texttt{ProST-G}}
\label{alg:F-MBPO}
\end{algorithm}

\section{Experiment Platforms and Licenses}
\subsection{Platforms}
All experiments are done on 12 Intel Xeon CPU E5-2690 v4 and 2 Tesla V100 GPUs.

\subsection{Licenses}
We have used the following libraries/ repos for our python codes: 
\begin{itemize}
    \item Pytorch (BSD 3-Clause "New" or "Revised" License).
    \item OpenAI Gym (MIT License).
    \item Numpy (BSD 3-Clause "New" or "Revised" License).
    \item Official codes distributed from the paper \cite{chandak2020optimizing}: to compare the four baselines.
    \item Official codes distributed from the paper \cite{janner2019trust}: to build \texttt{PMT-G}.
\end{itemize}

\end{document}